\documentclass[11pt,letterpaper]{article}
\usepackage[margin=1in]{geometry}
\usepackage[utf8]{inputenc}
\usepackage[T1]{fontenc} % Added to make curly braces look nice in \texttt mode
\usepackage{times}
\usepackage{cancel}
\usepackage{mathtools}
\usepackage{mathpazo}
\usepackage{manfnt}
\usepackage{figchild}
\usepackage{fontawesome5}

\usepackage{accents}
\usepackage{xcolor}
\usepackage{color}
\definecolor{niceRed}{RGB}{190,38,38}
\definecolor{niceYellow}{HTML}{f5b400}
\definecolor{blueGrotto}{HTML}{059DC0}
\definecolor{royalBlue}{HTML}{057DCD}
\definecolor{navyBlue}{HTML}{0B579C}
\definecolor{yaleBlue}{HTML}{00356b}
\definecolor{limeGreen}{HTML}{81B622}
\definecolor{nicePurple}{HTML}{9c27b0}
\definecolor{lightRoyalBlue}{HTML}{def2ff}  
\definecolor{gold}{HTML}{ffa300}

\usepackage{color-edits} 
\addauthor[Anay]{am}{gold}
\addauthor[Grigoris]{gv}{blue}
\addauthor[Felix]{fz}{Orchid}
\addauthor[Xifan]{xy}{niceRed}

\usepackage{hyperref}
\hypersetup{
  colorlinks = true, 
  urlcolor = {blueGrotto},
  linkcolor = {royalBlue},
  citecolor = {navyBlue}
}
\usepackage[
  backref=true,
  backend=biber,
  natbib=true,
  style=alphabetic,
  sorting=alphabeticlabel,
  sortcites=true,
  minbibnames=3,
  maxbibnames=999,
  mincitenames=1,
  maxcitenames=4,
  minalphanames=1,
  maxalphanames=4,
  doi=false, url=false, eprint=false, isbn=false, 
]{biblatex}

\DeclareSortingTemplate{alphabeticlabel}{
  \sort[final]{%
    \field{labelalpha} % Sort by the alphabetic label (\eg{}, DKL, DLL)
  }
  \sort{%
    \field{year}   % If labels are the same, sort by year
  }
  \sort{%
    \field{title}  % If both label and year are the same, sort by title
  }
}

\AtBeginRefsection{\GenRefcontextData{sorting=ynt}}
\AtEveryCite{\localrefcontext[sorting=ynt]}
\usepackage[most]{tcolorbox}

\tcbset{
    mybox/.style={
        colframe=black,        % Border color
        colback=gray!0,       % Background color
        boxrule=0.5mm,         % Border thickness
        width=\linewidth,      % Full width
        arc=3mm,               % Rounded corners
        left=5mm,              % Left padding
        right=5mm,             % Right padding
        top=5mm,               % Top padding
        bottom=5mm,            % Bottom padding
        enhanced,              % Allow for further customization
        halign=center        % Center the box
    }
}

\usepackage{booktabs} 
\usepackage{multicol} 
\usepackage{multirow} 
\usepackage{makecell} 
\usepackage{longtable}

\usepackage{graphicx}
\usepackage{float} 
\usepackage{tikz}
\usepackage{pgfplots}
\pgfplotsset{compat=1.17}
\usepgfplotslibrary{fillbetween}
\usetikzlibrary{arrows.meta}
\usepackage{setspace}

\tikzset{
  myNodeFlex/.style={
    draw,
    rectangle,
    rounded corners,
    text centered,
    minimum height=1.5em,
  }
}

\tikzset{
  myNode/.style={
    draw,
    rectangle,
    rounded corners,
    text centered,
    minimum height=1.5em,
    minimum width=3cm,
    text width=5cm,    
  }
}

\tikzset{
  myNodeNarrow/.style={
    draw,
    rectangle,
    rounded corners,
    text centered,
    minimum height=1.5em,
    minimum width=1cm,
  }
}

\tikzset{
  myNodeWide/.style={
    draw,
    rectangle,
    rounded corners,
    text centered,
    minimum height=1.5em,
    minimum width=6cm,
  }
}

\usepackage{xinttools} % for the \xintFor*** 
\usetikzlibrary{calc}

\def\biglen{20cm} % playing role of infinity (should be < .25\maxdimen)
\tikzset{
  half plane/.style={ to path={
       ($(\tikztostart)!.5!(\tikztotarget)!#1!(\tikztotarget)!\biglen!90:(\tikztotarget)$)
    -- ($(\tikztostart)!.5!(\tikztotarget)!#1!(\tikztotarget)!\biglen!-90:(\tikztotarget)$)
    -- ([turn]0,2*\biglen) -- ([turn]0,2*\biglen) -- cycle}},
  half plane/.default={1pt}
}

\usepackage{physics} % useful math marcos
\usepackage{amsmath}
\usepackage{amssymb}
\usepackage{amsthm}
\usepackage{bbm}
\usepackage{blkarray}
\usepackage{mathtools}
\usepackage{nicefrac}
\usepackage{dsfont}
\usepackage{xspace}
\usepackage{pifont} % for tick and cross symbols
\usepackage{thm-restate} % to re-state theorems before their proofs

\usepackage{enumerate} 
\usepackage[inline,shortlabels]{enumitem} 
\usepackage{typed-checklist}
\usepackage{tcolorbox}
\usepackage[framemethod=TikZ]{mdframed}
\mdfsetup{%
backgroundcolor=white, % yellow!10, , %royalBlue!20, 
roundcorner=4pt,
linewidth=1pt}
\usepackage{changepage} % to indent theorems, definitions etc which appear inside remarks

\usepackage{algorithm}
\usepackage{algpseudocode}[1]

\usepackage{mathrsfs} % for mathscript
\usepackage{soul} % for underline that wraps

\theoremstyle{plain} 
\newtheorem{theorem}{Theorem}[section]
\newtheorem{corollary}[theorem]{Corollary}
\newtheorem{proposition}[theorem]{Proposition}
\newtheorem{lemma}[theorem]{Lemma}
\newtheorem{fact}[theorem]{Fact}

\newtheorem{claim}[theorem]{Claim}

\newtheorem{inftheorem}{Informal Theorem}

\newtheorem{definition}{Definition}
\newtheorem{condition}{Condition}

\newtheorem{infdefinition}{Informal Definition}
\newtheorem*{definition*}{Definition}

\theoremstyle{definition} 
\newtheorem{example}[theorem]{Example}
\newtheorem{remark}[theorem]{Remark}

\theoremstyle{remark}

\NewDocumentEnvironment{pf}{o}
  {\IfNoValueTF{#1}{\begin{proof}}{\begin{proof}[Proof of #1.]}}
  {\IfNoValueTF{#1}{\end{proof}}{\end{proof}}}

\AfterEndEnvironment{definition}{\noindent\ignorespaces}
\AfterEndEnvironment{infdefinition}{\noindent\ignorespaces}
\AfterEndEnvironment{condition}{\noindent\ignorespaces}
\AfterEndEnvironment{infcondition}{\noindent\ignorespaces}
\AfterEndEnvironment{assumption}{\noindent\ignorespaces}
\AfterEndEnvironment{example}{\noindent\ignorespaces}
\AfterEndEnvironment{lemma}{\noindent\ignorespaces}
\AfterEndEnvironment{theorem}{\noindent\ignorespaces}
\AfterEndEnvironment{proposition}{\noindent\ignorespaces}
\AfterEndEnvironment{fact}{\noindent\ignorespaces}
\AfterEndEnvironment{question}{\noindent\ignorespaces}
\AfterEndEnvironment{corollary}{\noindent\ignorespaces}
\AfterEndEnvironment{model}{\noindent\ignorespaces}
\AfterEndEnvironment{remark}{\noindent\ignorespaces}
\AfterEndEnvironment{proof}{\noindent\ignorespaces}
\AfterEndEnvironment{fact}{\noindent\ignorespaces}
\AfterEndEnvironment{minftheorem}{\noindent\ignorespaces}
\AfterEndEnvironment{inftheorem}{\noindent\ignorespaces}
\AfterEndEnvironment{maintheorem}{\noindent\ignorespaces}
\AfterEndEnvironment{restatable}{\noindent\ignorespaces}
\AfterEndEnvironment{infassumption}{\noindent\ignorespaces}
\AfterEndEnvironment{conjecture}{\noindent\ignorespaces}
\AfterEndEnvironment{claim}{\noindent\ignorespaces}

\usepackage[capitalise,noabbrev,nameinlink,sort]{cleveref}
\crefname{section}{Section}{Sections}
\crefname{theorem}{Theorem}{Theorems}
\crefname{lemma}{Lemma}{Lemmas}
\crefname{definition}{Definition}{Definitions}
\crefname{infdefinition}{Informal Definition}{Informal Definitions}
\crefname{conjecture}{Conjecture}{Conjectures}
\crefname{corollary}{Corollary}{Corollaries}
\crefname{condition}{Condition}{Conditions}
\crefname{infcondition}{Informal Condition}{Informal Conditions}
\crefname{construction}{Construction}{Constructions}
\crefname{conjecture}{Conjecture}{Conjectures}
\crefname{claim}{Claim}{Claims}
\crefname{observation}{Observation}{Observations}
\crefname{proposition}{Proposition}{Propositions}
\crefname{fact}{Fact}{Facts}
\crefname{question}{Question}{Questions}
\crefname{problem}{Problem}{Problems}
\crefname{remark}{Remark}{Remarks}
\crefname{example}{Example}{Examples}
\crefname{equation}{Equation}{Equations}
\crefname{appendix}{Appendix}{Appendices}
\crefname{algorithm}{Algorithm}{Algorithms}
\crefname{model}{Model}{Models}
\crefname{figure}{Figure}{Figures}
\crefname{inftheorem}{Informal Theorem}{Informal Theorems}
\crefname{infassumption}{Informal Assumption}{Informal Assumptions}
\crefname{minftheorem}{Main Informal Theorem}{Main Informal Theorems}
\crefname{maintheorem}{Main Theorem}{Main Theorems}
\crefname{assumption}{Assumption}{Assumptions}
\crefname{case}{Case}{Cases}
\crefname{program}{Program}{Programs}

\newlist{asmpenum}{enumerate}{1} % should only occur inside assumption env.
\setlist[asmpenum]{label={\arabic*.},ref=\theassumption.{\arabic*}}
\crefname{asmpenumi}{Assumption}{Assumptions}

\usepackage{etoolbox}

\makeatletter
\renewcommand{\eqref}[1]{\textup{\eqrefform@{\ref{#1}}}}
\let\eqrefform@\tagform@

\newcommand{\changetag}[1]{%
  \renewcommand\tagform@[1]{\maketag@@@{(\ignorespaces#1\unskip\@@italiccorr)}}%
}
\makeatother

\makeatletter
\newcommand{\tagnum}[2]{%
    \refstepcounter{equation}%
    \tag{#1) \ (\theequation}%
    \protected@write \@auxout {}{%
        \string \newlabel {#2}{{\theequation}{\thepage}{}{equation.\theequation}{}}%
    }%
}
\makeatother

\newcommand{\quadtext}[1]{\quad\text{#1}\quad}
\newcommand{\qquadtext}[1]{\qquad\text{#1}\qquad}
 
\newcommand{\quadand}{\quadtext{and}}
\newcommand{\qquadand}{\qquadtext{and}}

\def\abs#1{\left| #1 \right|}

\newcommand{\sinparen}[1]{\ensuremath{(#1)}}
\newcommand{\sinbrace}[1]{\ensuremath{\{#1\}}}

\newcommand{\inbrace}[1]{\ensuremath{\left\{#1\right\}}}

\newcommand{\inparen}[1]{\ensuremath{\left(#1\right)}}

\newcommand{\inangle}[1]{\left\langle#1\right\rangle}
\newcommand{\floor}[1]{\ensuremath{\left\lfloor#1\right\rfloor}}

\newcommand{\ceil}[1]{\ensuremath{\left\lceil#1\right\rceil}}

\newcommand{\N}{\mathbb{N}}
\newcommand{\R}{\mathbb{R}}

\newcommand{\sfrac}[2]{{#1/#2}} 
\newcommand{\nfrac}[2]{\nicefrac{#1}{#2}}

\newcommand{\iid}{i.i.d.}

\usepackage{tikz}

\newcommand{\eps}{\varepsilon}
\renewcommand{\epsilon}{\varepsilon}
\makeatletter
\newcommand*{\tran}{{\mathpalette\@tran{}}}
\newcommand*{\@tran}[2]{\raisebox{\depth}{$\m@th#1\intercal$}}
\makeatother

\mathchardef\NABLA"272
\newcommand*{\Nabla}{\boldsymbol\NABLA}
\let\nabla\Nabla

\renewcommand{\hat}{\widehat}

\renewcommand{\bar}{\overline}

\renewcommand{\tilde}{\widetilde}
\newcommand{\wt}[1]{\widetilde{#1}}

\newcommand{\customcal}[1]{\euscr{#1}}
\newcommand{\cA}{\customcal{A}}

\newcommand{\cH}{\customcal{H}}

\newcommand{\cL}{\customcal{L}}

\newcommand{\cP}{\customcal{P}}

\newcommand{\cX}{\customcal{X}}

\DeclareMathAlphabet{\mathdutchcal}{U}{dutchcal}{m}{n}
\SetMathAlphabet{\mathdutchcal}{bold}{U}{dutchcal}{b}{n}
\DeclareMathAlphabet{\mathdutchbcal}{U}{dutchcal}{b}{n}

\DeclareMathAlphabet\urwscr{U}{urwchancal}{b}{n}%
\DeclareMathAlphabet\rsfscr{U}{rsfso}{m}{n}
\DeclareMathAlphabet\euscr{U}{eus}{m}{n}
\DeclareFontEncoding{LS2}{}{}
\DeclareFontSubstitution{LS2}{stix}{m}{n}
\DeclareMathAlphabet\stixcal{LS2}{stixcal}{m} {n}

\renewcommand{\paragraph}[1]{\medskip \noindent\textbf{#1}~}

\newcommand{\eg}{\emph{e.g.}}
\newcommand{\ie}{\emph{i.e.}}
\newcommand{\wrt}{w.r.t.}

\newcommand{\eat}[1]{}
\newcommand{\negLL}{\ensuremath{\mathscr{L}}}

\newcommand{\hypo}[1]{\mathdutchcal{#1}}

\newcommand{\hyH}{\hypo{H}}

\usepackage{filecontents}
\usepackage{caption}
\usepackage{subcaption}
\usepackage{graphicx}

\renewcommand{\cL}{\negLL}

\usepackage{array}
\newcolumntype{L}[1]{>{\raggedright\let\newline\\\arraybackslash\hspace{0pt}}m{#1}}
\newcolumntype{C}[1]{>{\centering\let\newline\\\arraybackslash\hspace{0pt}}m{#1}}
\newcolumntype{R}[1]{>{\raggedleft\let\newline\\\arraybackslash\hspace{0pt}}m{#1}}

\usepackage{fancyhdr}
\fancypagestyle{custom}{
  \fancyhf{}                % clear header and footer
  \fancyfoot[C]{\textcolor{black}{\fcFishD{0.025}{black}{0.5}}} % center footer: print a bold asterisk

}

\newcommand{\generator}{\mathds{G}}
\renewcommand{\cH}{\hyH}

\setlist[enumerate]{itemsep=0pt}
\setlist[itemize]{itemsep=0pt,label=$\triangleright$}

\DeclarePairedDelimiter{\set}{\{}{\}}

\DeclarePairedDelimiter{\card}{|}{|}

\newcommand{\sset}{\subseteq}

\newcommand{\Cl}{\mathrm{Cl}}

\title{Language Generation with Infinite Contamination}

\author{
    \begin{tabular}{cccc}	
        \begin{tabular}{c}
            Anay Mehrotra\\
            Yale University
        \end{tabular}
        & 	
        \begin{tabular}{c}
            Grigoris Velegkas\\
            Google Research
        \end{tabular}
        &
        \begin{tabular}{c}
            Xifan Yu\\
            Yale University
        \end{tabular}
         & 
         \begin{tabular}{c}
            Felix Zhou\\
            Yale University
        \end{tabular}
    \end{tabular}
}
\date{}

\begin{document}

\pagenumbering{gobble}
\maketitle

\begin{abstract}
    The remarkable success of large language models has led to a {growing body of} theoretical research aimed at understanding them.
    Here, a recent line of work studies language generation in the limit, a formal model of language learning where an algorithm observes an adversarially generated enumeration of strings from an unknown target language $K$ and must eventually generate new, unseen strings from $K$. 
    \cite[NeurIPS]{kleinberg2024language} proved that generation is achievable in surprisingly general settings; whenever $K$ belongs to a known countable collection of languages (even the collection of all Turing-enumerable languages). 
    The generator of \cite{kleinberg2024language} suffers from ``mode collapse:'' it generates from an increasingly {ever-}smaller subset of the target. 
    To address this, the recent work of \cite[FOCS]{kleinberg2025density} formalized stronger notions of generation that require the generator's output to be ``dense'' in the target language; informally, requiring it to asymptotically cover a positive fraction of the target language. 
    They showed that generation with density, surprisingly, remains achievable for all countable collections.
    
    However, both of these works rely on the crucial assumption of \textit{perfect} data: in their model, the adversary can neither insert strings from outside the target language (\ie{}, noise) nor omit strings from it (\ie{}, omissions).
    In practice, training data for language models is notoriously noisy, raising the fundamental question: 
    \begin{center}
        \emph{How much contamination (in the form of omissions or insertions) can language generation tolerate?}
    \end{center}
    Recent works made partial progress on this question by studying (non-dense) generation with either finite amounts of noise (but no omissions) \cite[\text{ICML}]{raman2025noisy} or omissions (but no noise) \cite[SODA]{bai2025noise}.
    In this work, we characterize the contamination tolerance of both types of generation by proving the following results:
    \begin{itemize}
        \item \textbf{Generation under Contamination:}~~ 
        Language generation in the limit is achievable for all countable collections if and only if the fraction of contaminated examples converges to zero. \mbox{When this condition fails, we characterize the collections which remain generable.} 
        \item \textbf{Dense Generation under Contamination:}~~ We show that dense generation is strictly less robust to contamination than standard generation, requiring stronger conditions on the contamination rate for the density notions introduced in \cite{kleinberg2025density}.
    \end{itemize}
    As a byproduct, we resolve an open question of \cite{raman2025noisy} by showing that generation is possible with only membership oracle access to languages with finitely many contaminated examples.
    
Finally, we provide hope for practical noise tolerance while maintaining density via a beyond-worst-case analysis:
We introduce a model where the adversary's enumeration must be ``close'' to a canonical ordering of the language, capturing settings where simpler examples appear before complex ones. 
Here, we prove that dense generation is achievable even with infinite contamination provided the fraction of contaminated examples converges to zero.
This result suggests that curriculum learning (the practice of presenting easier examples early in training that is widely used in pretraining) may be crucial for enabling models to learn despite the high contamination rates in datasets scraped from the web.
\end{abstract}

\thispagestyle{empty}
\clearpage
\thispagestyle{empty}
\tableofcontents
\thispagestyle{empty}
\clearpage 
\pagenumbering{arabic}

\section{Introduction}
    \label{sec:intro}
    Large language models (LLMs) have transformed text generation and are already seeing wide-ranging applications from information retrieval to theorem-proving~(\eg{}, \citep{nagda2025reinforced}). This progress has sparked a growing body of theoretical work seeking to understand the foundations of language generation. These works range from fine-grained analyses of specific architectures~\citep{wei2022representational,telgarsky2023representational,peng2024on,chen2024multilayer,alman2023fast} to abstract investigations asking: when is coherent generation possible at all? 
    Both views have led to useful insights.
    The fine-grained view has led to, \eg{}, speculative decoding~\citep{leviathan2023speculative,chen2023accelerating} and faster algorithms for implementing attention~\citep{han2024hyperattention,wang2020linformer,zaheer2020bigbird,alman2023fast,alman2024complexity}. 
    {On the other hand,} the abstract view has uncovered fundamental trade-offs between (lack of) hallucination and other desirable properties~\citep{kalai2024calibrated,kalavasis2024characterizations,charikar2024facets,kleinberg2025density,kalai2025languagemodelshallucinate,hanneke2018actively} and provided principled approaches for aggregating multiple model outputs (\eg{}, \citep{huang2025is}). 
    At the same time, bringing these perspectives closer to practical applications is a natural goal{:} {theoretical models that incorporate more realistic aspects of language generation can yield sharper insights}.
    Our work resides in the abstract viewpoint on language generation and contributes to bringing it closer to the fine-grained view.
    
    The study of language generation in computer science is not new; it predates LLMs.
    Indeed, Turing's 
    imitation game \citep{turing1950computing} proposed language as a probe of cognition; 
    \citet{shannon1951prediction} analyzed the compressibility of English via its entropy; 
    \citet{chomsky1956three} introduced a hierarchy that has played a central role in understanding the computational power of different automata models \citep{sipser2012introduction}; and 
    \citet{gold1967language} initiated one of the first works on studying when learning from examples is possible, which is arguably the essence of much of today's learning theory.

    Most relevant for us is Gold's model of language identification in the limit, which formulates learning as an online two-player game between an adversary and a learner.
    {In this game, the adversary selects a target language $K=L_z$ from a collection $\cL=\sinbrace{L_1,L_2,\dots}$, and starts enumerating $K$ one element at a time.
    After seeing each element, the learner's goal is to guess the index $z$ of the $K$; and the learner is said to succeed if it stabilizes to the right guess $z$ after some finite time (no matter what elements the adversary shows).}
    A line of work, culminating in the works of \citet{angluin1979finding,angluin1980inductive}, completely characterized the collections which are identifiable in the limit: this characterization showed that identification is largely intractable; even simple collections as regular languages are not identifiable in the limit.

    This raises a puzzle: on the one hand, LLMs acquire non-trivial generation ability from just observing examples, on the other hand, the above model of language learning rules out learnability even for simple collections like regular languages. \citet{kleinberg2024language} offered an elegant resolution to this puzzle by observing that the requirement of generation is weaker than the requirement in Gold's learning model: generation only asks to produce new strings from the language, not to identify the entire language.
    To formalize this, they introduce a model of language \emph{generation} in the limit: 
    \begin{infdefinition}[Language Generation in the Limit; see \cref{sec:preliminaries}]     
        \label{def:languageGeneration:informal}
        The game is specified by a countable collection of languages $\cL=\inbrace{L_1,L_2,\dots}$ (such as all regular languages or all context-free grammars).
        First, the adversary fixes a target language $K \in \cL$ and an enumeration of $K$.\footnote{Formally, an enumeration of $K$ is an infinite sequence $x_1,x_2,\ldots$ (possibly with duplicates) such that each $x_i\in K$ and every $x\in K$ appears at some position.} 
        Then, at iteration $n\geq 1$, the adversary reveals $x_n$; given the history $S_n=\{x_1,\ldots,x_n\}$, the generator outputs a new string $w_n\notin S_n$, which is its guess for an unseen element of $K$. 
        
        A generator $\generator$ is said to succeed in {generating from $\cL$ in the limit} if for all $K\in\cL$ and all enumerations of $K$, there exists a time $n^\star$ such that for every $n\geq n^\star$, we have $w_n\in K\setminus S_n$. 
        $\cL$ is said to be generable in the limit if there is a generator $\generator$ that succeeds for each $K\in \cL$.
    \end{infdefinition}
    In sharp contrast to identification, \cite{kleinberg2024language} showed this goal is remarkably feasible: there is a generator $\generator$ that succeeds in generating from any countable collection $\cL$ (even the collection of all Turing-enumerable languages).
    This surprising positive result sparked a flourishing line of research exploring various aspects of language generation~\citep{kalavasis2025limits,peale2024,li2024generation,kalavasis2024characterizations,hanneke2025union,kleinberg2025density,charikar2024facets,raman2025noisy,bai2025noise,charikar2025pareto}; see \cref{sec:relatedWork} for a detailed discussion.

    \paragraph{Language Generation with Density.}
        A downside of \cite{kleinberg2024language}'s model is that it only requires the generator to generate new elements of the target language $K$ without obtaining a meaningful {coverage} of it.
        For instance, their model {accepts a generator that outputs} ``I generated 1,'' ``I generated 2,'' and so on {as a successful generator for English}. 
        While this generator indeed demonstrates an interesting ability, it can count, it hardly captures the richness of English. 
        It turns out that the generator of \citet{kleinberg2024language} exhibits a similar behavior as it operates by  generating from progressively ``thinner'' subsets of $K$, effectively suffering from mode collapse. 
        To address this limitation, \citet{kleinberg2025density} introduced \textit{language generation with density,} requiring the generator to asymptotically cover a positive ``fraction'' of the target {language}.
        They proposed several notions of density {that,} at a high level, {are} grouped into two categories {according to their view of a generator}:
        \begin{itemize}[leftmargin=20pt]
            \item \textbf{Element-based density} focuses  on the output sequence generated during the infinite game, without taking into account the ``internal'' representation of the generator, \eg{}, the {weights} of an LLM. 
            This resembles the point of view of a user that is accessing the model through an API; the internal representation of the model might be evolving over time, but the user only sees its sequence of outputs.
            While we have not defined element-based density formally, we can still gain some intuition.
            Suppose that over its course of interaction with the adversary, $\generator$ outputs the sequence $4\N=\inbrace{4,8,12,\dots}$ and $K=\N$, then $\generator$ achieves element-based density $\nfrac{1}{4}$.
            If on the other hand, $K$ is the set of even numbers, then this density is $\nfrac{1}{2}.$
            \item \textbf{Set-based density} measures the density of all elements \textit{producible} by the generator $\generator$ at any given time. 
            For this to make sense, $\generator$ must be equipped with an option where after each round, we can pause and ask it to generate more and more elements.
            This is, of course, natural for real-world language models, which once trained are used to generate text as often as necessary.
            Compared to element-based density, set-based density tries to capture expressivity of one model, rather than of single elements generated by an infinite sequence of models.
        \end{itemize} 
        Both densities rely on limiting behavior.
        Element-based density considers an infinite sequence of outputs produced by $\generator$.
        Set-based density, considers the limiting behavior of $d_t$ as $t\to\infty$ where $d_t$ is the density of the generator at the $t$-th step.\footnote{Set-based density involves two limits: first, a limit over the sequence of elements producible by $\generator$ in around $t$; which is used to define $d_t$. Then the limit of $d_t$ as $t\to \infty$. It turns out that changing the former limit from $\liminf$ to $\limsup$ does not qualitatively change the notion of density; hence, \citep{kleinberg2025density} fixed the first limit as $\liminf$.}
        Depending on whether one considers $\liminf$ or $\limsup$ as the notion of a limit, we get four notions of density: lower/upper element-based density and lower/upper set-based density. 
        Since $\liminf d_t\leq \limsup d_t$, lower density is harder to achieve than upper density. 
        
        Dense language generation (for any of the aforementioned notions of density) is a significantly stronger notion than generation in the limit because density prohibits the generator from retreating to generating from ever-shrinking corners of $K.$ 
        Yet, remarkably, \cite{kleinberg2025density} showed that many notions of dense generation are achievable for all countable collections:
        \begin{itemize}[leftmargin=20pt]
            \item \textbf{(Element-based Density):} There is a generator that, for any countable collection, generates in the limit and has element-based lower density  $\geq \nicefrac{1}{8}$.
            For the easier notion of element-based upper density, there is a generator that achieves density $\geq \nicefrac{1}{2}$ for all countable collections.
            \item \textbf{(Set-based Density):}
                Set-based lower density is not achievable for all countable collections.  
                The weaker notion of set-based upper density, however, is: there is a generator achieving a set-based upper density of 1 for all countable collections.
        \end{itemize}
     \noindent In this work, we continue this study of dense language generation, bringing it closer to practice.

     \paragraph{Language Generation {under} Contamination.}
     A critical assumption in \cite{kleinberg2024language} and most subsequent works is that every element $x$ presented to the generator is valid (\ie{}, $x\in K$), and, conversely, every (fixed) element $x\in K$ is eventually presented to the generator. 
     This stands in stark contrast to the reality of LLM training. 
     Pre-training data is notoriously noisy, comprising an amalgam of text from sources of varying quality. Despite extensive resources devoted to data cleaning~\citep{dodge2021documenting, raffel2020exploring} and the emergence of entire industries focused on curating high-quality training data, the web-scale corpora used for pre-training LLMs remain riddled with errors, repetitions, low-quality content, and syntactically invalid code~\citep{lee2022deduplicating, kreutzer2022quality}. 
     Moreover, while the amount of pre-training data keeps increasing, it is unrealistic to hope that it can, eventually, contain every factually correct statement.
     Yet modern LLMs consistently generate coherent text despite training on such imperfect data, raising the following question which is the starting point of our work: \mbox{\textit{Is language generation feasible from training data with large amounts of contamination?}}

     Two works take an initial step for capturing contamination \citep{raman2025noisy,bai2025noise}.
     \citet{raman2025noisy} allow finitely many invalid examples (but no omissions) and show generation remains possible from all countable collections. 
     The finiteness of noisy examples is crucial for their analysis; it is not hard to show that if the noisy examples are unbounded their approach fails.
     The assumption of finite noise is particularly strong in this model where the amount of training data is unlimited.
     \citet{bai2025noise} allow omission of arbitrarily many elements from $K$ (without adding any invalid ones) and, again, show generation is possible for all countable collections provided infinitely many survive.
     This restriction of no invalid elements, however, is inherent to their approach and their generator fails even with \textit{two} invalid examples (\cref{rem:singleInvalidExampleFailure}).\footnote{Apart from these results, both \cite{raman2025noisy,bai2025noise} also have additional results which are less related to our focus here, and we survey them in more detail in \cref{sec:relatedWork}.}
     
     In practice, both invalid elements and omissions persist simultaneously, and in infinite quantity. 
     Understanding contamination-tolerance with both types of contamination in potentially unbounded quantity, is an important conceptual extension that both brings \citep{kleinberg2024language}'s model closer to reality and motivates \mbox{new techniques to handle resulting challenges. 
     This raises our first question:}
     \begin{mdframed}
         \textbf{Question 1.}~~ How much \emph{contamination} (both in the form of invalid examples and omissions)\\ \phantom{.}~~~\hspace{0.5mm}\quad\qquad\qquad can generation tolerate while still being possible for all countable collections? 
     \end{mdframed}
     \noindent Moreover, as shown in \cite{kleinberg2025density}, achieving density requires more sophisticated learning mechanisms, potentially making it more vulnerable to contamination. This raises our second question: 
        \begin{mdframed}
        \textbf{Question 2.}~~
            Is dense language generation as contamination-tolerant as generation?
    \end{mdframed}
    Our first set of results is a thorough answer to both of the above questions; we present informal statements of these results in the next section, but first present a quick summary below.
 
    \paragraph{Summary of Results for Questions 1 and 2.}
    We consider different regimes of the amount of contamination in the adversarial stream, and completely characterize when language generation in the limit is achievable {under} contamination across these regimes.
    Language generation shows surprising tolerance: under natural limits on the frequencies of contamination, generation remains possible for all countable collections even with infinite contamination of both types.
    In comparison, dense generation is much less tolerant: it becomes impossible for all countable collections with any infinite contamination, no matter how sparse.

    This paints a bleak picture for contamination tolerance of dense generation, which is the more interesting type of generation as it ensures $\generator$ actually learns a meaningfully rich subset of $K.$
    In practice, of course, data has much more regularity, which can be leveraged by language models for generation.
    This was already noted by \citep{kleinberg2024language}.
    Given the above results, a natural question is whether additional assumptions on the data-generating process can enable generation in the presence of infinite contamination. 
    Or phrased in another way:
    \begin{mdframed}
        \textbf{Question 3.}~~Is there a natural beyond-worst-case model where dense language generation is\\
        \phantom{.}\qquad\qquad\quad\hspace{0.5mm} achievable under infinite contamination?
    \end{mdframed}
    A natural idea is to consider a stochastic model, where the samples are \iid{} from some distribution. 
    This was studied by \citep{kalavasis2025limits}, who showed that stochasticity does not make achieving breadth (a stronger notion of density) any easier.
    Here, we introduce a new beyond-worst case model that limits the adversary's freedom in choosing the enumeration of $K$.
    Roughly speaking, there is a canonical enumeration of $K$ which can be thought of as having easier examples before harder examples, and the \textit{bounded} adversary is not allowed to present hard examples ``too early.''
    This is inspired by curriculum learning, which is a machine learning strategy that orders training data from easy to difficult, and is widely used to train LLMs \citep{bengio2009curriculum}.
    
    \paragraph{Summary of Results Beyond Worst-Case.}
        Briefly, we show that, with bounded adversaries, dense language generation (for all above notions of density) is achievable for all countable collections with (sufficiently sparse) infinite contamination; requiring the same amount of sparsity required as generating without density.
        We believe this model is of independent interest in the rapidly growing area of language generation. 

    \paragraph{Roadmap.}
        In \cref{sec:intro:results} we present the informal statements of the results described above.
        Then, in \cref{sec:intro:technicalOverview}, we provide an overview of the proofs of the results. 
        Finally, in \cref{sec:intro:discussion} we further discuss our results and present open problems.

\vspace{-6mm}
\subsection{Informal Results}
    \label{sec:intro:results}
    \vspace{-2mm}
    This section presents informal statements of our results.
    First, we present necessary definitions.
    \vspace{-3mm}
     
    \subsubsection{Definitions}
    \vspace{-1mm}
        \label{sec:intro:results:setup}
        We use $\generator=(\generator_n)_{n\in \N}$ to denote generators.
        For each $n$, $S_n=\inbrace{x_1,x_2,\dots,x_n}$ denotes the set of elements revealed by the adversary in the first $n$ iterations and $\generator_n(x_1,\dots,x_n)$ denotes the corresponding output of $\generator$.
        We consider two types of generators: element-based ones which output an element $\generator_n(x_1,\dots,x_n)=w_n$, and set-based ones that output a set $\generator_n(x_1,\dots,x_n)=G_n$.\footnote{This notion of set-based generation can be thought of as an ``improper'' version of index-based generation in \citep{kleinberg2025density} who restricted the outputs $G_n\in \cL$; this improper version was also studied by \citep{kalavasis2025limits}.}
        
        We already informally defined language generation in the limit for element-based generators in \cref{def:languageGeneration:informal}. 
        The set-based definition is very similar:
        A set-based generator $\generator$ is said to generate from $K$ in the limit if there is a finite $n^\star$ such that for all $n>n^\star$, $G_n\subseteq K\setminus S_n$ (\cref{def:generation-in-the-limit}).

        Next, we define different notions of density from \citep{kleinberg2025density}.
        For this, we need to fix a canonical ordering or enumeration of elements in each language $L$. 
        We do so by fixing a canonical enumeration of the universe (of all possible strings) and let{ting} $L_n$ denote the first $n$ elements in the induced canonical enumeration of $L$.
        The two element-based notions of density are as follows:
        \vspace{-1.5mm}
        \begin{infdefinition}[Element-based density; \cref{def:element-based-density}]
            Let $W=\inbrace{w_1,w_2,\dots}$ be the sequence of outputs of an element-based generator $\generator$ with target $K$.
            The element-based lower density and element-based upper density of \mbox{$\generator$ are $\mu_{\rm low}(W; K) \coloneqq \liminf_n \frac{\abs{W\cap K_n}}{n}$ and $\mu_{\rm up}(W; K) \coloneqq \limsup_n \frac{\abs{W\cap K_n}}{n}$ respectively.}
        \end{infdefinition}
        \vspace{-6mm}
        
        \noindent See \cref{app:preliminaries} for a refresher on $\liminf$ and $\limsup$.
        Since $\liminf\leq \limsup$ for any sequence, the element-based lower density is harder to satisfy than element-based upper density.
        {The two can be very different, \eg{}, if $W=\cup_{\ell\in \N} [(2\ell)!, (2\ell+1)!]$, then $\mu_{\rm low}(W,\N)=0$ and $\mu_{\rm up}(W,\N)=1$.}
        Next we introduce set-based notions of density.
        \vspace{-1.5mm}
        \begin{infdefinition}[Set-based density; \cref{def:set-based-density}]
            Let $G=\inbrace{G_1,G_2,\dots}$ be the sequence of sets output by a set-based generator $\generator$ with target $K$.
            The set-based lower density and set-based upper density of $\generator$ are $\liminf_{n\to \infty}\mu_{\rm low}(G_n; K)$ and $\limsup_{n\to \infty}\mu_{\rm low}(G_n; K),$ respectively.
        \end{infdefinition}
        \vspace{-6mm}
        
        \noindent As before, upper density is a weaker requirement than lower density.
        To gain some intuition, suppose $G_n = K$ if $n$ is even and, otherwise, $G_n=\inbrace{w_i}$ where $w_i$ is some unseen string in $K$.
        Then $\generator$ has set-based upper density of 1 and a set-based lower density of 0 (as $\mu_{\rm low}(\inbrace{w_i},K) = 0$). 

        The final ingredient we need to state our results is the levels of contamination considered.
        \vspace{-1.5mm}
        \begin{infdefinition}[Contamination Regimes; \cref{def:enumeration-omission,def:noisy-enumeration}]
            Fix a language $K$ and an enumeration $E=(x_1,x_2,\dots)$ not necessarily of $K$. 
            We focus on the following four cases:
            \vspace{-2mm}
            \begin{enumerate}[itemsep=-1pt]
                \item \textbf{(Finite Contamination)}~ $E$ is an enumeration of $K$ with finite contamination if $\abs{E\triangle K}<\infty$.
                \item \textbf{($o(1)$-Noise)}~ $E$ is an enumeration of $K$ with $o(1)$-Noise if $\nfrac{\abs{\inbrace{x_1,x_2,\dots,x_n}\setminus K}}{n} = o(1)$.
                \item \textbf{(Constant Noise)}~ $E$ is an enumeration of $K$ with $c$-noise if $ \nfrac{\abs{\inbrace{x_1,x_2,\dots,x_n}\setminus K}}{n}\leq c$, for large $n.$
                \item \textbf{(Arbitrary Omission)}~ $E$ is an enumeration of $K$ with arbitrary omissions if $\abs{K\setminus E}=\infty$.
            \end{enumerate}
        \end{infdefinition}
        \vspace{-6mm}
        
        \noindent Finite contamination allows both invalid examples and omissions, already generalizing the case of just finite noise in \cite{raman2025noisy}.
        \enlargethispage{1\baselineskip}
        The next two regimes quantify the fraction of noisy examples among the first $n$: in $o(1)$-noise this fraction tends to 0, and in $c$-noise, it becomes $\leq c$ for large enough $n$.
        With omissions, a notion of ``the fraction of omissions'' is more nuanced because we never know when an omission happened (a missing element could always appear later in the enumeration).
        {With omissions, quantifying ``the fraction of omissions'' is more nuanced since missing elements could always appear later.
        Hence, we primarily focus on two regimes: finite omissions (a special case of finite contamination) and arbitrary omission (studied by \citep{bai2025noise} without noise).}

    \subsubsection{Our Results}
        Our main results are comprehensive characterizations of when language generation with and without density is achievable under contamination.
        Interestingly, all our algorithmic results follow from two algorithmic templates (\cref{sec:templates}), explained in the technical overview (\cref{sec:intro:technicalOverview}).
        Our first result studies generation under contamination; its formal version is in \cref{sec:generation}.
        \begin{inftheorem}[Generation {under} Contamination]  
            \label{infthm:gen}
            The following results hold:
            \begin{enumerate}[itemsep=0pt,leftmargin=18pt]
                \item Under $o(1)$-noise and arbitrary omissions, all  countable  collections are generable in the limit.
                \item Under $c$-noise (for any fixed $c\in (0,1)$) and arbitrary omissions, a  countable  collection $\cL$ is generable in the limit if and only if $\cL$ satisfies the condition in \cref{thm:constant-noise-characterization} with parameter $c$.
                Moreover, for each $c\in (0,1)$, there is a {\underline{finite}} collection $\cL$ that violates this condition. 
            \end{enumerate}
        \end{inftheorem}
        The first result shows that generation is quite tolerant to contamination: 
        even with infinite contamination, generation remains possible for all countable collections when the noise-fraction tends to 0.
        This significantly generalizes both \citep{raman2025noisy,bai2025noise}.
        It might seem surprising that we do not require any constraint on omissions; this stems from generation's asymmetric objective: it requires generating new unseen elements, thus penalizing false positives much more than false negatives.

        The second result explores the harder regime where noise fraction does not vanish, revealing the limits of generation's contamination tolerance.
        While we omit the technical details of the condition, we note that all finite collections of size at most \mbox{$\nfrac{1}{c}$ satisfy the condition with parameter $c$.}

        The algorithms for both regimes utilize a novel idea: they carefully re-order languages based on the adversary's actions (\cref{sec:templates}), contrasting \mbox{with existing algorithms that use static orderings.}

        Our next results study dense generation under contamination, starting with set-based density:

        \begin{infdefinition}[Set-based Dense Generation {under} Contamination]
            \mbox{The following results hold:}
            \begin{enumerate}[itemsep=0pt,leftmargin=18pt]
                \item Under finite contamination, all  countable  collections are generable  with set-based upper density 1.
                
                \item Under finite contamination, a  countable  collection $\cL$ is generable in the limit with $\rho$ set-based lower density if and only if $\cL$ satisfies the condition in \cref{thm:set-based-lower-density-finite-noise} with parameter $\rho$. 
                Moreover, for each $\rho\in (0,1]$, there is a  countable  collection that violates this condition.
                \item Under $o(1)$-noise rate and arbitrary omissions, a  countable  collection $\cL$ is generable in the limit with (lower or upper) set-density $\rho$ if  and only if $\cL$ satisfies the condition in \cref{thm:vanishing-noise-set-density-characterization} with parameter $\rho$. 
                Moreover, for each $\rho\in (0,1]$, there is a  countable  collection that violates these conditions.
                \item Under $c$-noise rate (for $c\in (0,1]$) and arbitrary omissions, a  countable  collection $\cL$ is generable in the limit with (lower or upper) set-density $\rho$ if  and only if $\cL$ satisfies the condition in \cref{thm:constant-noise-set-density-characterization} with parameters $(c,\rho)$. 
                Moreover, for each $c,\rho\in (0,1]$, there is a  countable  collection that violates these conditions.
            \end{enumerate}
        \end{infdefinition}
        {Formal statements of these results appear in \cref{sec:dense-generation:set-based}.}
        The key takeaway is that dense generation is much less contamination-tolerant than non-dense generation.
        This contrasts sharply with \citep{kleinberg2025density}'s positive results, which suggested density was achievable whenever generation was.
        Interestingly, while set-based lower and upper densities have different characterizations without contamination, they collapse to one under contamination (parts 3-4).
        Finally, set-based lower density becomes particularly challenging, even under finite contamination: it is unachievable whenever $\cL$ has two languages, such as, $L=\N$ and $P=\inbrace{n \text{~is prime}}_{n\in \N},$ that  satisfy $P\subseteq L$ and $\mu_{\rm low}(P,L)=0$ (part 2).
        This is a much more stringent requirement the requirement without contamination \citep{kleinberg2025density}, who showed that impossibility of set-based generation with lower density requires the existence of a pathological sub-collection of $\cL$ with infinite cardinality.
        \begin{remark}
            The result for set-based upper density is stronger: we are able to achieve it under finite noise and an \textit{infinite} amount of omissions provided the omissions are not ``too dense;'' we formalize this in \cref{sec:preliminaries} as $c$-omissions.
            The formal results appear as \cref{thm:finite-noise-upper-density,thm:set-based-upper:lower-bound}.
        \end{remark}
        All algorithms above rely on our two algorithmic templates: finite contamination uses one template, while the remaining cases use the other.
        We believe these templates can be of independent interest in the study of language generation
        As an illustration, we use them to answer an open question from \citep{raman2025noisy}: %
        \begin{corollary}[Membership-Query--Based Algorithm with Finite Contamination; \cref{cor:generation-membership-oracle-finite-omission}]
            There is a \underline{computable} $\generator$ that, for any  countable  collection $\cL=\sinbrace{L_1,L_2,\dots}$, given access to an oracle which given $w$ and $i$ answers ``Is $w\in L_i?$,'' generates $\cL$ in the limit under finite contamination.
        \end{corollary}
        In contrast, \citep{raman2025noisy} required additional oracles to achieve generation under finite noise.

        Next, we study element-based density under contamination.
        Surprisingly, unlike set-based density, element-based density does not uniformly become harder in each contamination regime.
        \begin{inftheorem}[Element-based Dense Generation {under} Contamination]
            The following hold:
            \begin{enumerate}[leftmargin=14pt,itemsep=0pt]
                \item Under finite contamination, every countable collection $\cL$ is generable in the limit with element-based upper density $\rho=\nicefrac{1}{2}.$ %
                \item Under finite contamination, every countable collection $\cL$ is generable in the limit with element-based lower density $\rho =\nicefrac{1}{8}.$ %
                \item Under $o(1)$-noise, \mbox{there is a collection that isn't generable in the limit with ${>}0$ element-based upper density.}
            \end{enumerate} 
        \end{inftheorem}
        In fact, our formal versions of parts 1 and 2 are stronger: We show that if there exists an algorithm that achieves element-based lower (respectively upper) density $\rho$ for all countable collections in the contamination-less case, there exists an algorithm that achieves element-based lower (respectively upper) density $\rho$ under finite contamination (\cref{thm:element-based-density-finite-noise-omissions}).
        Hence, remarkably, element-based density with finite contamination is exactly as hard as with no contamination, contrasting with set-based lower density which became significantly harder.
        (Indeed, the upper and lower densities of $\rho=\nfrac{1}{2}$ and $\rho=\nfrac{1}{8}$ match those obtained by \citep{kleinberg2025density} with no contamination.)
        The formal statements of these results appear in \cref{sec:dense-generation:element-based}.

        \paragraph{Language Generation Beyond the Worst-Case (\cref{sec:beyond}).}
            Next, we introduce a beyond-worst-case model which restricts the order in which the adversary can present the elements $K$.
            \begin{infdefinition}[Bounded Adversary; \cref{def:bounded-displacement-enumeration}]
                Let the first $n$ elements of $K$ be $\inbrace{\kappa_1, \kappa_2,\dots }$
                We say that an adversary is $M$-bounded if it presents an enumeration $x_1,x_2,\dots$ such that, for sufficiently large $n$, if $x_n\in K$ then $x_{n}=\kappa_{i_n}$ for some $i_n\leq Mn$.
            \end{infdefinition}
            Importantly, an $M$-bounded adversary can still select arbitrary target languages $K$ but cannot present its elements in an arbitrary order. 
            Under the interpretation that the canonical enumeration of $K$ places ``easier'' elements before ``harder'' ones, $M$-bounded adversaries cannot, infinitely often, enumerate very hard examples way before easier ones. 
            This is inspired by practical observations of LLM training; \eg{}, it is a crucial folklore {practice} that LLMs are first trained on ``easier'' tasks before harder ones.
            More broadly, this phenomenon is known as ``curriculum learning'' \citep{bengio2009curriculum,hacohen2019power}.

            In this beyond-worst-case model, we show that dense generation is much more tractable.
            \begin{inftheorem}[Dense Generation with Bounded Adversary]
                With an $M$-bounded adversary:
                \begin{enumerate}[itemsep=0pt]
                    \item Under $o(1)$-noise and arbitrary omissions, all  countable  collections are generable in the limit with set-based lower density $\geq \frac{1-\eps}{M}$ (for any fixed $\eps>0$). 
                    Moreover, there is a  finite  collection for which it is impossible to generate in the limit with set-based upper density $>\nfrac{1}{M}$.
                    \item Under $o(1)$-noise and arbitrary omissions, all  countable  collections are generable in the limit with element-based lower density $\frac{1-\eps}{2M}$ (for any fixed $\eps>0$). 
                    Moreover, there is a  finite  collection for which it is impossible to generate in the limit with element-based upper density $>\nfrac{1}{M}$.
                \end{enumerate}
            \end{inftheorem}
            Thus, both set-based and element-based lower densities become achievable for all countable collections with $M$-bounded adversaries.
            This is a significant improvement over arbitrary adversaries where even finite contamination was problematic; and no notion of density was achievable with finite contamination (no matter how sparse).
            We remark that, while dense generation is more tractable, the beyond worst-case model remains non-trivially hard: identification largely remains impossible in the model (\cref{thm:beyond:identification-hard}) and generation without density still fails at constant noise rates as in \cref{infthm:gen} (\cref{thm:beyond:generation-cnoise-hard}).
            The formal statement of this result appears in \cref{sec:beyond}.

            \paragraph{Proper vs. Improper Learning in the Worst-Case.}
            We conclude this section with the following: %
        
            \begin{remark}[Proper vs.\ Improper Learning in the Worst-Case]
                For the vast majority of our set-based results, the learning algorithms are \emph{improper} meaning that the set they output is not part of the collection $\cL.$ 
                While we allow the outputs to be arbitrary, the outputs of our algorithms are much more structured.
                For instance, several of our algorithms output \emph{intersections} of finitely many languages from $\cL.$
                The remaining algorithms output a set $G$ such that $\abs{G\triangle L}<\infty$ for some $L\in \cL$.
                In fact, \cref{ex:index-failure-under-omission} shows that this is necessary: proper learning in the presence of contamination is much more restrictive than improper learning. This is in sharp contrast to the uncontaminated setting, where the results of \citep{kleinberg2024language,kleinberg2025density} show that proper learners are as powerful as improper ones.
            \end{remark} 

\subsection{Technical Overview}
    \label{sec:intro:technicalOverview}
    In this section, we give an overview of our techniques and how they relate to prior works.
    Our main results are a comprehensive characterizations of when different notions of generation can or cannot be achieved under different levels of contamination. 
    We divide this into two parts: upper bounds (or algorithms) and lower bounds.

    \subsubsection*{Upper Bounds for Dense and Non-Dense Generation}
        All of our algorithms across the worst-case model and the bounded adversarial model, with and without density, rely on two algorithmic templates. 
        Here, we overview these templates.
        They are discussed in more detail in \cref{sec:templates}.

    \paragraph{Finite-Expansion Sub-Routine (\cref{sec:templates:expansion}).}
        Our first algorithmic template is quite simple and enables us to achieve dense and non-dense generation with finite contamination.
        The algorithm relies on the following elementary observation: If the enumeration $E$ generated by the adversary has finite contamination, then it must be a 0-contamination enumeration of a language $K'$ satisfying with $\abs{K'\triangle K}<\infty$.
        Now, of course, $K'$ might not be a language in our collection $\cL$, but the natural thing to do is to add $K'$ to $\cL$.
        We will need to add all possible choices since we do not know $K'$ in advance.  Then, we end up with the following collection
        \[
            \Tilde{\cL} 
                \coloneqq 
            \{
                L_{A,B} \coloneqq L \cup A \setminus B : L \in \cL, A \subseteq \Sigma^* \setminus L, B \subseteq L, \abs{A} < \infty, \abs{B} < \infty
            \}\,,
        \]
        where $\Sigma^*$ is the universe of strings.
        Since the number of finite subsets of a countable collection are countable, $\wt{\cL}$ is a countable collection.
        This immediately gives us an algorithm for generation with finite amounts of contamination since we can use \citep{kleinberg2024language}'s algorithm that works for all countable collections with $\Tilde{\cL}$ instead of $\cL$.
        While this idea is obvious in hindsight and analyzing it is also straightforward, it has many useful properties that make it widely applicable:
        \begin{itemize}
            \item As we have seen, it preserves countability (if $\cL$ is countable, so is $\tilde{\cL}$)
            \item Further, membership access to $\cL$ is sufficient to get membership access to $\tilde{\cL}$ (\cref{lem:km-membership-algo})\footnote{A membership oracle to $\cL=\inbrace{L_1,L_2,\dots}$ is a primitive that, given $w\in \Sigma^*$ and $i$, answers ``Is $w\in L_i$?''}
            \item If $\generator$ generates with density $\rho$ (for any notion of density we study) with respect to $\cL$, then $\generator$ also generates density $\rho$ with respect to $\wt{\cL}$.
        \end{itemize}
        This is not an exhaustive list.
        For instance, the transformation also preserves generation with approximate breadth, which is a stronger notion than density introduced in \citep{kalavasis2025limits}.
        These properties are what make this transformation interesting, and it leads to a number of results. 
        In particular, points 1 and 2 above together allow us to get a membership oracle-based algorithm for generation with finite contamination that already resolves an open question in \citep{raman2025noisy}; which seems hard to get using prior techniques.

    \paragraph{Priority-Based Intersection Algorithm (\cref{sec:templates:priority}).}
        Next, we discuss our algorithms for the much more involved case where there are amount of contamination is infinite.
        We begin by explaining why earlier approaches fail:
        \begin{enumerate}
            \item The previous simple approach that constructs the ``expanded'' collection $\wt{\cL}$ fails because the collection $\wt{\cL}$ becomes uncountable (if the amount of noise is not finite) and not all uncountable collections are generable in the limit, even without the requirement of density.
            The situation with density is even more complicated as the transformation is no longer density preserving. 
            \item Another idea is to use approaches from \citep{raman2025noisy,bai2025noise}.
                \citep{raman2025noisy} use a nice observation for generating in the limit: if the amount of noise is finite then, for sufficiently large $n$, the second half of the training examples (namely, $\inbrace{x_{n/2}, x_{n/2+1},\dots,x_n}$) eventually contains no noisy examples.
                Hence, roughly speaking, feeding the second half of the examples to an appropriate generation algorithm suffices.
                This, of course, fails when there is an infinite amount of noisy examples in the stream.
                \citep{bai2025noise} make another nice observation: if $\generator$ has the property that it generates from $\cL$ after $n^\star=n^\star(K)$ iterations, where $n^\star$ depends on the target $K$, but not on its enumeration, then $\generator$ generates from $\cL$ in the limit under arbitrary omissions (provided there is no noise).
                This suffices as \citep{charikar2024facets} constructed a generator with this property for all countable collections.
                However, this approach fails even when there are two noisy examples in the enumeration without any omission (see \cref{rem:singleInvalidExampleFailure}).
        \end{enumerate}
        To understand our (meta) algorithm, it is instructive to first understand the algorithm of \citep{charikar2024facets}. 
        Roughly speaking, in the $n$-th iteration, their algorithm considers the $n$ languages $\inbrace{L_1,L_2,\dots,L_n}$ and, from these, it removes any language inconsistent with the training data seen so far $S_n$, \ie{}, any $L\not\supseteq S_n$.
        Let the resulting languages be $L_{\sigma(1)},L_{\sigma(2)},\dots,L_{\sigma(m)}$ (for $m\leq n$), then they output $L_{\sigma(1)}\cap L_{\sigma(2)}\cap \dots \cap L_{\sigma(\ell)}$ for the largest $1\leq \ell\leq m$ such that the resulting intersection is infinite.\footnote{To get an element-based generator, they output the smallest unseen element from this set.}
        
        Of course, this algorithm does not directly work when there is contamination in the sample stream because, for instance, the consistency check of $L\not\supseteq S_n$ is no longer meaningful. 
        Consider the case of $c$-noise; let the enumeration be $E=\inbrace{x_1,x_2,\dots}$.
        The natural counter part is to check whether $E$ is an $c$-noisy enumeration of $L$.
        This, however, requires having access to the entire enumeration $E$, which we have never seen at any iteration $n<\infty$.
        One could instead check if the fraction of elements from $\inbrace{x_1,x_2,\dots,x_n}$ not in $L$, which we term the \textit{empirical noise rate,} is at most $c+\eps$ for small fixed $\eps>0$.
        There are examples where this fails because this approach can include ``bad'' languages $L$ that do not meet the actual requirement that $E$ is a $c$-noisy enumeration for $L$.
        One could tighten this check by setting $\eps=0$, but then we run into the issue that the empirical noise rate for certain bad languages can fluctuate above and below $c$ infinitely often.
        To gain some intuition suppose $K=\inbrace{2,4,6,\dots}$ and $L=\inbrace{n\colon n\in \N, n\text{~mod~}4\neq 0}\cap A$ where $A=\cup_{\ell\in \N}[(2\ell)!,(2\ell+1)!]$.
        Then, the enumeration $E=\inbrace{1,2,3,\dots}$, is a $\nfrac{1}{2}$-noisy enumeration for $K$ and $\nfrac{3}{4}$-noisy enumeration of $L$ (so $L$ should be excluded in our check).
        However, the empirical noise rate of $L$ fluctuates between $0$ and $\nfrac{3}{4}$ infinitely often (and, hence, $L$ would end up being included in our check infinitely often).
        
        This is not merely a superficial problem in the above approaches, but rather an inherent problem in verifying $c$-noise or $o(1)$-noise, which are necessarily limiting phenomena, at finite times.
        Hence, to overcome this we need a new approach that is able to (i) remove ``bad'' languages from our ``active set'' of languages, and (ii) ensure that the target language $K$ always remains in this set after some finite time.
        To design our approach, we take inspiration from the failure of the above approaches, where certain ``bad'' languages fluctuate between passing our check and failing our check infinitely often.
        Our algorithm assigns each language a priority.
        Initially, the priority of language $L_i$ is simply $-i$ (so the languages in order of priority are $L_1,L_2,\dots$) and, each time, $L_i$ fails our check we penalize the language by decreasing its priority.
        The key observation is that for both $c$-noise and $o(1)$-noise, the target language $K$ will only be penalized for a finite amount of time; where as every ``bad'' language will be penalized infinitely often.
        This idea is sufficient for us to design algorithms for generation with $c$-noise and $o(1)$-noise.
        
        However, like our previous algorithmic template, this also has quite general applicability: to all notions of density and also for generation with bounded adversaries.
        Indeed, this core algorithmic template of assigning priorities to languages shows up in all of our algorithms beyond the finite contamination regime, each time, with a slightly different notion of priority.
        For instance, with $M$-bounded adversaries, we need to check whether the provided enumeration $E$ is indeed $M$-bounded for a langrage $L$; this is again a limiting phenomenon which cannot be checked at any finite time.
        The situation for dense generation is even more complicated.
        Indeed, \citep{kleinberg2025density} already demonstrated that achieving lower density (both element-based set-based) for all countable collections is incredibly difficult and their algorithms carefully need to balance the trade-off between generating only valid elements and covering enough fraction of $K$.
        Now, with infinite amount of contamination this becomes even harder due to the fluctuating empirical noise rates we discussed above.

    \subsubsection*{Lower Bounds for Dense and Non-Dense Generation}
        A common theme in learning theory is that, once we have the ``right'' algorithm, obtaining tight lower bounds is not that hard. 
        This is also the case with our lower bounds.
        An interesting property of our lower bounds is that many of them have a finite witness. 
        For instance, consider dense generation with finite contamination.
        All our characterizations here have the following form: a collection $\cL$ violates the condition if there are two pathological languages $L_1,L_2\in \cL$ with a certain property.
        (For instance that $L_1\subseteq L_2$ and $\mu_{\rm low}(L_1,L_2)=0$.)
        This contrasts existing lower bounds in generation and identification which required infinite witnesses.
        For instance, \citep{kleinberg2025density} showed that the existence of an infinite perfect tower characterizes lower set-based density. 
        However, as its name suggests, an infinite perfect tower is necessarily witnessed by an infinite sub-collection of languages.
        Similarly, the characterizations of identification in the limit \citep{angluin1980inductive} and language generation with breath \citep{kalavasis2025limits,charikar2024facets,kalavasis2024characterizations} also require infinite witnesses.

\subsection{{Open Problems}}
    \label{sec:intro:discussion}
    There are several interesting future directions related to noisy generation, and more broadly, the line of work on generation in the limit. Regarding noisy generation, an interesting direction is to fully characterize element-based generation for all types of contamination we consider in our work. Moreover, it would be nice to obtain tight bounds in the beyond-worst-case setting we introduced in our work. While we have developed an algorithm that uses only membership oracle access to $\cL$ in the setting of finite contamination, we have developed several algorithms that require more complicated oracles. It is an interesting open direction to understand what can be done using simpler oracles in this setting. Last but not least, perhaps a way to circumvent some of the lower bounds we have shown, other than restricting the adversary, is to relax the requirement of the learner: how does the landscape of (noisy) generation look like we allow the generator to output a vanishing amount of hallucinations? We remark that a similar question was asked by \citep{kleinberg2025density} in the context of improving the density guarantees of their algorithms.

\section{Related Work}
    \label{sec:relatedWork}
    {In this section, we present further related work, including other works on language identification in the limit and language generation in the limit.
    In an attempt to give through overviews of the presented works, this section is quite long, perhaps unavoidably.
    That said, reading it is not necessary to understand our results and we encourage readers to skip it in the first reading and revisit as necessary.}

    \subsection{Related Work on Language Identification in the Limit}
        {Starting from \citet{gold1967language} there has been a rich line of work including both linguistics and computer science on the model of language generation in the limit.
        Of particular relevance to our work are the line of works on language identification in the limit in settings where the adversary can include invalid examples in the enumeration (\ie{}, noise) or omit elements \citep{mukouchi2003refutable,case1997synthesizing,stephan1997noisy,baliga1992learning,fulk1989learning,schafer1985some}.}
        In particular, 
        \citet{jain1994identificationNoisy} extends Gold’s model to streams with infinitely many inaccuracies, controlling corruption via density (covering finite, vanishing, and constant-rate noise).  \citet{mukouchi2003identificationNoisy} use a neighbor-system (metric) model that permits insertions and deletions provided each corrupted item lies within a fixed distance of some true string. \citet{tantini2006identificationNoisy} analyze the same metric-based setting and prove identification results under such systematic noise.

    \subsection{Related Work on Language Generation in the Limit}
    {Next, we discuss works on language generation in the limit.
    We have already discussed some of these works \citep{kleinberg2024language,raman2025noisy,bai2025noise,kleinberg2025density} briefly in the introduction.
    We expand upon the discussion of these and other works below.}

    \paragraph{Language Generation with Uncountable Collections.}
        {As we mentioned in \cref{sec:intro}, \cite{kleinberg2024language}  introduced the model of language generation in the limit.
        While they and many subsequent works, including ours, focus on the setting where the collection of languages is countable, \citet{li2024generation} extended the study to uncountable collections of languages.
        \cite{li2024generation} introduced two more fine-grained notions of language generation: uniform and non-uniform generation.
        Uniform generation requires the number of strings $n^\star$ the generator needs to see before starting to generate (see \cref{def:generation-in-the-limit}) to be independent of the choice of the targe language $K$.
        Non-uniform generation allows $n^\star$ to depend on $K$, but requires it to be independent of the enumeration of $K$ chosen by the adversary.
        (Note that generation in the limit allows $n^\star$ to be {dependent} of both $K$ and the enumeration of $K$.)
        \citep{li2024generation} characterized the collections that are uniformly generable as well as collections that are non-uniformly generable.
        The latter characterization, in particular, demonstrates that all countable collections can be non-uniformly generated although with a different algorithm that that of \citep{kleinberg2024language}; this result was also concurrently and independently obtained by \citep{charikar2024facets}. Interestingly, \citet{charikar2024facets} showed that this type of generation is \emph{impossible} if the learner only has membership oracle access to $\cL.$
        \citep{li2024generation} gave several sufficient conditions for uncountable collections to be generable in the limit and left a complete characterization as an open problem.
        En route to obtaining a complete characterization, they left other open problems, notably checking closedness under finite unions: if $\cL_1,\cL_2,\dots,\cL_k$ are generable in the limit, then is $\cup_{i=1}^k \cL_i$ also generable in the limit?
        \cite{hanneke2025union,bai2025noise} concurrently and independently resolved this problem by showing that the statement is false even with $k=2$.
        Their proofs are near-identical and rely on a diagonalization argument which has also been featured in other works on language generation as discussed below.}

    \paragraph{Language Generation in a Statistical Model.} 
        {\citep{kleinberg2024language}'s model has also been extended along another access: adversarial enumeration of examples.
        While \citep{kleinberg2024language}, study generation when the examples are provided by a worst case adversary, \citep{kalavasis2025limits} introduce a statistical model where the samples are generated \iid{} from a \textit{fixed} distribution $\cP$ whose support matches the target language $K$.
        In this model, they obtain the first and near-tight sample complexities of identification in the limit, generation in the limit, and related tasks (which are discussed later in this section).
        This is precisely the notion of sample complexity studied in the universal-rates framework of \citep{bousquet2021theory}.
        Apart form \citep{bousquet2021theory}, \citep{kalavasis2025limits}'s model was also described by \citep{angluin1988identifying} in the context of language identification; \citep{angluin1988identifying}  showed that the characterization of which collections are identifiable does not change between the stochastic and in-the-limit models.}
        This model can also thought of as a model beyond the worst-case, as the enumeration shown to the generator is sampled \iid{} from a distributions instead of being chosen by an adversary. 
        However, it is quite different from the beyond-worst-case model we study in this work. 
        Neither model is a special case of the other.
        On the one hand, in the stochastic model the adversary can construct a distribution with a heavy tail ensuring that elements with large indices in the canonical ordering can appear early in the samples so does not meet the boundedness requirement (\cref{def:bounded-displacement-enumeration}).
        On the other hand, the bounded adversary model in this work allows the enumeration selected to be adversarial (and, in particular, non-stochastic) provided it meets the boundedness requirements.
        It is an interesting question to understand if the stronger results we obtain for dense generation can also be obtained in the statistical model.

    \paragraph{Language Generation in the Limit with Noise.}
        {Having discussed the variants of language generation in the literature, we now turn to discussing works on language generation with noisy examples.
        As already mentioned in \cref{sec:intro}, there are only two relevant works in the literature: \citep{raman2025noisy,bai2025noise}.}
        \cite{raman2025noisy} initiated the study of language generation in the limit with a \textit{finite} amount of additive noise $n^\star$ (unknown to the generator); in this setting, they characterized classes that are uniformly generable.
        Their results, in particular, imply that all countable collections remain generable with a finite amount of additive noise.
        In the context of generation from noisy samples, \cite{bai2025noise} paper shows that infinite omissions do not change which collections are uniformly or non-uniformly generable. 
        Combining this with the result in \cite{charikar2024facets} implies that all countable collection remains non-uniformly generable with infinite omissions.
        \cite{bai2025noise} also study the setting in \cite{raman2025noisy}; they show several results including a characterizations of non-uniform generation with a finite amount of additive noise.
        The main additional challenges in our work is that we deal with (i) the setting with an \textit{infinite} amount of additive noise (where the algorithms from these works fail) and (ii) also simultaneously deal with omissions.

    \paragraph{Breadth and Density in Language Generation.}
    Prior to the density-based definitions of ``breadth'' proposed by \citet{kleinberg2025density}, the parallel works of \citet{kalavasis2025limits,kalavasis2024characterizations,charikar2024facets} introduced more stringent notions of breadth, asking (roughly) that the (infinite) set of elements
    the generator can produce in every timestep is eventually a subset of the target and misses only \emph{finitely} many elements from $K.$ They characterized when this family of notions of breadth can be achieved. Moreover, \cite{kalavasis2025limits} introduced and studied a statistical model of language generation, and \cite{charikar2024facets} explored other aspects of generation including feedback, non-uniform\footnote{This is a technical term which studies how quickly the learner can start outputting valid, unseen elements of $K$.} generation with various types of oracle-access to the collection, and the role of feedback in generation. On a related note, \citet{peale2024}  {introduced a notion of ``representative generation'' where the strings of the universe are divided into different groups and the generator's outputs are required to represent all groups. This can also be thought of as a weakening of the notions of breadth studied in aforementioned works.} 

    \paragraph{Further Works on Language Generation.}
{Apart from the aforementioned works on language generation, there are two other recent papers studying this framework. \citet{charikar2025pareto}
study a notion of Pareto-optimal non-uniform generation, which seeks to obtain algorithms that achieve the Pareto frontier for non-uniform generation. On a different front, 
\citet{karbasi2025impossibility} propose and study a theoretical model of hallucination detection that is inspired by the language identification and generation settings. Lastly, \citet{vafa2024evaluating} propose a theoretical model that measures the proximity between the language represented by an LLM and the language it was trained on..
}

    \subsection{Other Theoretical Efforts at Understanding Language Generation}

        For completeness, we discuss some other theoretical efforts on understanding LLMs.
        
        \paragraph{Representational Power of Transformer Architecture.} There has been a long line of work trying to understand the class of functions that transformers can represent; see, \eg, \citep{telgarsky2023representational,peng2024on,chen2024multilayer} and references therein.

        \paragraph{Effectiveness of Chain-of-Thought.}
            Another recent theoretical effort has to do with building theoretical evidence about the usefulness of chain-of-thought in language models \citep{malach2023auto,altabaa2025cot,joshi2025theory,huang2025transformers}.

        \paragraph{Models Hallucinations.}
            There have several theoretical works
            giving evidence that LLMs that provide some non-trivial utility \emph{must}
            hallucinate \citep{kalai2024calibrated,kalai2025languagemodelshallucinate,xu2024hallucination}.

        \paragraph{Model Stealing.} The process of trying to infer the parameters of a proprietary LLM through calls to its API is known as \emph{model stealing}. \citet{liu2025model}, building on the work of \citet{mahajan2023learning}, showed that, under a theoretical model, this is possible for any LLM that has low rank.

\section{Preliminaries and Model}\label{sec:preliminaries}
In this section, we formalize language generation with adversarial noise in the example stream. We follow the learning-theoretic presentation of generation in the limit used in recent work, and then parameterize the adversary by the amount of incorrect or \textit{noisy} examples it can inject and the number of elements from $K$ it can skip or \textit{omit} from the enumeration.

\subsection{Notation}

Let $\Sigma$ be a finite alphabet. We use $\Sigma^*$ to denote the set of all finite strings over $\Sigma$.
We often call $\Sigma^*$ the \textit{universe} and use $U$ to denote it.
We will often take $U=\Sigma^*$ to be the set of natural numbers $\N$. This is without loss of generality as $\Sigma^*$ is a countable set because $\Sigma$ is finite. 
Given a sequence $x=(x_1, x_2, x_3, \dots)$ of elements of $U$, we use $x_{1:n}$ to denote the prefix of elements $x_1, x_2, \dots, x_n$. 

For us, a \textit{language} is any infinite subset $L\subseteq U$. An enumeration of a language $L$ is an ordered list $\ell_1,\ell_2,\ell_3,\dots$ of elements of $L$, so that every element of $L$ appears in the list exactly once.\footnote{We note that our definition of enumeration differs from some of the previous works in that we do not allow repeated elements. As we detail in \Cref{apx:repetitions},
this assumption is without loss of generality.}
Given a collection $\cL$ of languages (\ie{}, a collection of infinite sets in $U$), the closure of $\cL$ is the common intersection of all languages in $\cL$ and is denoted as $\Cl(\cL)\coloneqq \cap_{L\in \cL} L.$
When talking about densities, we will also endow the universe $U$ with a canonical ordering $u_1,u_2,\dots$.
When $U=\N$, the canonical enumeration will always be $1,2,3,\dots$.
This canonical ordering induces a canonical ordering of each language $L\subseteq U$ and, for language $L$, we denote it by $\ell_1,\ell_2,\dots$.
Intuitively, one can think of this ordering as ordering the elements of $L$ from ``easy'' ones to ``harder'' ones; we will return to this intuition in \cref{sec:beyond} where we go beyond the worst-case adversarial model.

Throughout the paper, we use standard asymptotic notations $o(\cdot), O(\cdot), \Omega(\cdot)$, and $\Theta(\cdot)$.
Finally, to define different notions of densities of a generator, we need the following standard notions of densities of subsets of natural numbers:
\begin{definition}[Set densities in $\N$] \label{def:set-densities}
    Let $A = \{a_1, a_2, a_3, \dots\}$ and $B = \{b_1, b_2, b_3, \dots\}$ be subsets of $\N$, with their elements listed in the natural ordering of $\N$.
    \begin{itemize}
        \item The upper density of $A$ in $B$ is $\mu_{\text{up}}(A,B) = \limsup_{n \to \infty} \frac{1}{n}|\{A \cap \{b_1, \dots, b_n\}\}|.$
    \item The lower density of $A$ in $B$ is $\mu_{\text{low}}(A,B) = \liminf_{n \to \infty} \frac{1}{n}|\{A \cap \{b_1, \dots, b_n\}\}|\,.$
    \end{itemize}
\end{definition}
We refer the reader to \cref{app:preliminaries} for a quick refresher on the definition of the limit superior and limit inferior. 
Now, to gain some intuition about $\mu_{\rm up}(\cdot)$ and $\mu_{\rm low}(\cdot)$, consider the following sets 
\[
    E=2\N\coloneqq\sinbrace{2n}_{n\in \N}\,,\quad 
    P=\sinbrace{n\text{~is prime}}_{n\in \N}\,,\quadand 
    F=\cup_{i\in \N}[(2i)!, (2i+1)!]\,.
\]
These languages have the following densities in the universe $U=\N$ :
\begin{align*}
    \mu_{\rm up}(E, U) &= \mu_{\rm low}(E, U) = \frac{1}{2}\,,~~
    \mu_{\rm up}(P, U) = \mu_{\rm low}(P, U) = 0\,,~~
    \mu_{\rm up}(F, U) = 1 ~~\text{and}~~ \mu_{\rm low}(F, U) = 0\,.
\end{align*}
Hence, as we can see, lower and upper densities of $L$ in $\N$ can be different for $L$ and the lower and upper density for one language in $\N$ can be the same or different.
Due to the definition of limit inferior and limit superior, it always holds that $\mu_{\rm low}(L, U) \leq \mu_{\rm low}(L, U) = 0.$

\subsection{Language Generation (without Noise)}
Let $\cL \coloneqq \{L_1, L_2, L_3, \dots\}$ denote a countable collection of languages in $U$. In the language generation problem formalized by \cite{kleinberg2024language}, an adversary fixes a target language $K \in \cL$ and an enumeration of $K$. At round $n\in\N$, the adversary reveals an element $x_n\in K$ according to the chosen enumeration, and the goal of generator is to be able to generate from the target language $K$ eventually. We call this setting the noiseless setting because the adversary commits to a complete enumeration of all the elements from the target language $K$, neither including noisy examples nor omitting examples in its enumeration.
\begin{remark}[Infinite Cardinality of Languages]
    Following \cite{kleinberg2024language} and other prior works, we will always assume that each language $L\in \cL$ has infinite cardinality, as otherwise, when $K=L$, the set of unseen valid elements a generator can output, namely, $K\setminus S_n$, can become empty.
\end{remark}
In this work, we consider three natural notions of generation introduced by \cite{kalavasis2024characterizations, charikar2025pareto, kleinberg2025density}: index-based generation, element-based generation, and set-based generation.
Towards defining this, we first need to define three types of generators: index-based, element-based, and set-based.
\begin{definition}[Index-based Generators; \cite{kleinberg2025density}]
    An index-based generator $\generator=(\generator_n)_{n\in\N}$ maps the observed history to an index $i \in \N$, \ie{}, given examples $x_1, \dots, x_n$ revealed by the adversary till round $n$, the generator outputs $\generator_n(x_1, \dots, x_n) = i_n \in \N$.
\end{definition}

\begin{definition}[Element-based Generators; \cite{kleinberg2025density}]
    
    An element-based generator $\generator=(\generator_n)_{n\in\N}$ maps the observed history to an element $w \in U$ different from the elements revealed by the adversary so far and the previously generated elements, \ie{}, given examples $x_1, \dots, x_n$ revealed by the adversary till round $n$, the generator outputs $\generator_n(x_1, \dots, x_n) = w_n \in U \setminus \left(\{x_1, \dots, x_n\} \cup \{w_1, \dots, w_{n-1}\}\right)$.
\end{definition}

\begin{definition}[Set-based Generators; \cite{kleinberg2025density,kalavasis2025limits,charikar2024facets}]
    A set-based generator $\generator=(\generator_n)_{n\in\N}$ maps the observed history to a set of element $S \subseteq U$ not containing any element revealed by the adversary so far, \ie{}, given examples $x_1, \dots, x_n$ revealed by the adversary till round $n$, the generator outputs $\generator_n(x_1, \dots, x_n) = A_n \subseteq U \setminus \{x_1, \dots, x_n\}$.
\end{definition}
Having defined the three types of generators, we are now ready to define language generation in the limit.
\begin{definition}[Language Generation in the Limit; \cite{kleinberg2024language,kleinberg2025density}]\label{def:generation-in-the-limit}
    A generator $\generator=(\generator_n)_{n\in \N}$ is said to generate in the limit from a language collection $\cL$ if for any $K\in \cL$ and any enumeration of $K$, there exists a finite time $n^\star$ such that for all $n\geq n^\star$, the following holds:
    \begin{itemize}
        \item If $\generator$ is an index-based generator, then $L_{i_n}\subseteq K$;
        \item If $\generator$ is an element-based generator, then $w_i\in K$; and %
        \item If $\generator$ is a set-based generator, then $A_n\subseteq K$.
    \end{itemize}
\end{definition}
\noindent  To draw comparisons between the three notions of generation above, clearly index-based generation is the most restrictive notion, as it implies set-based generation, and set-based generation implies element-based generation. However, simply getting set-based generation in the limit is not interesting, as once element-based generation in the limit is possible, one can simply output a singleton set according to the element-based generator. 
We introduce the set-based generation to study the breadth achieved by such a generator, which we will discuss in \Cref{sec:density}.
Here, element-based generation is identical to the notion of language generation in the limit introduced by \cite{kleinberg2024language}.

Recent results show that, in the noiseless setting, generation in the limit is possible for every countable collection $\cL$ \citep{kleinberg2024language}. In fact, the algorithm of \citep{kleinberg2024language} achieves index-based generation, the most restrictive notion of generation. However, as we will show, index-based generation is no longer sufficient once we consider the noisy setting where the input stream of examples contains noisy examples.

\subsection{Language Generation under Noise} 
It is easy to see that when infinitely many noisy examples are allowed, we need some restriction on the adversary to enable generation. For instance, without any restrictions the adversary can simply enumerate $K$ and another language $K'$ disjoint from $K$ in alternate steps, and we never know whether to generate from $K$ or $K'$.
Towards this, we define the noise rate as follows.
\begin{definition}[Empirical Noise Rate]
    For a language $L$ and an infinite sequence $x_{1:\infty} \in U^\N$, define the (empirical) noise rate up to time $n$ by
    \[
    R(L;x_{1:n}) \;=\; \frac{1}{n}\,\abs{\{t\le n : x_t\notin L\}}.
    \]
\end{definition}
Now we can define noisy enumeration as follows.
\begin{definition}[Regimes of Noisy Enumeration]\label{def:noisy-enumeration}
    Fix a language $L\subseteq U$:
    \begin{itemize}
        \item \textbf{$o(1)$-Noise Enumeration:} An enumeration of $L$ with $o(1)$-noise is an ordered list $x_1, x_2, x_3, \dots$ of elements so that every element of $L$ appears in the list exactly once, and the empirical noise rate satisfies $R(L; x_{1:n}) = o(1)$.
        \item \textbf{$c$-Noise Enumeration:} An enumeration of $L$ with $c$-noise is an ordered list $x_1, x_2, x_3, \dots$ of elements so that every element of $L$ appears in the list exactly once, and there exists $n^\star$ such that for all $n \ge n^\star$, the empirical noise rate satisfies $R(L; x_{1:n}) \le c$.
        \item \textbf{Finite Noise Enumeration:} An enumeration of $L$ with finite noise is an ordered list $x_1, x_2, x_3, \dots$ of elements so that every element of $L$ appears in the list exactly once, and there are at most a finite number of noisy examples in the enumeration, \ie{}, $|\{x_1, x_2, \dots\} \setminus L| < \infty$.
    \end{itemize}
\end{definition}
In this work, we focus on three regimes:
\begin{enumerate}
    \item \textbf{Vanishing noise rate:} The adversary chooses an enumeration of $K$ with $o(1)$-noise. 
    \item \textbf{Constant noise rate:} The adversary chooses an enumeration of $K$ with $c$-noise (for $c\in [0,1]$).
    \item \textbf{Finite noise:} The adversary chooses an enumeration of $K$ with finite noise.
\end{enumerate}
Both the first and the second regimes strictly generalize the third regime, ``finite-noise'' model with only finitely many noisy examples, which has been initiated and studied for generation by \cite{raman2025noisy} and later by \cite{bai2025noise}. 

Similar to the noiseless setting, we say a generator generates from a given language collection $\cL$ in the limit under vanishing noise rate/constant noise rate $c$/finite noise, if for any enumeration of the target language $K$ with $o(1)$-noise/$c$-noise/finite noise, the generator eventually generates from the target language. 

\subsubsection*{Separation of Proper and Improper Learning with a Single Noisy Element}
    Index-based generation can be thought as ``properly'' learning to generate from language $K$, while both element-based and set-based generations correspond to ``improper'' learning as they do not require the generator's outputs to be languages in $\cL$.
    Our next example shows that index-based generation is too restrictive even for a two-language collection when the adversary's enumeration contains $1$ noisy example.

\begin{example} \label{ex:index-failure-under-noise}
    Take $U \coloneqq \N = \{1, 2, 3, \dots\}$, $L_1 \coloneqq \N \setminus \{1\}$, $L_2 \coloneqq \N \setminus \{2\}$, and $\cL = \{L_1, L_2\}$. If the adversary chooses a complete enumeration of $\N$, then clearly this enumeration is an enumeration with $1$ noisy example for both $L_1$ and $L_2$ with one noisy example and thus the adversary is free to declare the target language $K = L_1$ or $K = L_2$. However, outputting any index $i_n \in \{1,2\}$ does not guarantee $L_{i_n} \subseteq K$.
\end{example}
Hence, we mainly focus on element-based and set-based generation in the noisy setting.

\subsection{Language Generation under Omission}

We also consider the setting where the adversaries omit certain elements of the target $K$ in its enumeration, which further complicates things. We remark that the algorithm of \citep{kleinberg2024language} needs that the adversary enumerates all the elements of $K$, and any index-based generator provably fails even when the adversary is allowed to omit one element of $K$ (see \cref{ex:index-failure-under-omission}).
Thus, we again focus on element-based generation and set-based generation in the setting with omission. We consider the following regimes of omissions in adversary's enumeration.

\begin{definition}[Regimes of Enumeration with Omission]\label{def:enumeration-omission}
    Fix a language $L \subseteq U$.
    \begin{itemize}[leftmargin=18pt]
        \item An enumeration of $L$ with finite omissions is an enumeration of $\hat{L} \subseteq L$, s.t., $|L \setminus \hat{L}| < \infty$.
        \item An enumeration of $L$ with $c$-omissions is an enumeration  of  $\hat{L} \subseteq L$, s.t., $\mu_{\text{low}}(\hat{L},L) \ge 1- c$.
        \item An enumeration of $L$ with arbitrary (or infinite) omissions is an enumeration  of  $\hat{L} \subseteq L$, s.t., $|\hat{L}| = \infty$.
    \end{itemize}
\end{definition}

\subsection{Language Generation under Contamination}
    Now that we have defined  both enumeration with noise in \cref{def:noisy-enumeration} and enumeration with omission in \cref{def:enumeration-omission}, we may pair up these definitions and talk about enumeration  under  contamination, including both noise and omissions. Consider the following examples.
    \begin{itemize}[leftmargin=18pt]
        \item An enumeration of $L$ with finite noise and finite omissions is an enumeration with finite noise  of  $\hat{L} \subseteq L$ such that $|L \setminus \hat{L}| < \infty$.
        \item An enumeration of $L$ with $o(1)$-noise and $c$-omissions is an enumeration with $o(1)$-noise  of  $\hat{L} \subseteq L$ such that $\mu_{\text{low}}(\hat{L}, L) \ge 1- c$.
        \item An enumeration of $L$ with $c$-noise and arbitrary omissions is an enumeration with $c$-noise  of  $\hat{L} \subseteq L$ such that $|\hat{L}| = \infty$.
    \end{itemize} 
    Generation in the limit when the adversary chooses an enumeration with both noise and omissions means that the generator still needs to be eventually consistent with $K$ under any such enumeration, as defined in \cref{def:generation-in-the-limit}.
 
    \begin{itemize}[leftmargin=18pt]
        \item An enumeration of $L$ with finite noise and finite omissions is an enumeration with finite noise  of  $\hat{L} \subseteq L$ such that $|L \setminus \hat{L}| < \infty$.
        \item An enumeration of $L$ with $o(1)$-noise and $c$-omissions is an enumeration with $o(1)$-noise  of  $\hat{L} \subseteq L$ such that $\mu_{\text{low}}(\hat{L}, L) \ge 1- c$.
        \item An enumeration of $L$ with $c$-noise and arbitrary omissions is an enumeration with $c$-noise  of  $\hat{L} \subseteq L$ such that $|\hat{L}| = \infty$.
    \end{itemize}

    \subsubsection*{Separation of Proper and Improper Learning with a Single Omission}
        Our next result shows that even for simple collections consisting of two languages, and very restricted adversaries that omit only a single element from $K$, index-based generation is provably impossible.
        \begin{example}
    \label{ex:index-failure-under-omission}
    Consider the same example as in \cref{ex:index-failure-under-noise}: $L_1 \coloneqq \N \setminus \{1\}$, $L_2 \coloneqq \N \setminus \{2\}$, and $\cL = \{L_1, L_2\}$. If the adversary enumerates $3, 4, 5, \dots$, this enumeration is a valid enumeration with $1$ element omitted for both $L_1$ and $L_2$. However, neither index $i_n \in \{1,2\}$ guarantees $L_{i_n} \subseteq K$.
\end{example}
\begin{remark}
    \label{rem:singleInvalidExampleFailure}
    
    We note that \cite{bai2025noise}'s approach for generating in the limit with omissions (with no noise) is fragile to noise.
    Their approach is simple: they show that if a generator generates $\cL$ in the limit with the property that $\generator$ starts generating after $n^\star$ iterations under an enumeration without noise or omissions, where $n^\star$ can depend on the choice of the target language $K\in \cL$ but does \textit{not} depend on the enumeration of $K$ chosen by the adversary, then $\generator$ (without any changes) already generates in the limit with arbitrary omissions.
    A generator with this property all countable collections was provided by \cite{charikar2024facets}.
    This approach, however fails when there is noise in the enumeration.
    We demonstrate the failure of this approach using just two noisy examples. Consider the collection $\cL = \set{K, L_1, L_2, \ldots}$ where
\begin{align*}
K &= \N, \\
L_i &= \set{-1, -2} \cup \set{1, 2, \ldots, i-1} \cup T\,, \quad \text{for each } i \in \N\,, \\
T &= \set{\ldots, -102, -101, -100}\,.
\end{align*}
Now suppose target language is $K = \N$ and the adversary's enumeration: $E = (-2, -1, 1, 2, 3, \ldots)$.
Note that $E$ contains exactly two noisy examples ($-2$ and $-1$) with no omissions.

\paragraph{Why their algorithm fails?} At iteration $t \geq 2$, the algorithm from \cite{charikar2024facets} incorrectly generates elements from set $T$ rather than from $K$. This occurs because:
The algorithm examines the first $n$ languages consistent with the observed data $\set{-1, -2, 1, 2, \ldots, n-2}$.
Among these first $n$ languages, only $L_{n-1}$ remains consistent with the observed data.
The algorithm then generates the smallest unseen element from $L_{n-1}$, which is the smallest unseen element from $T$.
Since $T \cap K = \emptyset$, the algorithm never generates elements of the target language $K$.
This demonstrates that noise-free generation strategies cannot be directly applied to noisy settings, even with minimal noise.
\end{remark}

\subsection{Density of Generators}\label{sec:density}

Besides requiring that a generator to generate in the limit so that hallucinations eventually stop, it is important to ensure that the generator has learnt a non-trivial fraction of the target language and can generate novel elements and does not limit the generation process to a restrictive subset of the target language. This motivates the study of coverage or breadth of a language generation algorithms.
Towards that end, \cite{kalavasis2025limits,charikar2024facets,kalavasis2024characterizations} defined some notions of ``breadth'' which treat the generating algorithm as a set-based generator, and require that its output misses only finitely many elements of the target.\footnote{To be precise, these works considered more notions of breadth, but they all share a similar viewpoint.}

For convenience, we will assume the universe of all possible strings $U = \N$, since for any countable set, we may identify its elements with natural numbers using a fixed canonical ordering of its elements.

\cite{kleinberg2025density} introduced the following idea for describing the breadth of a generation algorithm. For any two sets $A, B \subseteq \N$, there are two natural ways to describe what fraction of elements of $B$ is covered by $A$, namely the upper density $\mu_{\text{up}}(A,B)$ and the lower density $\mu_{\text{low}}(A,B)$ as defined in \cref{def:set-densities}. From now on, whenever we talk about densities, we assume the ground set $U = \N$. We may then define the following densities achieved by an element-based generation algorithm.

\begin{definition}[Element-based Density]\label{def:element-based-density}
    Let $\cL$ be a countable collection of languages in $\N$. Let $\generator = (\generator_n)_{n\in \N}$ be an element-based generator. Let $K$ be the target language, and $W = \{w_1, w_2, w_3, \dots\}$ be the infinite set of elements that the generator ever outputs in response to the adversary's enumeration.

    The element-based generator $\generator$ achieves element-based upper density $\rho$ if $\mu_{\text{up}}(W, K) \ge \rho$. It achieves element-based lower density $\rho$ if $\mu_{\text{low}}(W, K) \ge \rho$.
\end{definition}
While \citep{kleinberg2025density} also defined a density measure for index-based  generators , since we argued that index-based generation is too restrictive for the noisy generation model, we instead consider the following density for set-based  generators .

\begin{definition}[Set-based Density]
    \label{def:set-based-density}
    Let $\cL$ be a countable collection of languages in $\N$. Let $\generator = (\generator_n)_{n\in \N}$ be a set-based generator. Let $K$ be the target language, and $A_n$ be the set generated in round $n$.

    The set-based generator $\generator$ achieves set-based upper density $\rho$ if $\limsup_{n \to \infty} \mu_{\text{low}}(A_n, K) \ge \rho$. It achieves set-based lower density $\rho$ if $\liminf_{n \to \infty} \mu_{\text{low}}(A_n, K) \ge \rho$.
\end{definition}
While we used limsup and liminf of the lower set-density to define set-based upper density and set-based lower density, one can also use limsup and liminf of the upper set-density to define similar notions. The arguments in the paper can be easily adapted to handle these other notions, but we omit the discussion of them for conciseness of the presentation.

Finally, we remark that there is an even stronger notion of breadth for set-based generation, called approximate breadth by \cite{kalavasis2025limits}. A set-based generator is said to achieve approximate breadth if there exists a time $n^\star$, such that for all $n \ge n^\star$, the output set $S_n$ satisfies $|K \setminus (S_n \cup \{x_1, \dots, x_n\})| < \infty$. Note that any set-based generator that achieves approximate breadth also achieves set-based lower density $1$.

\subsection{What Type of Access Does the Generator Have To $\cL$?}
Before we move on to discuss our main results, let us take a moment and clarify what the learner can do and cannot do, \ie{}, what is within fair-game of the language generation model.

\paragraph{What the learner cannot do:} The learner cannot get access to or make any query about the unknown target language $K$. For example, it cannot ask ``what is the next smallest element of $K$ not enumerated by the adversary so far.''

Moreover, the learner cannot query what the adversary will do in the future. For example, at round $n$, it cannot ask ``what will the adversary enumerate at time $n'$'' for any $n' > n$. At any round, it only gets to see the elements revealed by the adversary so far.

\paragraph{What the learner can do:} We assume the learner is allowed to make any query about the language collection $\cL$, as long as it does not involve accessing the unknown $K$. For example, the following are some well-studied oracles considered in the previous works:
    \begin{itemize}
        \item (Membership Oracle) \, Given index $i$ of some language $L_i$ and element $w \in U$, the oracle returns $\mathds{1}\{w \in L_i\}$.

        \item (Subset Oracle)\, Given indices $i,j$ of two languages $L_i$ and $L_j$, the oracle returns $\mathds{1}\{L_i \subseteq L_j\}$.
        \item (Intersection Oracle)\, Given a finite collection of indices $I$, the oracle returns $\mathds{1}\{|\cap_{i \in I} L_i| = \infty\}.$
    \end{itemize}
    In this work, we will also consider a few more oracles 
    \begin{itemize}
        \item (Density Oracle)\, Given languages $L, K$, compute $\mu_{low}(L, K)$.

        \item (Density Rate Oracle)\, Given languages $L, K$ and $\eps>0$,
        compute the number of elements $m^\star$ until the empirical density $\frac{L\cap \sinbrace{\kappa_1, \dots, \kappa_m}}{m}\geq \frac{\mu_{low}(L, K)}{1+\eps}$ is a good approximation of the true density for all $m\geq m^\star$.
    \end{itemize}
    When we show certain generators exist, we will specify the oracles they need.

\section{Algorithmic Templates} 
\label{sec:templates}
In this section we describe two algorithmic templates that we will use extensively in our work.
    
    \subsection{Priority-Based Intersection Meta Algorithm}
    \label{sec:templates:priority}
    Before presenting our first algorithmic template,
    it is instructive to first consider an algorithm due to \citet{charikar2024facets} for generation in the limit without contamination.

    \paragraph{The \cite{charikar2024facets} Algorithm.}
    At the $n$-th step,
    this algorithm considers the input $x_{1:n}$ and the first $n$ languages $L_1, \dots, L_n$ of a given collection $\cL$.
    We first filter for \emph{consistent} languages,
    \ie{}, languages $L_{i_n(1)}, L_{i_n(2)}, \dots$ containing all strings $x_1, \dots, x_n$
    for $i_n(j) < i_n(j+1)$.
    Then, the algorithm computes the largest intersection in the order of indices $i_n(1), i_n(2), \dots$ such that the intersection is infinite,
    \ie{}, take the intersection of the first $J_n$ filtered languages where $J_n$ is the largest integer such that \mbox{$\card*{\cap_{j=1}^{J_n} L_{i_n(j)}} = \infty$}.
    We then output an unseen element from this intersection.
    
    Suppose the target language $L_{i^\star}$ has index $i^\star$ and consider some $L_i$ for $i<i^\star$.
    Either $L_i\supseteq L_{i^\star}$,
    or there is some string $x\in L_{i^\star}\setminus L_i$ that is enumerated at some point,
    so that $L_i$ is no longer consistent.
    Since $L_{i^\star}$ is always consistent, and the consistent languages ordered before $L_{i^\star}$ are eventually supersets,
    the intersection is always taken over a collection containing $L_{i^\star}$
    so that the algorithm is guaranteed to generate from some infinite subset of $L_{i^\star}$.

    Our template algorithm (\Cref{alg:intersection-meta}) makes two main changes to the \cite{charikar2024facets} algorithm:
    \begin{enumerate*}[(1)]
        \item we change the order when taking intersections based on the priority of a language
    and 
        \item we introduce different stopping rules when taking the intersection.
    \end{enumerate*}
    
    \paragraph{Re-ordering by Priorities.}
    In the presence of contamination,
    we cannot completely eliminate candidate languages $L_i$ from consideration if we observe a string $x_n\notin L_i$,
    since it is unclear if $x_n$ is actually a member of the target language or just noise.
    Instead, at step $n$ of our algorithm,
    we assign a priority $P_i^{(n)}\in \N$ to the language $L_i$.
    Although we state present priorities abstractly,
    we should think of incrementing $P_i^{(n)}$ as penalizing a candidate language $L_i$ whenever we observe $x_n\notin L_i$.
    Then, we take the intersection in increasing order of priorities.
    Ideally, only supersets of the target language $K$ can be ordered before $K$,
    and the intersection we compute will always be a subset of $K$.
    This recovers the \cite{charikar2024facets} algorithm when we take the priorities to be $P_i^{(n)} = i$ if $L_i$ is consistent
    and $P_i^{(n)} = \infty$ once we observe $x_m\notin L_i$ for some $m\leq n$.

    \paragraph{Earlier Stopping.}
    If we would like to generate with density,
    we must avoid taking the intersection with too many languages,
    or we can end up generating from an infinite subset of the target language with density 0.
    We thus define an abstract stopping rule that will be instantiated differently depending on the desired guarantee.

    \paragraph{Pseudocode.}
    The pseudocode of our meta-algorithm is presented in \Cref{alg:intersection-meta}.
    Here we summarize the key notation for the reader's convenience.
    \begin{enumerate}
        \item The priority function $P(L_i, x_{1:n})$ takes a language $L_i$ and the current enumeration $x_{1:n}$ 
        and outputs a priority $P_i^{(n)}\in \N \cup \{\infty\}$.
        \item $i_n(j)$ is the index of the $j$-th ranked language at step $n$,
        \ie{}, $L_{i_n(1)}$ has the smallest priority value $P_{i_n(1)}^{(n)}$ (and thus ``highest priority'') at step $n$
        and $L_{i_n(2)}$ has the second smallest priority value, etc.
        \item The stopping function $J(L_{i_1}, \dots L_{i_n}, x_{1:n})$ takes a sequence of languages $L_{i_j}$ and the current enumeration $x_{1:n}$ and outputs a stopping index $J_n\in \N$.
    \end{enumerate}
    \begin{algorithm}[bht!]
        \caption{Meta-Algorithm for Generation with Noise}
        \label{alg:intersection-meta}
        \begin{algorithmic}[1]
        \Require Countable {collection} $\cL=\{L_1,L_2,\dots\}$; 
        priority function $P:\cL\times U^*\to \N \cup \{\infty\}$
        stopping function $J: \cL^*\times \Omega^*\to \N$;
        enumeration $x_{1:\infty}$
        \State Let $S_n \gets \sinbrace{x_1, \dots, x_n}$ be the set of examples seen in the first $n$ steps
        \State Let $W_{n-1} \gets \sinbrace{w_1, \dots, w_{n-1}}$ be the set of strings output before the $n$-th step
        \vspace{2mm}
        \For{$i=1, 2, \dots, n$}
            \State Compute priority $P_i^{(n)}$ at step $n$ for language $L_i$
        \EndFor
        \vspace{2mm}
        \State Re-order $\inbrace{L_1, \dots, L_n}$ in increasing priority, tie-breaking by index,
        as {$\sinbrace{L_{i_n(1)}, \dots, L_{i_n(n)}}$,
        \ie{}, for each $j\in [n-1]$,
        ensure either $P_{i_n(j)}^{(n)} < P_{i_{n}(j+1)}^{(n)}$
        or $P_{i_n(j)}^{(n)} = P_{i_{n}(j+1)}^{(n)}$ and $i_n(j) < i_{n}(j+1)$}.
        \State Compute the stopping index $J_n$.
        \State Output {${\bigcap}_{j\leq J_n} L_{i_n(j)}$}. \Comment{Set-based output}
        \State Output any {$w_n\in {\bigcap}_{j\leq J_n} L_{i_n(j)} \setminus (S_n\cup W_{n-1})$}. \Comment{Element-based output}
        \end{algorithmic}
    \end{algorithm}

    \paragraph{Analysis.}
    One key property of the prefix-based intersection meta algorithm (\Cref{alg:intersection-meta}) is that the languages that are ranked higher than the target language will eventually stabilize,
    \ie{}, the list of such languages eventually stops changing for sufficiently large $n$.
    This is formalized in \Cref{lem:meta-algorithm:prefix-stabilizes} below.
    We should think of $p\in \N$ below as an upper bound on the priority of the target language.
    \begin{lemma}[Prefix Priority Stabilization]\label{lem:meta-algorithm:prefix-stabilizes}
        Suppose \Cref{alg:intersection-meta} is executed with input enumeration $x_{1:\infty}$.
        Let $P_i^{(n)}\in \N \cup \{\infty\}$ be the priority of language $L_i$ computed in step $n$.
        Assume $P_i^{(n)}$ is non-decreasing in $n$ and lower bounded by $i$,
        \ie{}, $i\leq P_i^{(n)}\leq P_i^{(n+1)}$ for all $n\geq 1$.
        For any $p \in \N$,
        define
        \begin{align*}
            P_i^\infty &\coloneqq \lim_{n\to \infty} P_i^{(n)}\,, 
            \qquad
            \cL(p) \coloneqq \sinbrace{L_i: P_i^{\infty}\leq p}\,.
        \end{align*}
        Then there is a step $n^\star$ such that for all $n\geq n^\star$,
        \begin{enumerate}[(a)]
            \item $P_i^{(n)}\leq p$ for all $L_i\in \cL(p)$
            \item $P_i^{(n)} > p$ for all $L_i\notin \cL(p)$
            \item $P_i^{(n+1)} = P_i^{(n)}$ for all $L_i\in \cL(p)$
        \end{enumerate}
    \end{lemma}
    Before proving \Cref{lem:meta-algorithm:prefix-stabilizes},
    we provide some intuition for the quantities involved.
    By the assumption that $P_i^{(n)}$ is a non-decreasing integer sequence indexed by $n$,
    we know that the limit $P_i^\infty$ always exists.
    Since the sequence takes on discrete values,
    we can further deduce that either $P_i^{(n)} = P_i^\infty$ after some finite time $n$,
    or $P_i^\infty = \infty$ and the priority diverges.
    Assume for now that the priority $P_{i^\star}^\infty < \infty$ of the target language $L_{i^\star}$ does not diverge.
    Then $\cL(p)$ for $p\coloneqq P_{i^\star}^\infty$ is consists of the languages whose priorities are always bounded above by the limiting priority of the target language.
    \Cref{lem:meta-algorithm:prefix-stabilizes} states that languages in $\cL(p)$ will eventually be ordered before any language not in $\cL(p)$.
    Thus when we take the intersection in \Cref{alg:intersection-meta} in increasing order of priorities,
    we will always consider languages in $\cL(p)$ before all other languages.
    Note that the definition of $\cL(p)$ does not depend on a particular step $n$.
    \Cref{lem:meta-algorithm:prefix-stabilizes} simplifies some of the analyses of our algorithms in that it suffices for us to analyze properties of $\cL(p)$ rather than analyzing any particular stage of our algorithm.

    We are now ready to prove \Cref{lem:meta-algorithm:prefix-stabilizes}.
    \begin{proof}[Proof of \Cref{lem:meta-algorithm:prefix-stabilizes}]
        First,
        we note that $P_i^\infty\geq P_i^{(n)}\geq i$ for any $n\geq i$.
        Hence $\cL(p)\subseteq \sinbrace{L_1, \dots, L_p}$ and $P_i^{(n)} > p$ for all $n\geq i> p$.
        Let $n^\star\geq 1$ be any integer such that
        for every $i\in [p]$,
        either 
        \begin{enumerate}[(a)]
            \item $P_i^{(n)} = P_i^\infty\leq p$ for all $n\geq n^\star$, or
            \item $P_i^{(n)} > p$ for all $n\geq n^\star$.
        \end{enumerate}
        This is guaranteed to exist since there are only finitely many limits,
        each of which is monotonic.
        The result follows.
    \end{proof}
    Our goal will be to design a priority function for which the target language $L_{i^\star}$ has bounded priority $p = P_{i^\star}^\infty < \infty$
    so that $L_{i^\star}\in \cL(p)$.
    Recall the notation $\Cl(\cL') \coloneqq \cap_{L\in \cL'} L$ for a collection $\cL'$ of languages.
    If the stopping rule also ensures that the prefix class $\cH_n\coloneqq \sinbrace{L_{i_n(j)}: j\leq J_n}$
    eventually contains $\cL(p)$
    and $\card{\Cl(\cH_n)} = \infty$,
    then the algorithm generates from $\cL(p)\sset L_{i^\star}$.
    This observation is summarized below.

    \begin{corollary}[Sufficient Condition for Generation]\label{cor:meta-algorithm:generation}
        Suppose the assumptions of \Cref{lem:meta-algorithm:prefix-stabilizes} hold.
        If in addition the prefix class $\cH_n\coloneqq \sinbrace{L_{i_n(j)}: j\leq J_n}$
        satisfies $\cH_n\supseteq \cL(p)$
        and $\card{\Cl(\cH_n)} = \infty$
        for all sufficiently large $n$,
        then \Cref{alg:intersection-meta} generates from $\Cl(\cL(p))$.
    \end{corollary}

    \paragraph{Instantiations of \Cref{alg:intersection-meta}.}
    In this paper, we make use of the meta algorithm \cref{alg:intersection-meta} to prove \cref{thm:vanishing-noise-generation}, \cref{thm:constant-noise-characterization}, \cref{thm:vanishing-noise-set-density-characterization}, \cref{thm:constant-noise-set-density-characterization}, and \cref{thm:improper:generation:inf-inf-set-density}. We now illustrate a few examples of instantiations of \Cref{alg:intersection-meta}.
    
    For our algorithms that generate in the limit under vanishing noise (\Cref{sec:generation:vanishing-noise}) and constant noise (\Cref{sec:generation:constant-noise}),
    we instantiate the priority $P_i^{(n)}$ of language $L_i$ at time step $n$ so that,
    roughly speaking,
    $P_i^{(n)}$ increases every time the empirical noise rate $R(L_i; x_{1:n})$ exceeds some threshold $c_i$.
    For a vanishing noise rate,
    we can set $c_i$ to be any positive sequence with small sum $\sum_i c_i < \eps$.
    This ensures that we prioritize languages that have very small empirical error
    and the total error is at most an $\eps$-fraction of the input enumeration.
    For a constant noise rate $c$,
    we set $c_i=c$ for all $i$,
    so that we prioritize languages whose nose rate eventually falls below $c$.
    The stopping condition in these two cases is identical to the \cite{charikar2024facets} algorithm.

    Our algorithms for generation with density under $M$-bounded enumerations with vanishing noise (\Cref{sec:beyond}) are similar to the algorithm for generation under vanishing noise with two key distinctions.
    In addition to incrementing the priority $P_i^{(n)}$ of language $L_i$ at step $n$ when the empirical noise rate exceeds some threshold,
    we also increment $P_i^{(n)}$ if the current input violates the $M$-boundedness condition with respect to $L_i$.
    This ensures we de-prioritize languages for which the input enumeration cannot be $M$-bounded.
    We also require a more involved stopping rule which ensures that the intersection we take is dense with respect to every language over which the intersection is taken over.
    This ensures that if the target language is included in the intersection,
    then we will always generate from some dense subset of the target language.

    \subsection{Finite Expansion Sub-Routine}
    \label{sec:templates:expansion}
    Next, we present a conceptually simple, but quite useful subroutine that we will use throughout, especially in settings where the adversary is restricted to finite contamination of the input. To show why this sub-routine is useful, we state a preliminary result in this section showing that it can be used to achieve generation in the limit, under finite contamination of the input stream. In subsequent sections we will illustrate more involved use-cases of this idea.
    
    \begin{algorithm}[bht!]
        \caption{Sub-Routine for Finite Expansion}
        \label{alg:finite-expansion-routine}
        \begin{algorithmic}[1]
        \Require Countable {collection} $\cL=\{L_1,L_2,\dots\}$
        \State  Output $\Tilde{\cL}  \coloneqq \{L_{A,B} \coloneqq L \cup A \setminus B : L \in \cL, A \subseteq U \setminus L, B \subseteq L, \abs{A} < \infty, \abs{B} < \infty\}\,.$
        \end{algorithmic}
    \end{algorithm}

    \begin{lemma}\label{lem:expansion-subroutine}
        Suppose $\generator$ is an element-based (set-based) generator that generates from arbitrary countable collection in the limit under enumeration without noise or omissions. Let $\cL$ be a countable collection, and $\Tilde{\cL}$ be the expanded countable collection constructed from \cref{alg:finite-expansion-routine}. Suppose $x_{1:\infty}$ is an enumeration of some target language $K \in \cL$ with finite noise and finite omissions . Then, if we apply $\generator$ to $\Tilde{\cL}$, it will generate from $K$ in the limit under the aforementioned enumeration. 
    \end{lemma}

    \begin{proof}
        We start with element-based generators. Notice that, since $x_{1:\infty}$ contains finitely many noisy examples
        and omits finitely many elements from $K$, it corresponds to an enumeration of 
        $K_{A,B} \in \Tilde{\cL}$ \emph{without noise or omissions}, where $A$ is the set of noisy elements and $B$ is the set of omitted elements. Thus, $\generator$
        executed on $x_{1:\infty}$ and $\Tilde{\cL}$ generates in the limit from $K_{A,B}$. Hence, there exists some $n^* \in \N$ such that for all $n \geq n^*$ its output is a fresh element of $K_{A,B}.$ Next, notice that since $\abs{A} < \infty$ it must be the case that there is some $n'$ such that for all $n \geq n'$ the output of the algorithm is a fresh element of $K_{A,B} \setminus A.$ Since $(K_{A,B} \setminus A) \subseteq K$ we have shown that $\generator$ generates $K$ in the limit.

        Similarly, if $\generator$ is a set-based generator that generates in the limit from $K_{A,B}$, there exists some $n^\star \in \N$ such that for all $n \ge n^\star$, the output set $A_n$ is a subset of $K_{A,B}$. Again, since $\abs{A} < \infty$ and $x_{1:\infty}$ is an enumeration of $K_{A,B} = K \cup A \setminus B$ without noise or omissions, there exists some $n'$ such that for all $n \geq n'$, we have $K_{A,B} \setminus \inbrace{x_1, \dots, x_n} \subseteq K_{A,B} \setminus A \subseteq K$. Since the output set satisfies $A_n \subseteq K_{A,B} \setminus \inbrace{x_1, \dots, x_n} \subseteq K$ for all $n \ge \max\inbrace{n^\star, n'}$, $\generator$  generates $K$ in the limit.
    \end{proof}
    To illustrate the usefulness of this sub-routine and the above result, we get as an immediate corollary that
    there exists an algorithm that generates in the limit under finite contamination, using only \emph{membership oracle access}
    to the underlying language collection $\cL.$ This resolves in the affirmative the open question of \citet{raman2025noisy}, who asked for generators with membership oracle access to $\cL$ and finite noise (and without omissions). First, we state 
    a result from \citet{kleinberg2024language} that we utilize in our proof.

    \begin{lemma}[Generation with Membership Oracle Access under Uncontaminated Input \citep{kleinberg2024language}]\label{lem:km-membership-algo}
        There is an algorithm that generates in the limit for every countable collection $\cL$ with only membership oracle 
        access to $\cL$ when the adversarial stream does not include any contaminated examples.
    \end{lemma}
    We are now ready to state and prove our result.

    \begin{corollary}[Generation with Membership Oracles under Finite Omissions]\label{cor:generation-membership-oracle-finite-omission}
        There is an algorithm that generates in the limit given only membership oracle access to $\cL$ under adversarial streams that contain finite amount of contamination.
    \end{corollary}

    \begin{proof}
        We will instantiate \cref{alg:finite-expansion-routine} with the algorithm of \citet{kleinberg2024language}, stated in \cref{lem:km-membership-algo}. Notice that if we have membership oracle access to $\cL$, we can implement membership oracle access to $\Tilde{\cL}$, by keeping track of the way we enumerate languages in $\Tilde{\cL}.$ In other words, if $\tilde L_i$ is the $i$-th language of $\Tilde{\cL}$ we can maintain a mapping to $L_{j,A,B} \in \cL,$ for some $j, A, B.$ Then, the query ``is $u_\ell \in \tilde L_i$'' is answered by checking if $u_\ell \in L_j \cup A \setminus B;$ since we have membership access to $L_j$, it suffices to check if $u_\ell \in L_j$ and then, since we have explicit description of the (finite) sets $A, B$, we can check if $u_\ell \in A, u_\ell \in B.$ Thus, since we can implement a membership oracle for $\Tilde{\cL}$ we can run the algorithm from \cref{lem:km-membership-algo} on it. Then, \cref{lem:expansion-subroutine} guarantees that this algorithm achieves generation in the limit.
    \end{proof}
    
    \noindent Later in the paper, we will also use the finite expansion sub-routine in \cref{alg:finite-expansion-routine} to prove \cref{thm:element-based-density-finite-noise-omissions}, a result about element-based density under finite noise and omissions.
    
    We will also consider a slightly different subroutine that only expands by adding a finite set to the existing languages, as opposed to the one in \cref{alg:finite-expansion-routine} that allows both addition of finitely many elements and deletion of finitely many elements.

    \begin{algorithm}[bht!]
        \caption{Alternative Sub-Routine for Finite Expansion}
        \label{alg:finite-expansion-routine-addition}
        \begin{algorithmic}[1]
        \Require Countable {collection} $\cL=\{L_1,L_2,\dots\}$
        \State  Output $\Tilde{\cL}  \coloneqq \{L_{A} \coloneqq L \cup A : L \in \cL, A \subseteq U \setminus L, \abs{A} < \infty\}\,.$
        \end{algorithmic}
    \end{algorithm}

    \begin{lemma}\label{lem:expansion-subroutine-addition} Suppose $\generator$ is an element-based (set-based) generator that generates from arbitrary countable collection in the limit under enumeration without noise and with potentially infinite omissions. Let $\cL$ be a countable collection, and $\Tilde{\cL}$ be the expanded countable collection constructed from \cref{alg:finite-expansion-routine-addition}. Suppose $x_{1:\infty}$ is an enumeration of some target language $K \in \cL$ with finite noise and potentially infinite omissions. Then, if we apply $\generator$ to $\Tilde{\cL}$, it will generate from $K$ in the limit under the aforementioned enumeration. 
    \end{lemma}

    \begin{proof}
        The proof is almost identical to that of \cref{lem:expansion-subroutine}.

        We start with element-based generators. Notice that, since $x_{1:\infty}$ contains finitely many noisy examples
        outside of $K$, it corresponds to an enumeration of 
        $K_{A} \in \Tilde{\cL}$ \emph{without noise} and with potentially infinite omissions, where $A$ is the set of noisy elements and every element of $A$ is enumerated in $x_{1:\infty}$. Thus, $\generator$
        executed on $x_{1:\infty}$ and $\Tilde{\cL}$ generates in the limit from $K_{A}$. Hence, there exists some $n^* \in \N$ such that for all $n \geq n^*$ its output is a fresh element of $K_{A}.$ Next, notice that since $\abs{A} < \infty$ it must be the case that there is some $n'$ such that for all $n \geq n'$ the output of the algorithm is a fresh element of $K_{A} \setminus A = K$. This shows that $\generator$  generates $K$ in the limit.

        Similarly, if $\generator$ is a set-based generator that generates in the limit from $K_{A}$, there exists some $n^\star \in \N$ such that for all $n \ge n^\star$, the output set $A_n$ is a subset of $K_{A}$. Again, since $\abs{A} < \infty$ and every element of $A$ is enumerated in $x_{1:\infty}$, there exists some $n'$ such that for all $n \geq n'$, we have $K_{A} \setminus \inbrace{x_1, \dots, x_n} \subseteq K_{A} \setminus A =  K$. Since the output set satisfies $A_n \subseteq K_{A} \setminus \inbrace{x_1, \dots, x_n} \subseteq K$ for all $n \ge \max\inbrace{n^\star, n'}$, $\generator$  generates $K$ in the limit.
        
    \end{proof}
    We will later utilize the above sub-routine and \cref{lem:expansion-subroutine-addition} to show a result similar to the upper density guarantee for indexed-based generation under enumeration without noise or omissions proved by \citep{kleinberg2025density}. In \cref{thm:finite-noise-upper-density}, we will show that under finite noise and infinite omissions, there exists set-based generator that achieves the best possible set-based upper density.

\section{Generation in the Limit under Contamination}
\label{sec:generation}
    In this section, we provide several results in the model where the adversary can contaminate the input enumeration with noise and omissions.

\subsection{Generation with Vanishing Noise Rate and Arbitrary Omissions}\label{sec:generation:vanishing-noise}
    Our first result shows that under the mild assumption that the noise rate of the enumeration converges to zero, all countable collections become generable. 
    
    We remark that even though the noise rate converges to zero, the amount of noisy examples the adversary can add to the enumeration is still infinite. 
    Hence, this result significantly strengthens the result of \cite{raman2025noisy}, which only showed that all collections remain generable when the adversary introduces a \textit{finite} number of noisy examples. 

    \begin{theorem}\label{thm:vanishing-noise-generation}
        There is a generator $\generator$ that, 
        for any collection $\cL$, 
        target language $K\in \cL$, 
        given an enumeration of $K$ with $o(1)$-noise and arbitrary omissions,
        {$\generator$} generates in the limit from $K$. 
    \end{theorem}

    \paragraph{Pseudocode.}
    The algorithm for \Cref{thm:vanishing-noise-generation} is \Cref{alg:intersection} instantiated with threshold parameters $c_i\coloneqq \frac{1}{2^{i+1}}$.

    We can view our algorithm as a sub-template of our generic meta-algorithm \Cref{alg:intersection-meta},
    where the priorities are defined by thresholds
    and the stopping rule is based on the size of the intersection of the languages in the ordered prefix.
    The changes from \Cref{alg:intersection-meta} are highlighted in blue.
    
    \begin{algorithm}[tbh!]
        \caption{Algorithm for \Cref{thm:constant-noise-characterization,thm:vanishing-noise-generation}}
        \label{alg:intersection}
        \begin{algorithmic}[1]
        \Require Countable  collection  $\cL=\{L_1,L_2,\dots\}$;  thresholds $c_1, c_2,\ldots\in (0, 1)$; enumeration $x_{1:\infty}$
        \State Let $S_n \gets \sinbrace{x_1, \dots, x_n}$ be the set of examples seen in the first $n$ steps
        \State Let $W_{n-1} \gets \sinbrace{w_1, \dots, w_{n-1}}$ be the set of strings output before the $n$-th step
        \vspace{2mm}
        \For{$i=1, 2, \dots, n$}
            \State {\color{blue} Compute the smallest $N_i^{(n)}$ such that $L_i$ is consistent with $x_{1:m}$ for each $N_i^{(n)}\leq m\leq n$
            \[
                N_i^{(n)} \gets
                \begin{cases}
                    \min\inbrace{N\geq 1: \forall m\in \sinbrace{N, \dots, n}, R(L_i; x_{1:m})\leq c_i}\,, &R(L_i; x_{1:n})\leq c_i\,, \\
                    n+1\,, &\text{else}\,.
                \end{cases}
            \]}
            \State {\color{blue} Assign the language $L_i$ a priority of $P_i^{(n)}\gets i + N_i^{(n)}$}
        \EndFor
        \vspace{2mm}
        \State Re-order $\inbrace{L_1, \dots, L_n}$ in increasing priority, tie-breaking by index,
        as {$\sinbrace{L_{i_n(1)}, \dots, L_{i_n(n)}}$,
        \ie{}, for each $j\in [n-1]$,
        ensure either $P_{i_n(j)}^{(n)} < P_{i_{n}(j+1)}^{(n)}$
        or $P_{i_n(j)}^{(n)} = P_{i_{n}(j+1)}^{(n)}$ and $i_n(j) < i_{n}(j+1)$}.
        \State {\color{blue} Compute the largest index $J_n$ such that the intersection in the re-ordering up to $L_{i_n(J_n)}$ is infinite 
        \[
            J_n \gets \sup\inbrace{\bar j\geq 1: \card*{\bigcap\nolimits_{j=1}^{\bar j} L_{i_n(j)}} = \infty}\,.
        \]}
        \State Output any $w_n\in \bigcap_{j\leq J_n} L_{i_n(j)} \setminus (S_n\cup W_{n-1})$.
        \end{algorithmic}
    \end{algorithm}

    \paragraph{Analysis.}
    \Cref{alg:intersection} generates under both vanishing noise, 
    omissions,
    or constant noise (assuming the $c$-constant noise generation property holds)
    when instantiated with an appropriate choice of thresholds.
    
    Before presenting a useful lemma,
    recall the following notation related to the priorities from our generic meta-algorithm \Cref{alg:intersection-meta}.
    For a fixed $p\geq 1$,
    \begin{align*}
        P_i^\infty &\coloneqq \lim_{n\to \infty} P_i^{(n)}\,, 
        \qquad
        \cL(p) \coloneqq \sinbrace{L_i: P_i^{\infty}\leq p}\,, 
        \qquadand
        \Cl(\cL(p)) \coloneqq \bigcap_{L\in \cL(p)} L\,.
    \end{align*}
    \begin{lemma}[Sufficient Condition to Generate using \Cref{alg:intersection}]\label{lem:constant-vanishing-noise:sufficient-generation}
        Consider executing \Cref{alg:intersection} with some arbitrary choice of thresholds $c_i$.
        Fix $p\geq 1$ and suppose that the set $\cL(p)$ satisfies $\card{\Cl(\cL(p))} = \infty$.
        Then \Cref{alg:intersection} generates from $\Cl(\cL(p))$ in the limit.
    \end{lemma}

    \begin{proof}
        We argue using \Cref{cor:meta-algorithm:generation}.
        To do so,
        we first note that the priorities $P_i^{(n)}$ are non-decreasing by construction.
        Define the prefix class $\cL_n = \sinbrace{L_{i_n(j)}: j\leq J_n}$.
        By the definition of \Cref{alg:intersection},
        this set always has infinite intersection.
        Thus to apply \Cref{cor:meta-algorithm:generation}, 
        it suffices to show that $\cL(p)\sset \cL_n$ 
        for all sufficiently large $n$.

        To see this,
        first note that the priorities $P_i^{(n)}$ of every $L_i$ is non-decreasing by definition.
        Hence by \Cref{lem:meta-algorithm:prefix-stabilizes},
        there is some $n^\star\in \N$ such that $\cL(p)$ is ordered before all other languages for every $n\geq n^\star$.
        But $\card{\Cl(\cL(p))}=\infty$ by assumption
        so that $\cL_n$ must contain $\cL(p)$ for $n\geq n^\star$ by the definition of the stopping rule in \Cref{alg:intersection}.
    \end{proof}

    We are now ready to prove \Cref{thm:vanishing-noise-generation}.
    \begin{proof}[Proof of \Cref{thm:vanishing-noise-generation}]
        Similar to the case of finite noise rate,
        our strategy is to apply \Cref{lem:constant-vanishing-noise:sufficient-generation} to argue that \Cref{alg:intersection} generates from $\cL$,
        but this time with threshold parameters $c_i \coloneqq \frac{1}{2^{i+1}}$ .

        Write $K = L_{i^\star}$.
        Set $p \coloneqq P_{i^\star}^{\infty}$ and note that $p < \infty$
        since by the $o(1)$-noise rate assumption,
        $R(L_{i^\star}; x_{1:n})\leq c_{i^\star}$ for all sufficiently large $n$.
        If we show that $\card{\Cl(\cL(p))} = \infty$,
        then we can apply \Cref{lem:constant-vanishing-noise:sufficient-generation} to conclude that \Cref{alg:intersection} generates from $\Cl(\cL(p))\supseteq K$.

        By \Cref{lem:meta-algorithm:prefix-stabilizes},
        since the priorities are non-decreasing,
        there is some $\bar n\in N$ such that for $n\geq \bar n$,
        the priorities $P_i^{(n)}$ of languages $L_i\in \cL(p)$ remain constant.
        For $n\geq \bar n$,
        we see that
        \begin{align*}
            \card{\Cl(\cL(p))}
            &\geq \card{S_n\cap \Cl(\cL(p))}
            \geq n \inparen{1 - \sum_{L_i\in \cL(p)} R(L_i; x_{1:n})} \\
            &\geq n \inparen{1 - \sum_{i\geq 1} \frac{1}{2^{i+1}}}
            \geq \frac{n}2\,.
        \end{align*}
        In particular,
        $\card{\Cl(\cL(p))} = \infty$ as desired.
    \end{proof}

    \begin{remark}
        While \cref{alg:intersection} uses a priority that depends on how long a language's empirical noise rate has stayed below the threshold $c_i$, which significantly simplifies our analysis by using the stable prefix property guaranteed by \cref{lem:meta-algorithm:prefix-stabilizes}, it is unclear if some notion of priority similar to that used in \cref{alg:intersection} is inherently needed. Indeed, in the appendix, we show in \cref{thm:vanishing-noise-generation-sorting} that \cref{alg:intersection-sorting} still achieves generation in the limit for any countable collection under vanishing noise rate, and does not use this notion of priority in the algorithm. In particular, \cref{alg:intersection-sorting} will not have the stable prefix property described in \cref{lem:meta-algorithm:prefix-stabilizes}.

        In this paper, we repeatedly make use of similar notions of priorities as that in \cref{alg:intersection}, \eg{}, in \cref{alg:vanishing-noise-intersection-set,alg:constant-noise-intersection-set,alg:intersection-vanishing-noise-improper-inf-inf}. We believe it is an interesting question to investigate if there exist generation algorithms that achieve the same guarantees and do not use this notion of priority we defined.
    \end{remark}

\subsection[Characterization of Generation with Constant Noise Rate and Arbitrary Omissions]{\mbox{Characterization of Generation with Constant Noise Rate and Arbitrary Omissions}}\label{sec:generation:constant-noise}
    Next, we investigate the setting where the adversaries can adopt an enumeration with constant noise rate and arbitrary omissions. 
    While not all countable collections are generable in this setting, our next result provides a complete characterization of generation in the limit and describes the collections for which generation in the limit is not achievable.
    \begin{theorem}\label{thm:constant-noise-characterization}
        Fix any constant $c\in (0,1)$.
        A collection $\cL$ is generable with $c$-noise and arbitrary omissions
        if and only if the following \emph{$c$-constant noise generation property} holds.

        \begin{condition}[$c$-constant noise generation property]
            \label{def:condition:cNoisy-Generation}
            For every non-empty finite subcollection $\cL'\subseteq \cL$
            and every enumeration $x_{1:\infty}$,
            either
            \begin{enumerate}[(a)] %
                \item there is some $L'\in \cL'$ such that $R(L'; x_{1:n}) > c$ infinitely often, or
                \item the intersection $\card{\Cl(\cL')} = \infty$ is infinite.
            \end{enumerate}  
        \end{condition} 
    \end{theorem}
    
    \begin{example}
        For any $k\in \N$,
        there are simple enumerations and collections of just $k$ languages $L_1, \dots, L_k$ that do not satisfy \Cref{def:condition:cNoisy-Generation} for $c=\nfrac1k$.
        For instance, 
        for $i\in [k]$,
        we can take \mbox{$L_i\coloneqq \sinbrace{n\in \N: n\mod k = i}$} to be the positive integers with remainder $i$ modulo $k$,
        and $x_{1:\infty}$ to be the canonical enumeration of the integers.
    \end{example}
    \subsubsection*{Necessity of \cref{def:condition:cNoisy-Generation} (Constant Noise Generation Property)}
    We first prove the necessity of \cref{def:condition:cNoisy-Generation}. %

    \begin{proof}[Proof of \Cref{thm:constant-noise-characterization} (Necessity)]    
        Suppose that the $c$-constant noise generation property does not hold.
        Then, there is some non-empty finite subcollection $\cL'\subseteq \cL$
        and enumeration $x_{1:\infty}$
        such that for every $L'\in \cL'$,
        $R(L'; x_{1:n})\leq c$ for all sufficiently large $n$,
        but the intersection $\card{\Cl(\cL')} < \infty$ is finite.
        Note this implies that $\card{\cL'}\geq 2$.

        Suppose towards a contradiction that $\cL$ is generated by some generator $\generator$.
        {Order} $\cL' = \set{L_1, \dots, L_k}$ for some $2\leq k< \infty$.
        Since $x_{1:\infty}$ is a valid enumeration with $c$-noise of each $L_j\in \cL'$,
        $\generator$ must simultaneously generate from all $L_j$ for $j\in [k]$ given the same enumeration $x_{1:\infty}$.
        We will argue that this yields a contradiction.

        Let $S_n\coloneqq \set{x_1, \dots, x_n}$,
        $w_n\coloneqq \generator_n(x_{1:n})$,
        and $W_n\coloneqq \set{w_1, \dots, w_n}$
        denote the set of input and generated strings up to time $n$.
        By assumption,
        for each $j\in [k]$,
        there is some $n_j^\star$ such that $w_n\in L_j\setminus (S_n\cup W_{n-1})$ for all $n\geq n_j^\star$.
        Take $n^\star\coloneqq \max_{j\in [k]} n_j^\star$.
        Then
        \[
            w_n\in \bigcap_{j\in [k]} \inparen{L_j\setminus (S_n\cup W_{n-1})}
            = \inparen{\bigcap_{j\in [k]} L_j} \setminus \inparen{S_n\cup W_{n-1}}
            \subseteq \Cl(\cL')\setminus W_{n-1}
        \]
        for all $n\geq n^\star$.
        However,
        this is impossible for $n > n^\star + \card{\Cl(\cL')}$ as $\generator$ has exhausted all elements in the finite intersection.
    \end{proof}
    Next, we design an element-based generation algorithm to show that the constant noise generation property suffices to guarantee generation.
    Interestingly, the same algorithm achieves generation with unknown vanishing error as well as generation with arbitrary omissions.

    \subsubsection*{Sufficiency of \cref{def:condition:cNoisy-Generation} (Constant Noise Generation Property)}
    Next, we prove the sufficiency of \cref{def:condition:cNoisy-Generation}.
    
    \paragraph{Pseudocode.} The pseudocode is \Cref{alg:intersection} instantiated with uniform thresholds $c_i=c$.

    \paragraph{Analysis.}
    We are now ready to prove the other direction of our characterization of constant noise generation (\Cref{thm:constant-noise-characterization}).
    \begin{proof}[Proof of \Cref{thm:constant-noise-characterization} (Sufficiency)]  
        Suppose that the $c$-constant noise generation property holds.
        We will apply \Cref{lem:constant-vanishing-noise:sufficient-generation}
        to prove that \Cref{alg:intersection} can generate
        when instantiated with thresholds $c_1 = c_2 = \dots = c$.

        Write $K = L_{i^\star}$ and define $p\coloneqq P_{i^\star}^\infty$.
        By definition,
        there is some $n_{i^\star}$ such that $R(L_{i^\star}; x_{1:n})\leq c$ for $n\geq n_{i^\star}$.
        Thus for $n\geq n_{i^\star}$,
        $P_{i^\star}^{(n)} = p < \infty$ remains constant moving forwards.
        By \Cref{lem:constant-vanishing-noise:sufficient-generation},
        it suffices to check that $\card{\Cl(\cL(p))} = \infty$
        in order to ensure that \Cref{alg:intersection} generates from $\Cl(\cL(p))\supseteq K$.
        Now,
        $\card{\cL(p)}\leq p$ is finite and by \Cref{lem:meta-algorithm:prefix-stabilizes},
        the priorities $P_i^{(n)}$ of all its members $L_i\in \cL(p)$ stabilize after some finite time.
        By the definition of the priorities in \Cref{alg:intersection},
        this means that no $L_i\in \cL(p)$ can satisfy $R(L_i, x_{1:n}) > c$ infinitely often.
        But then by the $c$-constant noise generation property,
        $\card{\Cl(\cL(p))} = \infty$.

        This concludes the proof.
    \end{proof}

\section{Generation with Density under Contamination}
\label{sec:dense-generation}
We now shift our attention to generation with breadth using the notions of generation with density defined in \cref{sec:density}. We will first discuss our results for generation with set-based densities in \cref{sec:dense-generation:set-based} and then for generation with element-based densities in \cref{sec:dense-generation:element-based}.

\subsection{Generation with Set-Based Density under Contamination}
\label{sec:dense-generation:set-based}

    We first consider set-based generators. Recall that under this definition (\cref{def:set-based-density}), the generator outputs a set in every timestep and we measure the set-based density using the sequence of set densities of the output sets at every timestep in the target language $K$. The validity requirement asks the generator to eventually output a subset of $K$ for all large enough $n$, and the generator achieves set-based lower (upper) density if for all large enough $n$ (infinitely often), the output set $A_n$ has good lower density in $K$. 

    \subsubsection{Set-Based Upper Density under Finite Contamination}
    In this section, we prove a result similar to the index-based upper density guarantee proved by \citep{kleinberg2025density}. While we have previously argued in \cref{ex:index-failure-under-noise} that even with $1$ noisy example, index-based generator can fail to generate in the limit, we will show that there exists a set-based generator that achieves the best possible set-based upper density. Our algorithm is inspired by the approach of \citet{kleinberg2025density} who used a ``fall-back'' strategy to ensure that, infinitely often, the output is (exactly) the target language, but requires some important new ideas as now the adversary's enumeration is contaminated by both noise and omissions. First, we use prove a result showing set-based upper density is achievable in the noiseless setting where the adversary's enumeration contains potentially infinitely omissions but \emph{no noise}. Then, we invoke the finite expansion sub-routine in \cref{alg:finite-expansion-routine-addition} and reduce the setting with finite noise and $c$-omissions to the setting without noise and with $c$-omissions.
    
    To design an algorithm achieving set-based upper density for the noiseless setting with infinite omissions, at each round we consider an active set $\cA_n$ consisting of all languages up to $L_n$. Then, we consider the longest prefix $\cL_n$ of languages of $\cA_n$ that satisfies
    \begin{itemize}
        \item the intersection of all languages in $\cL_n$, $\Cl(\cL_n)$, is infinite;
        \item $\cL_n$ also appears as a prefix of $\cL_{n-1}$.
    \end{itemize}
    The second requirement is exactly the fall-back strategy we used: whenever some language $L_i$ leaves the active set at round $n$, we only consider prefixes of $\cA_n$ consisting of languages that come before $L_i$. In other words, we only consider prefixes of the longest prefix of languages of $\cA_n$ that did not change compared to $\cA_{n-1}$. The strategy, implemented by \cref{alg:fall-back}, ensures that we are neither forever ``overshooting'' by intersecting a long prefix, thus sacrificing density, nor infinitely often ``undershooting'' by intersecting a small prefix, thus sacrificing correctness. 
    
    \begin{theorem}\label{thm:finite-noise-upper-density}
        For all countable collections of languages $\cL$, there is a set-based generator that generates in the limit from $\cL$ and achieves set-based upper density at least $1-c$ under adversaries that use an enumeration with finite noise and $c$-omissions.
    
        In particular, if the adversary uses an enumeration with finite noise and finite omissions, there is a set-based generator that generates in the limit from $\cL$ and achieves set-based upper density $1$.
    \end{theorem}
    
    \begin{remark}
        As a nice feature, the set-based generator in \cref{thm:finite-noise-upper-density} does not assume the knowledge of the amount of omissions in the adversary's enumeration.
    \end{remark}
    To prove \cref{thm:finite-noise-upper-density}, we will first prove the following result in the noiseless setting, and then make use of the alternative finite expansion sub-routine in \cref{alg:finite-expansion-routine-addition} and \cref{lem:expansion-subroutine-addition}.
    
    \begin{proposition}\label{prop:noiseless-upper-density}
        For all countable collections of languages $\cL$, there is a set-based generator that generates in the limit from $\cL$ and achieves set-based upper density at least $1-c$ under enumerations without noise and with $c$-omissions.
    \end{proposition}
    
    \begin{algorithm}[tbh!]
            \caption{Algorithm for \Cref{thm:finite-noise-upper-density}}
            \label{alg:fall-back}
            \begin{algorithmic}[1]
            \Require Countable  collection  $\cL=\{L_1,L_2,\dots\}$; enumeration $x_{1:\infty}$
            \State Let $S_n \gets \sinbrace{x_1, \dots, x_n}$ be the set of examples seen in the first $n$ steps
            \State Let $\cA_n \gets \inbrace{L_i: S_n \subseteq L_i, 1 \le i \le n}$ be the subcollection of languages consistent with the seen examples.
            \State If $n = 1$, set $\cL_n \gets \emptyset$.
            \State If $n \ge 2$, compute the largest common prefix $\cL_n$ of $\cA_{n-1}$ and $\cA_n$ with infinite intersection, \ie{}, 
            \[
                \cL_n \gets \inbrace{L_i \in \cA_n: \inbrace{L_1, \dots, L_i} \cap \cA_n = \inbrace{L_1, \dots, L_i} \cap \cA_{n-1}, \text{ and } \abs{\Cl(\inbrace{L_1, \dots, L_i} \cap \cA_n)} = \infty }.
            \]
            \State Output $\Cl(\cL_n) \setminus S_n$.
            \end{algorithmic}
        \end{algorithm}
    
    \begin{proof}[Proof of \cref{prop:noiseless-upper-density}]
        We will show that \cref{alg:fall-back} works for \cref{prop:noiseless-upper-density}.
    
        Let $K = L_{i^\star} \in \cL$ be the target language, and $x_{1:\infty}$ be an enumeration of $K$ without noise and with $c$-omissions. Let $\hat{K} \subseteq K$ be the subset of $K$ such that $x_{1:\infty}$ is an enumeration of $\hat{K}$ \emph{without noise or omissions}. By assumption, $\mu_{\text{low}}(\hat{K}, K) \ge 1-c$.

        \paragraph{Proof of Generation in the Limit.} First, we show that \cref{alg:fall-back} achieves generation in the limit. We observe that for $n \ge i^\star + 1$, we always have $K \in \cA_n$. Note that if $S_n \not\subseteq L_i$, then $S_{n+1} \not\subseteq L_i$. Thus, eventually, the prefix of languages in $\cA_n$ that come before $K$ stabilizes, since once some language leaves the set $\cA_n$, it will never enter it again. Let $n^\star \in \N$ be the time such that for all $n \ge n^\star$, $\cA_n \cap \{L_i: i \le i^\star\}$ is fixed. Denote this collection as $\cA_{\infty}$. Note that $K \in \cA_{\infty}$. We also note that every language in $\cA_{\infty}$ has to be a superset of $\hat{K}$, since each of them has to be consistent with $x_{1:\infty}$, which is an enumeration of $\hat{K}$ without noise or omissions. As a result, $\hat{K} \subseteq \Cl(\cA_{\infty})$ and thus the largest common prefix $\cL_n$ of $\cA_{n-1}$ and $\cA_n$ satisfies $\cA_{\infty} \subseteq \cL_n$ for $n \ge n^\star + 1$. This finishes the proof of generation in the limit, as for $n \ge n^\star + 1$, $\Cl(\cL_n) \subseteq \Cl(\cA_{\infty}) \subseteq K$.

        \paragraph{Proof of Density Guarantee.} Next, we show that infinitely often, \cref{alg:fall-back} outputs a set with lower set density at least $1-c$. At time $n$, consider $\cL_n$. Either $\hat{K} \subseteq \Cl(\cL_n)$, or there exists $L \in \cL_n$ such that $\hat{K} \not\subseteq L$. In the first case, we readily get $\mu_{\text{low}}(\Cl(\cL_n) \setminus S_n, K) \ge \mu_{\text{low}}(\hat{K_n} \setminus S_n, K) \ge 1 - c$. Thus, we now assume we are in the second case.

        List the languages in $\cL_n$ in order as $\cL_n = \inbrace{L_{i_n(1)}, \dots, L_{i_n(t)}}$ where $t = \abs{\cL_n}$. For $n \ge n^\star + 1$, we know that $\cA_{\infty}$ appears as a nonempty prefix of $\cL_n$, and moreover every language in $\cA_{\infty}$ is a superset of $\hat{K}$. Let $L$ be the first language in $\cL_n$ which is not a superset of $\hat{K}$, which exists by the assumption. Since $L \not\subseteq \hat{K}$ and $x_{1:\infty}$ enumerates $\hat{K}$ without omissions, there exists some finite time $n' > n$ such that $L$ is no longer consistent with $S_{n'}$. By our falling back strategy used in the construction of $\cL_n$, at time $n'$, we set $\cL_{n'}$ to be the prefix of languages of $\cL_n$ coming before $L$. Since we assumed that $L$ is the first language in $\cL_n$ which is not a superset of $\hat{K}$, we know that every language in $\cL_{n'}$ is a superset of $\hat{K}$, and thus $\hat{K} \subseteq \Cl(\cL_{n'})$. This means that infinitely often, the output set $\Cl(\cL_{n}) \setminus S_n$ achieves lower set density at least $1-c$ in K, concluding the proof.
    \end{proof}
    Now we are ready to prove \cref{thm:finite-noise-upper-density} by invoking \cref{alg:finite-expansion-routine-addition} and \cref{lem:expansion-subroutine-addition}.

    \begin{proof}[Proof of \cref{thm:finite-noise-upper-density}]
        Consider the following generation algorithm. 
        
    \paragraph{1. The Algorithm.}
    First, we construct the expanded collection $\Tilde{\cL}$ using the expansion subroutine from \cref{alg:finite-expansion-routine-addition}. 
    Recall that
    \[
        \Tilde{\cL} := \inbrace{L_{A} \coloneqq L \cup A: L \in \cL, A \subseteq U \setminus L, \abs{A} < \infty} \,.
    \]
    Then, we execute \cref{alg:fall-back} on the collection $\Tilde{\cL}$
    and the enumeration $x_{1:\infty}.$
    Since $x_{1:\infty}$ is a valid enumeration of an unknown target language $K \in \cL$, with finite noise set $A$ and $c$-omission, there exists $\hat{K} \subseteq K$ such that $x_{1:\infty}$ is an enumeration of $\hat{K} \cup A$ without noise or omissions, and $\mu_{\text{low}}(\hat{K}, K) \ge 1- c$. Thus, $x_{1:\infty}$ corresponds to an enumeration for the language $K_{A}$ without noise and with $c$-omission, where we used that $\mu_{\text{low}}(\hat{K} \cup A, K_A) = \mu_{\text{low}}(\hat{K}, K)$ for any finite set $A$.
    
    \paragraph{2. Proof of Generation in the Limit.}
    By \cref{prop:noiseless-upper-density}, \cref{alg:fall-back} applied to $\Tilde{\cL}$
    and $x_{1:\infty}$ generates in the limit from $K_A$. By \cref{lem:expansion-subroutine-addition}, we know that the same algorithm generates in the limit from $K$.
    
    \paragraph{3. Proof of Density Guarantee.}
    If we denote the output set of \cref{alg:fall-back} applied to $\Tilde{\cL}$
    and $x_{1:n}$ as $A_n$, by \cref{prop:noiseless-upper-density}, we know that $\limsup_{n\to\infty} \mu_{\text{low}}(A_n, K) = \limsup_{n\to\infty} \mu_{\text{low}}(A_n, K_A) \ge 1 - c$, where we used that $K_A = K \cup A$ differs from $K$ only by a finite set.
    \end{proof}

    \begin{theorem}
        \label{thm:set-based-upper:lower-bound}
        There exists a countable collection of languages $\cL$, such that no set-based generator can generate in the limit and achieve set-based upper density at least $1 - c + \eps$ for any $\eps > 0$, under adversaries that use an enumeration with $c$-omissions without noise.
    \end{theorem}
    
    \begin{proof}
        Consider a two language collection $L_1, L_2 \subseteq U$ such that $L_1 \subseteq L_2$ and $\mu_{\text{low}}(L_1, L_2) = \mu_{\text{up}}(L_1, L_2) = 1 - c$. Consider the canonical enumeration of $L_1$ without noise and without omissions. Note that this enumeration is a valid enumeration for both $L_1$ and $L_2$ with $c$-omissions without omissions. Therefore, any set-based generator that generates in the limit from the target language $K$ has to generate from $L_1$ in the limit. Since $L_1$ has upper density $1 - c$ in $L_2$, the set-based generator achieves set-based upper density at most $1 - c$ if the adversary sets $K = L_2$.
    \end{proof}

    \subsubsection{Characterization of Set-Based Lower Density under Finite Contamination}\label{sec:set-based-lower-density-finite-noise}
    We now shift our attention to obtaining density guarantees that hold for all timesteps beyond a finite 
    number of rounds. Recall that our results in \cref{thm:finite-noise-upper-density} give weaker density guarantees: for infinitely many timesteps the set outputted by the generator has good density in $K$; here we ask that the set always has good density, except perhaps for a finite set of timesteps. Naturally, since the density requirement here is 
    significantly stronger than in \cref{thm:finite-noise-upper-density} one would expect that the results we get are weaker. Indeed, even
    in the absence of noise, \cite{kleinberg2025density} showed that this type of guarantee is only achievable
    if the collection does not contain \emph{infinite perfect towers}, a technical condition they introduced. Importantly, this condition is \emph{not} satisfied by all countable collections of languages. Given this result, the main question we aim to understand is whether
    injecting noisy elements in the enumeration and omitting elements from the target hinges our ability to
    achieve this type of density guarantee compared to the noiseless setting. Perhaps surprisingly, the answer 
    is that achieving this guarantee in the noisy setting is significantly harder than in the noiseless setting, 
    even when the adversary is restricted to \emph{finite} amount of noise. 
    Our main result for the finite noise and omissions case is a complete characterization of when 
    this type of generation is achievable.
    As an immediate corollary of our result, we show that there are collections consisting of just two languages in which the generator cannot achieve any non-trivial density guarantee. To compare that with the noiseless setting, recall that the much stronger requirement of \emph{identification in the limit} is achievable
    for all finite collections of languages. Essentially, an informal interpretation of our result is that
    this type of generation is possible if all pairs of languages that are not dense in each other are ``infinitely separated''.

    We now proceed to the formal statement of
    our result.
    
    \begin{theorem}[Characterization of Set-Based Lower Density Generation with Finite Contamination]\label{thm:set-based-lower-density-finite-noise}
        A countable collection of languages $\cL$ is generable in the limit with set-based lower density $c > 0$ under finite noise and finite omissions if and only if for all $L, L' \in \cL$ with $\abs{L \setminus L'} < \infty$ it holds that 
        $\mu_{\text{low}}(L, L') \geq c.$
    \end{theorem}

    \begin{example}
        As a simple application of the above characterization, consider the following two language collection $\cL = \inbrace{L_1, L_2}$, where $L_1$ is the set of even numbers and $L_2$ is the set of all natural numbers. Then, by \cref{thm:set-based-lower-density-finite-noise}, under finite noise and finite omissions, there exists a set-based generator achieving set-based lower density $1/2$, but no generator can do better.
    \end{example}
    We prove \cref{thm:set-based-lower-density-finite-noise} in two steps; first, we show that when the stated condition does not hold, then no algorithm can generate from $\cL$ with set-based lower density{;} then, we show that if the condition holds, there exists an algorithm that generates from $\cL$ with the desired lower density. 
    The proof of \cref{thm:set-based-lower-density-finite-noise} follows as direct corollary of these two results.
    
    \begin{lemma}\label{lem:lower-set-based-density-lower-bound-finite}
        Let $\cL$ be a countable collection of languages that contains $L, L'$ with 
        $\abs{L \setminus L'} < \infty$ and
        $\mu_{\text{low}}(L, L') < c.$ Then, no algorithm can generate from $\cL$ in the limit with set-based lower density $c$ under finite noise and finite omissions.
    \end{lemma}

    \begin{proof}
    We begin by formally recalling the requirements for an algorithm $\generator$ to ``generate from $\cL$ in the limit with set-based lower density $c$'' under the specified adversaries. Let $x = (x_1, x_2, \dots)$ be an input sequence, and let $S_n = \generator(x_{1:n})$ be the algorithm's hypothesis at timestep $n$.
    Recall that $x$ is an \textbf{enumeration} of a target language $K \in \cL$ with finite noise and omissions if:
    \begin{enumerate}
        \item \textbf{Finite Noise:} The set $\{x_i \mid x_i \notin K\}$ is finite.
        \item \textbf{Finite Omissions:} The set $K \setminus \{x_i \mid i \in \mathbb{N}\}$ is finite.
    \end{enumerate}
    The algorithm $\generator$ must satisfy two properties for \emph{every} $K \in \cL$ and \emph{every} (contaminated) enumeration $x$ of $K$:
    \begin{itemize}
        \item[\textbf{(i)}] \textbf{Generation in the Limit:} There exists an $n^*$ such that for all $n > n^*$, it holds that $S_n \subseteq K$.
        \item[\textbf{(ii)}] \textbf{Density Guarantee:} $\liminf_{n \to \infty} \mu_{\text{low}}(S_n, K) \ge c$.
    \end{itemize}
    We now proceed by contradiction.
    
    Assume, for the sake of contradiction, that such an algorithm $\generator$ exists. Let $L, L' \in \cL$ be two languages satisfying the lemma's hypotheses:
    \begin{enumerate}
        \item $F = L \setminus L'$ is a finite set.
        \item $\mu_{\text{low}}(L, L') < c$.
    \end{enumerate}
    We construct an input $x$ and a sequence of time indices $n_0 < n_1 < n_2 < \dots$ inductively. 
    Let $\mu_{\text{low}}(L, L') = c -\eps$
    for some $\eps > 0.$
    This construction will force $\generator$ to fail either property (i) or (ii) in these timesteps.

    The construction proceeds in phases. We treat even-numbered and odd-numbered phases differently. For any phase $k \ge 1$:
    
    \paragraph{Phase $2k-1$ (Target $L'$):}
    \begin{enumerate}
        \item Define a potential enumeration $T_{2k-1}$ as follows:
        \begin{itemize}
            \item It begins with the prefix $x_{1:n_{2k-2}}$ (which is empty for $k=1$).
            \item It is followed by a complete, ordered enumeration of all elements in $L' \setminus \{x_1, \dots, x_{n_{2k-2}}\}$.
        \end{itemize}
        \item $T_{2k-1}$ is an \textbf{enumeration of $L'$ with finite noise}.
        \begin{itemize}
            \item \textbf{Noise:} The number of elements in $T_{2k-1}$ not in $L'$ is bounded 
            by $\abs{x_{1:n_{2k-2}}} < \infty$. Thus, this noise is finite.
            \item \textbf{Omissions:} By construction, $T_{2k-1}$ contains all elements of $L'$, so there are 0 omissions.
        \end{itemize}
        \item By the definition of successful generation in this setting, $\generator$ running on $T_{2k-1}$ must eventually be outputting subsets of $L'$ that have $c$ lower density in $L'$.  Thus, there must exist a time $n_{2k-1} > n_{2k-2}$ such that $\mu_{\text{low}}(S_{n_{2k-1}}, L') \geq c - \nicefrac{\eps}{2}, \text{ and } S_{n_{2k-1}} \subseteq L'$ (where $S_{n_{2k-1}}$ is the output on the prefix $T_{2k-1}[1:n_{2k-1}]$). Thus, since $\mu_{\text{low}}(L, L') = c-\eps$ \mbox{it must be the case that $S_{n_{2k-1}} \not\subseteq L.$}
        \item If no such $n_{2k-1}$ exists, $\generator$ fails the generation requirement for $L'$ on input $T_{2k-1}$. We halt the construction and have found our contradiction.
        \item Otherwise, set $x_{1:n_{2k-1}} = T_{2k-1}[1:n_{2k-1}]$ and proceed.
    \end{enumerate}
    
    \paragraph{Round $2k$ (Target $L$):}
    \begin{enumerate}
        \item Define a text $T_{2k}$ as follows:
        \begin{itemize}
            \item It begins with the prefix $x_{1:n_{2k-1}}$.
            \item It is followed by a complete, ordered enumeration of all elements in $L \setminus \{x_1, \dots, x_{n_{2k-1}}\}$. (This includes all elements of $F$ not already in the prefix).
        \end{itemize}
        \item $T_{2k}$ is an \textbf{enumeration of $L$ with finite noise}.
        \begin{itemize}
            \item \textbf{Noise:} The set of elements in $T_{2k}$ not in $L$ is $x_{1:n_{2k-1}} \cap (L' \setminus L)$. This is a finite prefix of $L' \setminus L$, so the noise is finite.
            \item \textbf{Omissions:} By construction, $T_{2k}$ contains all elements of $L$ so there are 0 omissions.
        \end{itemize}
        \item By the definition of successful generation in this setting, $\generator$ running on $T_{2k}$ must eventually be outputting subsets of $L$ that have $c - \nicefrac{\eps}{2}$ lower density in $L$. Thus, there must exist a time $n_{2k} > n_{2k-1}$ such that $S_{n_{2k}} \subseteq L$.
        \item If no such $n_{2k}$ exists, $\generator$ fails the generation requirement for $L$ on input $T_{2k}$. We halt and have a contradiction.
        \item Otherwise, set $x_{1:n_{2k}} = T_{2k}[1:n_{2k}]$ and proceed to the next round.
    \end{enumerate}
    This inductive process has two possible outcomes:
    
    \paragraph{Case 1: The construction halts after a finite number of rounds.}
    As shown in step (4) of the inductive rounds, if the construction halts, it is because $\generator$ failed to find a required output in all the rounds of this phase. This constitutes a failure of the generation requirement for either $L$ or $L'$ on a complete enumeration of the corresponding language with finite noise and zero omissions.
    
    \paragraph{Case 2: The construction proceeds for infinitely many rounds.}
    This process defines a single, infinite input $x = \lim_{k \to \infty} x_{1:n_k}$ and an infinite sequence of time indices $n_1 < n_2 < n_3 < \dots$.
    
    Let us show that $x$ is a valid enumeration of $L'$ with finite noise and zero omissions.
    \begin{itemize}[leftmargin=20pt]
        \item \textbf{Noise:} The set of elements in $x$ not in $L'$ is $x \cap (L \setminus L') = x \cap F$. Elements from $F$ are \emph{only} added during the ``Target $L$'' rounds (even $k$). Since $F$ is finite, $x$ contains at most $\abs{F}$ noisy elements. The noise is finite.
        \item \textbf{Omissions:} The set $L' \setminus x$ is empty. In every ``Target $L'$'' round (odd $k$), the text $T_{2k-1}$ is defined to contain \emph{all} elements of $L'$. Since $x$ is the limit of these prefixes, it must contain all of $L'$.
    \end{itemize}
    Therefore, $x$ is a \textbf{an enumeration of $L'$ with finite noise}.
    
    By property (ii) (Density Guarantee), $A$ running on $x$ must satisfy:
    \[
    \liminf_{n \to \infty} \mu_{\text{low}}(S_n, L') \ge c\,.
    \]
    However, consider the infinite subsequence of times $n_2, n_4, n_6, \dots, (n_{2k}, \dots)$. By construction, at every time $n_{2k}$ (the end of a ``Target $L$'' round), the algorithm's output satisfies $S_{n_{2k}} \subseteq L$.
    
    A key property of set-based density is that it is monotone in its first argument: if $A \subseteq B$, then $\mu_{\text{low}}(A, C) \le \mu_{\text{low}}(B, C)$ for any $C$.
    Thus, for every $k \ge 1$:
    \[
    S_{n_{2k}} \subseteq L \implies \mu_{\text{low}}(S_{n_{2k}}, L') \le \mu_{\text{low}}(L, L')\,.
    \]
    By the lemma's hypothesis, $\mu_{\text{low}}(L, L') < c$. This means we have an infinite subsequence of hypotheses $S_{n_{2k}}$ such that:
    \[
    \mu_{\text{low}}(S_{n_{2k}}, L') \leq c - \eps \quad \text{for all } k \ge 1\,.
    \]
    This directly implies that the limit inferior of the \emph{entire} sequence must be less than $c$:
    \[
    \liminf_{n \to \infty} \mu_{\text{low}}(S_n, L') \le \lim_{k \to \infty} \mu_{\text{low}}(S_{n_{2k}}, L') \le \mu_{\text{low}}(L, L') < c\,.
    \]
    This is a direct contradiction of the Density Guarantee (ii).
    
    Thus, in both possible cases, the existence of algorithm $\generator$ leads to a contradiction. In Case 1, $\generator$ fails the Generation requirement. In Case 2, $\generator$ fails the Density Guarantee.
    Therefore, no such algorithm $\generator$ can exist.
    \end{proof}
    Next, we proceed with describing an algorithm
    that achieves set-based lower density $c >0$ whenever this condition holds.
    For our upper bound, we will utilize an algorithm that achieves the desired set-based lower density
    guarantee in the absence of any contamination in the dataset, and then we will utilize 
    our expansion subroutine (\cref{alg:finite-expansion-routine}) to convert it to a generator that achieves this guarantee in the setting of finite contamination. First, we state a result which follows as an immediate corollary from \citet{kleinberg2024language} and will be useful in our derivations.
    
    \begin{lemma}[Noiseless Set-Based Lower Density \citep{kleinberg2024language}]\label{lem:lower-set-density-km24}
        Let $c \in [0,1]$ and $\cL$ be a countable collection for which every $L, L' \in \cL$ with $\abs{L \setminus L'} < \infty$ satisfies $\mu_{\text{low}}(L, L') \geq c$. Then, there exists a generating algorithm $\generator$ that achieves set-based generation in the limit from $\cL$ with set-based lower density at least $c$, under enumerations without noise or omissions.
    \end{lemma}
    We remark that this is an immediate corollary of the algorithm of \citet{kleinberg2024language}. The main property of their algorithm is that, in the limit, it outputs subsets of the target language. Thus, if every subset of the target language is $c$-dense in the target (as guaranteed by the premises of our characterization), the above result follows immediately. We are now ready to state and prove our result.

    \begin{lemma}\label{lem:lower-set-based-density-upper-bound-finite}
        Let $\cL$ be a countable collection of languages such that for all $L, L' \in \cL$ either 
        $\abs{L \setminus L'} = \infty$ or
        $\mu_{\text{low}}(L, L') \geq c.$ Then, there exists algorithm that generates from $\cL$ in the limit with set-based lower density $c$ under finite noise and finite omissions.
    \end{lemma}

    \begin{proof}
    We construct an algorithm $\generator$ that generates from $\cL$ in the limit and achieves the stated density guarantee, given the lemma's condition.

    \paragraph{1. The Algorithm.}
    First, we construct the expanded collection $\Tilde{\cL}$ using the expansion subroutine from \cref{alg:finite-expansion-routine}. 
    Recall that
    \[
        \Tilde{\cL} := \inbrace{L_{A,B} \coloneqq L \cup A \setminus B: L \in \cL, A \subseteq U \setminus L, B \subseteq L, \abs{A} < \infty, \abs{B} < \infty} \,.
    \]
    Then, we execute the algorithm described in \cref{lem:lower-set-density-km24} on the collection $\Tilde{\cL}$
    and the enumeration $x_{1:\infty}.$
    Since $x_{1:\infty}$ is a valid enumeration of an unknown target language $K \in \cL$, with finite noise set $A$ and finite omission set $B$,
    this corresponds to an enumeration without noise or omissions for the language $K_{A,B}$.

    \paragraph{2. Proof of Generation in the Limit.}
    Since the algorithm from \cref{lem:lower-set-density-km24} achieves (set-based) generation in the limit with respect to $\Tilde{\cL},$ it follows as a direct corollary from \cref{lem:expansion-subroutine} that it achieves (set-based) generation in the limit with respect to $\cL$ when the enumeration contains finite amount of
    noise and omissions.
    
    \paragraph{3. Proof of Density Guarantee.}
    We need to show that $\liminf_{n \to \infty} \mu_{\text{low}}(A_n, K) \ge c$, where $A_n$ is the output of the algorithm during the $n$-th step. First, notice that since every $L, L' \in \cL$
    with $\abs{L \setminus L'} < \infty$ satisfy $\mu_{\text{low}}(L, L') \geq c$ it must also be the case
    that every $L_{A,B}, L'_{A,B} \in \Tilde{\cL}$
    with $\abs{L_{A,B} \setminus L'_{A',B'}} < \infty$ satisfy $\mu_{\text{low}}(L_{A,B}, L'_{A',B'}) \geq c$. This is simply because
    if $\abs{L_{A,B} \setminus L'_{A',B'}} < \infty,$ it must be the case that $\abs{L \setminus L'} < \infty$, and lower density is invariant to adding or subtracting finitely many elements to its arguments. Thus, \cref{lem:lower-set-density-km24} guarantees that $\liminf_{n \to \infty} \mu_{\text{low}}(A_n, K_{A,B}) \geq c.$ Since $K, K_{A,B}$ differ on finitely many elements, this directly implies
    that $\liminf_{n \to \infty} \mu_{\text{low}}(A_n, K) \geq c$.

    \end{proof}

    \begin{remark}[Known Noise Level]
        It is worth highlighting that in our lower bound construction we made use of the crucial fact that the generator does not know the noise level; it merely knows that noise is finite. Interestingly, the previous characterization does not hold anymore when the generator knows a bound on the finite noise level. Recall that in the absence of the density requirement, knowledge of the noise rate does not change the set of countable collections that can be generated in the limit. 
    \end{remark}

    \subsubsection{Characterization of Set-Based Lower \& Upper Density under Vanishing Noise Rate and Arbitrary Omissions}\label{sec:set-based-lower-infinite-noise}
        In this section, we characterize when set-based density is achievable under vanishing noise rate.
    \begin{theorem}[Characterization of Set-Based Density with Vanishing Noise Rate]\label{thm:vanishing-noise-set-density-characterization}
        Fix a collection $\cL$ and $\rho\in (0,1]$. 
        Under vanishing noise rate and arbitrary omissions, there exists a set-based generator that generates in the limit and achieves set-based lower density $\rho$ if and only if the following property holds.
    
        \begin{condition}[Vanishing Noise $\rho$-Dense Set-Generation]
            \label{prop:vanishingNoise:dense:set}
            For every non-empty finite sub-collection $\cL' \subseteq \cL$, if languages in $\cL'$ share infinitely many elements, \ie{}, $|\Cl(\cL')| = \infty$, then for any $L \in \cL'$,
            \[
                \mu_{\text{low}}(\Cl(\cL'), L) \ge \rho\,.
            \]
        \end{condition}
    \end{theorem}
    \begin{remark}[Upper Density]
        \cref{thm:vanishing-noise-set-density-characterization}'s characterization is sharp in the following sense:
        \begin{itemize}
            \item Whenever \cref{prop:vanishingNoise:dense:set} does not hold, no set-based generator can generate in the limit and achieve set-based \emph{upper density} $\rho$, even when the adversary is restricted to use an enumeration with $o(1)$-noise and \emph{without omissions}.
            \item Whenever \cref{prop:vanishingNoise:dense:set} holds, there exists a set-based generator that generates in the limit and achieves set-based \emph{lower density} $\rho$, even when the adversary presents an enumeration with $o(1)$-noise and \emph{with infinite omissions}.
        \end{itemize}
    \end{remark}

    \begin{example}
        As a simple application of the characterization in \cref{thm:vanishing-noise-set-density-characterization}, consider the following collection of languages $\cL = \inbrace{L_1, L_2, L_3}$, where $L_1 = \inbrace{i \in \N: i \ne 0 \mod{3}}$, $L_2 = \inbrace{i \in \N: i \ne 1 \mod{3}}$, $L_3 = \inbrace{i \in \N: i \ne 2 \mod{3}}$. By the characterization, under vanishing noise rate and arbitrary omissions, the best set-based lower density any generator can achieve is $1/2$, since the intersection of any two languages have infinite cardinality and we can easily check the intersection has density $1/2$ in each of the two languages, whereas the intersection of all three languages is empty.
    \end{example}
    In the remainder of this section we prove \cref{thm:vanishing-noise-set-density-characterization}.
    
    \subsubsection*{Necessity of \cref{prop:vanishingNoise:dense:set}}
    
    We first prove the necessity of \cref{prop:vanishingNoise:dense:set}. 
    As we will see,
    this condition is necessary even for enumerations without omission.
    Note that the case with enumerations can only be more challenging. 
    
    \begin{proof}[Proof of \cref{thm:vanishing-noise-set-density-characterization} (Necessity)]
        Assume that there exists a non-empty finite subcollection $\cL' \subseteq \cL$ such that $|\Cl(\cL')| = \infty$, and there exists $L' \in \cL'$ such that $\mu_{\text{low}}(\Cl(\cL'), L') < \rho$.
    
        We will construct an enumeration $x_1, x_2, x_3, \dots$ that is simultaneously an enumeration for any language $L \in \cL'$ with $o(1)$-noise and without omission. Given such an enumeration, any set-based generator that generates in the limit from $\cL$ has to generate from all $L \in \cL'$ in the limit, and thus will output a subset of $\Cl(\cL')$ for all large enough $n$. In particular, any subset of $\Cl(\cL')$ has lower set-density in $L'$ strictly less than $\rho$. Since the adversary could have chosen $K = L'$, any set-based generator that generates in the limit cannot even achieve set-based upper density at least $\rho$, even if the adversary does not omit any elements of $K$ in its enumeration.
    
        Now, we will construct such an enumeration $x_1, x_2, x_3, \dots$ that is simultaneously an $o(1)$-noisy enumeration for any language $L \in \cL'$ without omission. Define a scheduling of time $T = \{1, 2, 4, 8, \dots\}$. We will only use that $\mu_{\text{up}}(T, \N) = 0$. For $n \not\in T$, the adversary enumerates the next element in $\Cl(\cL')$ that has not been enumerated. For $n \in T$, the adversary enumerates the next element in $(\bigcup_{L \in \cL'} L) \setminus \Cl(\cL')$ that has not been enumerated. Clearly, this enumeration will enumerate every elements of every $L \in \cL'$. Clearly, for any $L \in \cL'$, if the element $x_n$ enumerated at time $n$ does not belong to $L$, we must have $n\in T$. Since $T$ has density $0$ in $\N$, we conclude that $x_1, x_2, x_3, \dots$ is an enumeration for any language $L \in \cL'$ with $o(1)$-noise and without omission. This concludes the proof.
    \end{proof}
    
    \subsubsection*{Sufficiency of \cref{prop:vanishingNoise:dense:set}}
    
    We next show that if \cref{prop:vanishingNoise:dense:set} holds, then there exists a set-based generator that generates in the limit and achieves set-based lower density $\rho$, under $o(1)$-noise and potentially infinite omissions in its enumeration.
    
    \paragraph{Pseudocode.} The pseudocode is presented in \Cref{alg:vanishing-noise-intersection-set}.
        
        \begin{algorithm}[tbh!]
            \caption{Algorithm for \Cref{thm:vanishing-noise-set-density-characterization}}
            \label{alg:vanishing-noise-intersection-set}
            \begin{algorithmic}[1]
            \Require Countable  collection  $\cL=\{L_1,L_2,\dots\}$;  thresholds $c_1, c_2,\ldots\in (0, 1)$; enumeration $x_{1:\infty}$
            \State Let $S_n \gets \sinbrace{x_1, \dots, x_n}$ be the set of examples seen in the first $n$ steps
            \vspace{2mm}
            \For{$i=1, 2, \dots, n$}
                \State Compute the following number
                \[
                    N_i^{(n)} \gets \min\inbrace{N\geq 1: \forall m\in \sinbrace{N, \dots, n},~~ R(L_i; x_{1:m})\leq c_i}\,.
                \]
                \State Assign the language $L_i$ a priority of $P_i^{(n)}\gets i + N_i^{(n)}$
            \EndFor
            \vspace{2mm}
            \State Re-order $\inbrace{L_1, \dots, L_n}$ in increasing priority, tie-breaking by index, as $\sinbrace{L_{i_n(1)}, \dots, L_{i_n(n)}}$, \ie{}, for each $j \in [n-1]$, ensure either $P_{i_n(j)}^{(n)} < P_{i_n(j+1)}^{(n)}$,
            or $P_{i_n(j)}^{(n)} = P_{i_n(j+1)}^{(n)}$ and $i_n(j) < i_{n}(j+1)$.
            \State Compute the largest index $j_n \in [n]$ such that the intersection of the prefix of the sorted list of languages in the re-ordering up to $L_{i_n(j_n)}$ is infinite, \ie{}, 
            \[
                J_n \gets \max\inbrace{\bar j \in [n]: \card*{\bigcap\nolimits_{j=1}^{\bar j} L_{i_n(j)}} = \infty}\,.
            \]
            \State Output $\bigcap_{j\leq J_n} L_{i_n(j)} \setminus S_n$.
            \end{algorithmic}
        \end{algorithm}
    
    \begin{proof}[Proof of \cref{thm:vanishing-noise-set-density-characterization} (Sufficiency)]
    
    Assume that for any non-empty finite subcollection $\cL' \subseteq \cL$ with $|\Cl(\cL')| = \infty$, we have $\mu_{\text{low}}(\Cl(\cL'), L) \ge \rho$ for any $L \in \cL'$.
    
    Now consider the set-based generator described in \cref{alg:vanishing-noise-intersection-set}, with threshold parameters $c_i \coloneqq \sfrac{\eps}{2^i}$ for all $i \in \N$.
    
    Note that the priority of any language $P_i^{(n)}$ is non-decreasing. Also note that the priority of the target language $K$ will remain fixed after some finite time $n'$, since the adversary uses an enumeration of $K = L_{i^\star}$ with $o(1)$-noise, and eventually the empirical noise rate $R(L_{i^\star}; x_{1:n})$ stays below the positive threshold $c_{i^\star}$ for all large enough $n$. Let $p \coloneqq P_{i^\star}^{\infty}$ be the number that the priority of $K$ stabilizes to.
    
    By \cref{lem:meta-algorithm:prefix-stabilizes}, if we denote $\cL(p) \coloneqq \{L_i: P_i^\infty \le p\}$ and $\cL^{(n)}(p) \coloneqq \{L_i: P_i^{(n)} \le p\}$, there exists $n^\star$ such that for all $n \ge n^\star$, we have $\cL^{(n)}(p) = \cL(p)$. Moreover, we know that $K \in \cL(p)$ by definition.
    
    Now we will show that $|\Cl(\cL(p))| = \infty$. Consider any $n \ge \max \{n^\star, p\}$, we know that $\cL^{(n)}(p) = \cL(p)$. For any $L_i \in \cL^{(n)}(p)$, we have $N_i^{(n)} = P_i^{(n)} - i < p$. Thus, for all $m \in \{p, p+1, \dots, n\}$, we have $R(L_i; x_{1:m}) \le c_i$. In particular, $R(L_i; x_{1:n}) \le c_i$. Therefore, we get that
    \begin{align*}
        \left|\Cl(\cL^{(n)}(p)) \cap S_n\right|~~
        &=~~ |S_n| - \left|\bigcup_{L_i \in \cL^{(n)}} (S_n \setminus L_i)\right|\\
        &\ge~~ |S_n| - \sum_{L_i \in \cL^{(n)}} |S_n \setminus L_i|\\
        &\ge~~ n - \sum_{L_i \in \cL^{(n)}} n \cdot c_i\\
        &\ge~~ n(1 - \eps)\,.
    \end{align*}
    Since the above inequality holds for any $n \ge \max\{n^\star, p\}$ and $\cL(p) = \cL^{(n)}(p)$ for $n \ge \max\{n^\star, p\}$, we conclude that $|\Cl(\cL(p))| = \infty$.
    
    Therefore, at any round time $n \ge \max\{n^\star, p\}$, we know that the finite subcollection $\cL^{(n)}(p) = \cL(p)$ appears as a prefix of the sorted list $\{L_{i_n(1)}, \dots, L_{i_n(n)}\}$. Since the intersection of this prefix, $\Cl(\cL(p))$, has infinite cardinality, we know that the computed index $j_n \in [n]$ satisfies that $\cL(p) \subseteq \{L_{i_n(1)}, L_{i_n(2)}, \dots, L_{i_n(j_n)}\}$. Denote this subcollection as $\cL_n \coloneqq \{L_{i_n(1)}, L_{i_n(2)}, \dots, L_{i_n(j_n)}\}$. By definition of $j_n$, $|\Cl(\cL_n)| = \infty$.
    
    Thus, for any $n \ge \max\{n^\star, p\}$, we know that $\cL(p) = \cL^{(n)}(p) \subseteq \cL_n$. Moreover, $\cL_n$ is a finite subcollection of $\cL$, and $K \in \cL(p) \subseteq \cL_n$. By the \cref{prop:vanishingNoise:dense:set}, for any $L \in \cL_n$, we have $\mu_{\text{low}}(\Cl(\cL_n), \cL_n) \ge \rho$. In particular, since $K \in \cL_n$, we have $\mu_{\text{low}}(\Cl(\cL_n), K) \ge \rho$. Thus, the output
    \[\bigcap_{j \le j_n} L_{i_n(j)} \setminus S_n = \Cl(\cL_n) \setminus S_n\]
    has lower density at least $\rho$ in $K$, which concludes the proof.
    
    \end{proof}
    
    \subsubsection{Characterization of Set-Based Lower \& Upper Density under Constant Noise Rate and Arbitrary Omissions}
        In this section, we characterize when set-based density is achievable under constant noise rate and arbitrary omissions.
    
    \begin{theorem}[Characterization of Set-Based Density under Constant Noise Rate]
        \label{thm:constant-noise-set-density-characterization}
        Fix a collection $\cL$ and $c\in (0,1]$. Under constant noise rate $c$ and arbitrary omissions, there exists a set-based generator that generates in the limit and achieves set-based lower density $\rho$ if and only if the following property holds.
        \begin{condition}[$c$-constant Noise $\rho$-Dense Set-Generation]
            \label{prop:constantNoise:dense:set}
            For every non-empty finite subcollection $\cL' \subseteq \cL$ and every enumeration $x_{1:\infty}$, either
            \begin{enumerate}[(a)]
                \item there exists some language $L' \in \cL'$ such that $R(L'; x_{1:n})> c$ infinitely often, or
                \item each language $L \in \cL'$ satisfies $\mu_{\text{low}}(\Cl(\cL'), L) \ge \rho$.
            \end{enumerate}
        \end{condition}
    \end{theorem}
    In fact, we will show:
    \begin{itemize}
        \item Whenever \cref{prop:constantNoise:dense:set} does not hold, no set-based generator can generate in the limit and achieve set-based \emph{upper density} $\rho$ when the adversary presents an enumeration with $c$-noise arbitrary omissions.
        \item Whenever \cref{prop:constantNoise:dense:set} holds, there exists a set-based generator that generates in the limit and achieves set-based \emph{lower density} $\rho$ when the adversary presents an enumeration with $c$-noise and arbitrary omissions.
    \end{itemize}

    \subsubsection*{Necessity of \cref{prop:constantNoise:dense:set}}
    
    We first prove the necessity of \cref{prop:constantNoise:dense:set}. The proof is similar to that of \cref{thm:vanishing-noise-set-density-characterization} but simpler, since the negation of the $c$-constant noise rate $\rho$-dense set-generation property provides us with an enumeration that the adversary could use to force the density achieved by any set-based generator to be less than $\rho$. 
    
    \begin{proof}[Proof of \cref{thm:constant-noise-set-density-characterization} (Necessity)]
        Assume that there exists a non-empty finite subcollection $\cL' \subseteq \cL$ and an enumeration $x_1, x_2, x_3, \dots$ such that
        \begin{enumerate}[(a)]
            \item for every $L \in \cL'$, the $x_1, x_2, x_3, \dots$ is a $c$-noisy enumeration for $L$,
            \item and there exists $L' \in \cL'$ such that $\mu_{\text{low}}(\Cl(\cL'), L) < \rho$.
        \end{enumerate}
        Since $x_1, x_2, x_3, \dots$ is an enumeration with $c$-noise for any language in the finite subcollection $\cL'$, any set-based generator that generates in the limit from $\cL$ has to generate from all $L \in \cL'$ in the limit under this enumeration. Thus, such a set-based generator will output a subset of $\Cl(\cL')$ for all large enough $n$. In particular, any subset of $\Cl(\cL')$ has lower set-density in $L'$ strictly less than $\rho$. Since the adversary could have chosen $K = L'$, any set-based generator that generates in the limit cannot even achieve set-based upper density at least $\rho$.
    \end{proof}
    
    \subsubsection*{Sufficiency of \cref{prop:constantNoise:dense:set}}
    
    We next show that by slightly adapting \cref{alg:vanishing-noise-intersection-set} with threshold parameters $c_i \coloneqq c$ and a different stopping condition, we obtain a set-based generator that achieves set-based lower density whenever \cref{prop:constantNoise:dense:set} holds.
    
    \paragraph{Pseudocode.} The pseudocode is presented in \Cref{alg:constant-noise-intersection-set}.

    \begin{algorithm}[tbh!]
            \caption{Algorithm for \Cref{thm:constant-noise-set-density-characterization}}
            \label{alg:constant-noise-intersection-set}
            \begin{algorithmic}[1]
            \Require Countable  collection  $\cL=\{L_1,L_2,\dots\}$; $c \in (0, 1]$; enumeration $x_{1:\infty}$
            \State Let $S_n \gets \sinbrace{x_1, \dots, x_n}$ be the set of examples seen in the first $n$ steps
            \vspace{2mm}
            \For{$i=1, 2, \dots, n$}
                \State Compute the following number
                \[
                    N_i^{(n)} \gets \min\inbrace{N\geq 1: \forall m\in \sinbrace{N, \dots, n},~~ R(L_i; x_{1:m})\leq c}\,.
                \]
                \State Assign the language $L_i$ a priority of $P_i^{(n)}\gets i + N_i^{(n)}$
            \EndFor
            \vspace{2mm}
            \State Re-order $\inbrace{L_1, \dots, L_n}$ in increasing priority, tie-breaking by index, as $\sinbrace{L_{i_n(1)}, \dots, L_{i_n(n)}}$, \ie{}, for each $j \in [n-1]$, ensure either $P_{i_n(j)}^{(n)} < P_{i_n(j+1)}^{(n)}$,
            or $P_{i_n(j)}^{(n)} = P_{i_n(j+1)}^{(n)}$ and $i_n(j) < i_{n}(j+1)$.
            \State Compute the largest index $j_n \in [n]$ such that the intersection of the prefix of the sorted list of languages in the re-ordering up to $L_{i_n(j_n)}$ has lower set-density at least $\rho$ for language $L_{i_n(j)}$ for any $j \le j_n$, \ie{},
            \[
                J_n \gets \max\inbrace{\bar j \in [n]: \forall j' \le \bar j,\quad   \mu_{\text{low}}\left(\bigcap\nolimits_{j=1}^{\bar j} L_{i_n(j)}, L_{i_n(j')}\right) \ge \rho}\,.
            \]
            \State Output $\bigcap_{j\leq J_n} L_{i_n(j)} \setminus S_n$.
            \end{algorithmic}
        \end{algorithm}
    
    \begin{proof}[Proof of \cref{thm:constant-noise-set-density-characterization} (Sufficiency)]
        Assume that for any non-empty finite subcollection $\cL' \subseteq \cL$ and every enumeration $x_1, x_2, x_3, \dots$, either
        \begin{enumerate}[(a)]
            \item there is some $L' \in \cL'$ such that $R(L'; x_{1:n})> c$ infinitely often, or
            \item each language $L \in \cL'$ satisfies $\mu_{\text{low}}(\Cl(\cL'), L) \ge \rho$.
        \end{enumerate}
        Now consider the set-based generator described in \cref{alg:vanishing-noise-intersection-set}, with threshold parameters $c_i \coloneqq c$ for all $i \in \N$.
    
        Note that the priorities of any language $P_i^{(n)}$ is non-decreasing. Also note that the priority of the target language $K$ will remain fixed after some finite time $n'$, since the adversary uses an enumeration with $c$-noise of $K = L_{i^\star}$, and eventually the empirical noise rate $R(L_{i^\star}; x_{1:n}) \le c$ for all large enough $n$. Let $p \coloneqq P_{i^\star}^{\infty}$ be the number that the priority of $K$ stabilizes to.
    
        By \cref{lem:meta-algorithm:prefix-stabilizes}, if we denote $\cL(p) \coloneqq \{L_i: P_i^\infty \le p\}$ and $\cL^{(n)}(p) \coloneqq \{L_i: P_i^{(n)} \le p\}$, there exists $n^\star$ such that for all $n \ge n^\star$, we have $\cL^{(n)}(p) = \cL(p)$, and $P_{i}^{(n)} = P_{i}^{\infty}$ for all $L_i \in \cL(p)$, \ie{}, the priorities of all languages in $\cL(p)$ stabilize. Moreover, we know that $K \in \cL(p)$ by definition.
    
        Therefore, at any round time $n \ge \max\{n^\star, p\}$, we know that the finite subcollection $\cL^{(n)}(p) = \cL(p)$ appears as a prefix of the sorted list $\{L_{i_n(1)}, \dots, L_{i_n(n)}\}$. Moreover, since the priorities of all languages in $\cL(p)$ stabilize, $N_i^{(n)} = P_i^{(n)} - i = P_i^{\infty} - i$ for any $L_i \in \cL(p)$. In other words, by definition of $N_i^{(n)}$, the enumeration $x_1, x_2, x_3, \dots$ that the adversary uses is a $c$-noisy enumeration for any $L_i \in \cL(p)$. Thus, each language $L \in \cL(p)$ satisfies $\mu_{\text{low}}(\Cl(\cL(p)), L) \ge \rho$.

        Since the intersection of this prefix, $\Cl(\cL(p))$, satisfies that $\mu_{\text{low}}(\Cl(\cL(p)), L) \ge \rho$ for any $L \in \cL(p)$, we know that the computed index $j_n \in [n]$ satisfies that $\cL(p) \subseteq \{L_{i_n(1)}, L_{i_n(2)}, \dots, L_{i_n(j_n)}\}$. Denote this subcollection as $\cL_n \coloneqq \{L_{i_n(1)}, L_{i_n(2)}, \dots, L_{i_n(j_n)}\}$. By definition of $j_n$, every $j' \le j_n$ satisfies $\mu_{\text{low}}(\Cl(\cL_n), L_{i_n(j_n)}) \ge \rho$. In particular, since $K \in \cL(p) \subseteq \cL_n$, we have $\mu_{\text{low}}(\Cl(\cL_n), K) \ge \rho$. Thus, the output
        \[\bigcap_{j \le j_n} L_{i_n(j)} \setminus S_n = \Cl(\cL_n) \setminus S_n\]
        has lower density at least $\rho$ in $K$, which concludes the proof.
    \end{proof}

\subsection{Generation with Element-Based Density under Contamination}
\label{sec:dense-generation:element-based}

Next we consider the notion of generation with element-based density introduced by \citep{kleinberg2025density}. Recall that under this definition (\Cref{def:element-based-density}), the generator outputs one element in every timestep and we measure the density of
the \emph{entire output sequence} of the generator in the target language $K.$ The validity requirement remains the same as in the original framing of the setting \citep{kleinberg2024language}, \ie, after some finite timestep the generator needs to output unseen elements of the target language $K.$ 
In \cref{sec:low-den-set-to-low-den-element}
we show a transformation from algorithms that 
achieve set-based lower density to algorithms 
that achieve element-Based lower density.
In \cref{sec:element-based-density-finite-noise} we discuss our results in the case of \emph{finite} amount of noise and omissions from the adversary, and in \cref{sec:element-based-infinite-noise} we discuss a general impossibility result for achieving element-based density under infinite noise.

\subsubsection{From Set-Based Lower Density to Element-Based Lower Density}\label{sec:low-den-set-to-low-den-element}
To begin with, we show that under some mild assumptions, which all the set-based generation algorithms in this paper satisfy, we may turn any generator that achieves set-based density into one that achieves element-based density. 

\begin{theorem}\label{thm:set-density-imply-element-density}
    Let $\cL$ be a countable collection. Suppose under a specific enumeration $x_{1:\infty}$ of the target language $K$, potentially  under  contamination, there is a set-based generator $\generator$ that generates in the limit from $\cL$. Suppose at every round $n$, given $S_n = \{x_1, \dots, x_n\}$, $\generator$ outputs an infinite set $A_n \subseteq U$ and computes a finite subcollection $\cL_n \subseteq \cL$ such that
    \begin{enumerate*}[(a)]
        \item $K \in \cL_n$ for all large enough $n$ and
        \item $A_n \subseteq \Cl(\cL_n) \setminus \{x_1, \dots, x_n\}$ for all $n$.
    \end{enumerate*}
    
    Assume further that $\generator$ achieves $\mu_{low}(A_n, K)\geq \rho$,
    for all $n$ sufficiently large.
    Then, there exists an element-based generator $\bar \generator$ that,
    under the same enumeration $x_{1:\infty}$,
    generates in the limit from $\cL$ and achieves element-based lower density at least $\nfrac{\rho}{2}$.
\end{theorem}
If we assume some set-based generator achieves set-based lower density at least $\rho$, then for any $\eps > 0$, we have $\mu_{\text{low}}(A_n, K) \ge \rho - \eps$ for all large enough $n$. Thus, we have the following simple corollary.
\begin{corollary}
    In the same setting as \cref{thm:set-density-imply-element-density}, if the set-based generator achieves set-based lower density at least $\rho$, \ie{}, $\liminf_{n \to \infty} \mu_{\text{low}}(A_n, K) \ge \rho$. Then, for any $\eps > 0$, there exists an element-based generator that generates in the limit and achieves element-based lower density at least $(\rho  - \eps)/2$.
\end{corollary}
We remark that the element-based generator we construct for \cref{thm:set-density-imply-element-density} \emph{assumes} the knowledge of $\rho$.

\paragraph{Pseudocode.} The pseudocode is presented in \Cref{alg:set-to-element2}.
We assume that at each round, the set-based generator computes a finite set of indices $I(n)$ such that $\cL_n = \{L_i: i \in I(n)\}$, and outputs an infinite set $A_n \subseteq \Cl(\cL_n) \setminus \{x_1, \dots, x_n\}$.
Moreover,
for all sufficiently large $n$, $K \in \cL_n$ and $\mu_{low}(A_n, K)\geq \rho$
for some given $\rho\in [0, 1]$.

\begin{algorithm}[tbh!]
        \caption{Algorithm for \Cref{thm:set-density-imply-element-density}}
        \label{alg:set-to-element2}
        \begin{algorithmic}[1]
        \Require Countable  collection  $\cL$; enumeration $x_{1:\infty}$; set-based generator $\generator$
        \State Let $S_n \gets \sinbrace{x_1, \dots, x_n}$ be the set of examples enumerated by the adversary till round $n$
        \State Let $W_{n-1} \gets \sinbrace{w_1, \dots, w_{n-1}}$ be the set of elements output before round $n$
        \State Compute $(A_n, \cL_n) \gets \generator(S_n)$
        \vspace{2mm}
        \State Compute $m_n \in \N$ to be the smallest number such that for all $L \in \cL_n$,
        \begin{align}
            \frac{\card*{A_n \cap \{\ell_1, \ell_2, \dots, \ell_m\}}}{m} \ge \frac{\mu_{\text{low}}(A_n, L)}{1+2^{-n}}, \quad \forall\,m \ge m_n, \label{eq:eps-liminf-condition}
        \end{align}
        where $L = \{\ell_1, \ell_2, \ell_3, \dots\}$ denote the listing of elements of $L$ in the natural ordering.
        \If{$n > 1$}
            \State Update $m_n\gets \max\left\{ m_n, m_{n-1} + 1\right\}$ so that the sequence $m_1 < \dots < m_n$ is increasing.
        \EndIf
        \State Compute the index
        \[
            k(n)\gets
            \begin{cases}
                \max\inbrace{k\in [n]: \frac{m_k}{1+2^{-k}}\leq \frac{2n}{\rho}}, & \quad \text{ if }\frac{m_1}{1+2^{-1}}\leq \frac{2n}{\rho}\,, \\
                1, & \quad \text{ else}\,.
            \end{cases}
        \]
        \State Output the smallest $w_n \in A_{k(n)} \setminus (S_n \cup W_{n-1})$ in the canonical ordering.
        \end{algorithmic}
    \end{algorithm}

\begin{proof}[Proof of \cref{thm:set-density-imply-element-density}]
    If $\rho = 0$, then the statement of the theorem becomes trivial. We assume that $\rho > 0$ is a positive constant.
    
    Fix an enumeration $x_{1:\infty}$, potentially  under  contamination. Let $\generator$ be a set-based generator satisfying the assumptions in \cref{thm:set-density-imply-element-density}. Consider the element-based generator $\bar \generator$ described in \cref{alg:set-to-element2}.

    We first show that the element-based generator $\bar \generator$ in \cref{alg:set-to-element2} is well-defined. Since $\cL_n$ is a finite subcollection of $\cL$, if we define $m_n(L) \in \N$ to be the smallest number such that \eqref{eq:eps-liminf-condition} holds for $L$, then 
    \[m_n = \max_{L \in \cL_n} m_n(L)\]
    is well-defined.
    Each $A_n$ is infinite,
    hence there is always some $w_n\in A_{k(n)}\setminus (S_n\cup W_{n-1})$.
    
    Now we proceed to show that $\bar \generator$ generates in the limit and achieves element-based lower density at least $\frac{\rho}{2}$.

    \paragraph{$\bar \generator$ generates in the limit:} 
    Let $N_1$ be sufficiently large so that $\cL_n\ni K$ for all $n\geq N_1$.
    Moreover,
    let $N_2$ be sufficiently large so that $N_2 \ge N_1$ and $\frac{2n}\rho \geq \frac{m_{N_1}}{1+2^{-N_1}}$ for all $n\geq N_2$.
    Since the sequence $\frac{m_n}{1+2^{-n}}$ is non-decreasing in $n$,
    we see that $k(n)\geq N_1$ for all $n\geq N_2$.
    This means that $K \in \cL_{k(n)}$ and  $A_{k(n)}\sset K$ for all $n\geq N_2$,
    so that $w_n\in A_{k(n)}\setminus (S_n\cup W_{n-1})\sset K\setminus (S_n\cup W_{n-1})$ for all $n\geq N_2$.

    \paragraph{$\bar \generator$ achieves element-based lower density:} 
    Let $K=\sinbrace{\ell_1, \ell_2, \dots}$ denote the canonical ordering of $K$.
    We first prove the following claim:
    \begin{claim} \label{claim:smallest-element-bound}
        There exists $n^\star \in \N$ such that $w_n\in \inbrace{\ell_1, \dots, \ell_{\ceil{2n(1+2^{-k(n)})/\rho}}}$ for all $n\geq n^\star$\,.
    \end{claim}
    To see the claim,
    let $N_3$ be sufficiently large so that $\cL_n\ni K$ and $\mu(A_n, K)\geq \rho$ for all $n\geq N_3$.
    Moreover,
    let $N_4$ be sufficiently large so that $\frac{2n}\rho \geq \frac{m_{N_3}}{1+2^{-N_3}}$ for all $n\geq N_4$.
    This means that $A_{k(n)}\sset K$ for all $n\geq N_4$.
    By construction, for all $n \ge N_4$, we have $k(n) \ge N_3$ and
    \[
        \card{A_{k(n)}\cap \sinbrace{\ell_1, \dots, \ell_m}}
        \geq \frac{\mu_{low}(A_{k(n)}, K)\cdot m}{1+2^{-k(n)}}
        \geq \frac{\rho m}{1+2^{-k(n)}}
    \]
    for all $m\geq m_{k(n)}$,
    and in particular for $m=\ceil{2n(1+2^{-k(n)})/\rho}$ by construction of $k(n)$.
    
    In other words,
    \[
        \card*{A_{k(n)}\cap \inbrace{\ell_1, \dots, \ell_{\ceil{2n(1+2^{-k(n)})/\rho}}}}
        \geq \frac{\rho \inparen{2n(1+2^{-k(n)})/\rho}}{1+2^{-k(n)}}
        = 2n\,.
    \]
    Since there are at most $(2n-1)$ elements in $S_n\cup W_{n-1}$,
    the smallest unseen element $w_n$ according to the canonical ordering of $A_{k(n)}$
    that we output
    must be belong to the first $\ceil{2n(1+2^{-k(n)})/\rho}$ elements in the canonical ordering of $K$. This finishes the proof of the claim, in which we can take $n^\star = N_4$.

    With \cref{claim:smallest-element-bound} in hand, for $n \ge n^\star$, if we define
    \begin{align*}
        t(n) &\coloneqq \max_{\bar{n} \le n} \ceil{2\bar{n}(1+2^{-k(\bar{n})})/\rho}
        \intertext{then,}
        \frac{\card*{w_{1:\infty}\cap \sinbrace{\ell_1, \dots, \ell_{t(n)}}}}{t(n)}
        &\geq \frac{{\card*{\{w_{n^\star}, w_{n^\star + 1}, \dots, w_n\} \cap \sinbrace{\ell_1, \dots, \ell_{t(n)}}}}}{t(n)} \geq \frac{n - n^\star}{t(n)}\,. \stepcounter{equation}\tag{\theequation}\label{eq:element-density-lower-bound}
    \end{align*}
    Since $t(n)= \max_{\bar{n} \le n} \ceil{2\bar{n}(1+2^{-k(\bar{n})})/\rho} \to \infty$ as $n \to \infty$ and is non-decreasing, for all $m$ large enough, there exists a unique number $n(m) \in \N$ such that
    \[t(n(m)) \le m < t(n(m) + 1).\]
    Moreover, $n(m)$ is non-decreasing in $m$, and $n(m) \to \infty$ as $m \to \infty$. Also note that $k(n)$ is an unbounded and non-decreasing integer sequence.
    Since $k(n) \to \infty$ and $t(n) \to \infty$ as $n \to \infty$, we have
    \begin{align*}
        \mu_{\text{low}}(w_{1:\infty}, K) &= \liminf_{m \to \infty} \frac{\card*{w_{1:\infty} \cap \sinbrace{\ell_1, \dots, \ell_m}}}{m}\\
        &\ge \liminf_{m \to \infty} \frac{\card*{w_{1:\infty} \cap \sinbrace{\ell_1, \dots, \ell_{t(n(m))}}}}{t(n(m) + 1)}\\
        &= \liminf_{m \to \infty} \frac{\card*{w_{1:\infty} \cap \sinbrace{\ell_1, \dots, \ell_{t(n(m))}}}}{t(n(m))} \cdot \frac{t(n(m))}{t(n(m) + 1)}\\
        &= \liminf_{m \to \infty} \frac{\card*{w_{1:\infty} \cap \sinbrace{\ell_1, \dots, \ell_{t(n(m))}}}}{t(n(m))} \cdot \frac{\max_{\bar{n} \le n(m)} \ceil{2\bar{n}(1+2^{-k(\bar{n})})/\rho}}{\max_{\bar{n} \le n(m)+1} \ceil{2\bar{n}(1+2^{-k(\bar{n})})/\rho}}\\
        &= \liminf_{m \to \infty} \frac{\card*{w_{1:\infty} \cap \sinbrace{\ell_1, \dots, \ell_{t(n(m))}}}}{t(n(m))} \cdot \frac{t(n(m))}{\max\{t(n(m)), \ceil{2(n(m)+1)(1+2^{-k(n(m)+1)})/\rho}\} }
        \intertext{since $t(n(m)) = \max_{\bar{n} \le n(m)} \ceil{2\bar{n}(1+2^{-k(\bar{n})})/\rho} \ge \ceil{2n(m)(1+2^{-k(n(m))})/\rho}$, we have}
        &\ge \liminf_{m \to \infty} \frac{\card*{w_{1:\infty} \cap \sinbrace{\ell_1, \dots, \ell_{t(n(m))}}}}{t(n(m))}\\
        &\quad \cdot \frac{\ceil{2n(m)(1+2^{-k(n(m))})/\rho}}{\max\{\ceil{2n(m)(1+2^{-k(n(m))})/\rho}, \ceil{2(n(m)+1)(1+2^{-k(n(m)+1)})/\rho}\} }\\
        &= \liminf_{m \to \infty} \frac{\card*{w_{1:\infty} \cap \sinbrace{\ell_1, \dots, \ell_{t(n(m))}}}}{t(n(m))} \cdot \min\left\{1, \frac{\ceil{2n(m)(1+2^{-k(n(m))})/\rho}}{\ceil{2(n(m)+1)(1+2^{-k(n(m)+1)})/\rho}}\right\}
        \intertext{since $n(m) \to \infty$ and $k(n(m)) \to \infty$ as $m \to \infty$, we get}
        &= \liminf_{m \to \infty} \frac{\card*{w_{1:\infty} \cap \sinbrace{\ell_1, \dots, \ell_{t(n(m))}}}}{t(n(m))}
        \intertext{by inequality \eqref{eq:element-density-lower-bound}, we get}
        &\ge \liminf_{m \to \infty} \frac{n(m) - n^\star}{t(n(m))}
        \intertext{since $n(m)$ is non-decreasing and unbounded, we have}
        &\ge \liminf_{n \to \infty} \frac{n - n^\star}{t(n)}\\
        &= \liminf_{n \to \infty} \frac{n - n^\star}{\max_{\bar{n} \le n}\ceil{2\bar{n}(1+2^{-k(\bar{n})})/\rho}}
        \intertext{note that if $\bar{n} < n/2$, we have $2\bar{n}(1 + 2^{-k(\bar{n})})/ \rho \le 2n/\rho \le 2n(1 + 2^{-k(n)})/\rho$. Thus, we have}
        &= \liminf_{n \to \infty} \frac{n - n^\star}{\underset{\substack{\bar{n}:\\ n/2 \le \bar{n} \le n}}{\max}\ceil{2\bar{n}(1+2^{-k(\bar{n})})/\rho}}\\
        &\ge \liminf_{n \to \infty} \frac{n - n^\star}{\underset{\substack{\bar{n}:\\ n/2 \le \bar{n} \le n}}{\max}\ceil{2n(1+2^{-k(\bar{n})})/\rho}}\\
        &= \frac{\rho}{2}\,,
    \end{align*}
    where the last equality follows from $k(n) \to \infty$ as $n \to \infty$.

    This concludes the proof.
\end{proof}

\subsubsection{Element-Based Density with Finite Contamination}\label{sec:element-based-density-finite-noise}
    Our main result in this setting is reduction to the noiseless case: given an algorithm that  achieves upper (lower) element-based density for \emph{all} countable collections in the noiseless setting,
    we obtain an algorithm that
    achieves upper (lower) element-based density for all countable collections when the adversary is allowed to omit finitely many elements from $K$ and inject finitely many elements outside of $K$ in its enumeration.
    This reduction is based on the expansion subroutine described in \cref{alg:finite-expansion-routine}.
    
    \begin{theorem}[Element-Based Density Guarantee with Finite Contamination]\label{thm:element-based-density-finite-noise-omissions}
        Fix $\rho\in [0,1]$.
        Suppose there is a generator $\generator_{\rm vanilla}$, that, for any countable collection $\cL$, generates $\cL$ in the limit with upper (respectively lower) element-based density $\rho_{\rm up}$ (respectively $\rho_{\rm up}$) under no noise and no omissions.
        
        Then there exists a generator $\generator_{\rm tolerant}$, that, for any countable collection $\cL$,  generates $\cL$ in the limit with upper (respectively lower) element-based density $\rho_{\rm up}$ (respectively $\rho_{\rm up}$). 
    \end{theorem}
    Before we give the proof of the result, we recall two results from \cite{kleinberg2025density} that
    are important to our derivations.
    
    \begin{theorem}[Noiseless Element-Based Density Guarantees \citep{kleinberg2025density}]\label{thm:kw25-element-density}
        For every countable collection of languages there is an algorithm that generates in the limit and achieves $\nicefrac{1}{2}$ (respectively $\nicefrac{1}{8}$) upper (respectively lower) element-based density when there is no noise or omissions.
    \end{theorem}
    This, combined with \cref{thm:element-based-density-finite-noise-omissions}, gives us the following corollary.
    \begin{corollary}
        For all countable collections of languages $\cL,$ there is an algorithm that achieves $\nicefrac{1}{2}$ (respectively $\nicefrac{1}{8}$) upper (respectively lower) element-based density under finite noise and finite omissions.
    \end{corollary}
    We are now ready to prove our result.

    \begin{proof}[Proof of \cref{thm:element-based-density-finite-noise-omissions}]
    The proof proceeds by reduction. We construct a new algorithm, $\generator_{\text{tolerant}}$, that uses the algorithm $\generator_{\text{vanilla}}$ from \cref{thm:kw25-element-density} as a subroutine. We will show that $\generator_{\text{tolerant}}$ generates from $\cL$ under the noisy setting and inherits the density guarantees of $\generator_{\text{vanilla}}$.
    
    \paragraph{1. The Reduction.}
    Let $\cL$ be the countable collection of languages from the theorem statement. We define a new, expanded collection of languages, $\cL'$ (also described in \cref{alg:finite-expansion-routine}), as follows:
    \[
    \cL' := \{L \cup A \setminus B \mid L \in \cL, A, B \subseteq \cX, \abs{A} < \infty, \abs{B} < \infty\}
    \]
    In words, $\cL'$ contains every language from $\cL$ plus all possible ``contaminated'' versions of those languages that result from adding a finite set ($A$) and removing a finite set ($B$). Since $\cL$ is countable and the set of all pairs of finite sets $(A, B)$ is countable, $\cL'$ is also a countable collection.
    
    \paragraph{Our Algorithm ($\generator_{\text{tolerant}}$):}
    Our new algorithm $\generator_{\text{tolerant}}$ is defined simply as the algorithm $\generator_{\text{vanilla}}$ from \cref{thm:kw25-element-density} executed on the language class $\cL'$.
    
    \paragraph{The Reduction:}
    The key insight is that an adversary presenting a \emph{contaminated enumeration of $K \in \cL$} is indistinguishable from an adversary presenting an \emph{enumeration of some $T \in \cL'$ without noise or omissions}.
    Specifically, let the adversary choose a target language $K \in \cL$ and present an enumeration $x$ with a finite noise set $A_x = x \setminus K$ and a finite omission set $B_x = K \setminus x$. The set of all elements enumerated in the text $x$ is therefore precisely the language $T = (K \setminus B_x) \cup A_x$.
    
    By definition, this language $T$ is a member of our augmented class $\cL'$ (with $L=K, A=A_x, B=B_x$). Therefore, the contaminated enumeration of $K$ is, from the algorithm's perspective, a \emph{complete and noiseless} enumeration of the language $T \in \cL'$.
    
    \paragraph{2. Proof of Correctness (Generation in the Limit).}
    From \cref{lem:expansion-subroutine} we immediately get that our algorithm generates in the limit from $K.$
    
    \paragraph{3. Proof of Density Guarantees.}
    We again consider the adversary's perspective (noisy enumeration of $K$) and the algorithm's perspective (noiseless enumeration of $T$). Let $S_{\text{out}} = (s_1, s_2, \dots)$ be the infinite output sequence from $\generator_{\text{tolerant}}$.
    By \cref{thm:kw25-element-density}, $\generator_{\text{vanilla}}$ achieves the density guarantees for its output with respect to its target $T$. Formally, per the specified definition, the guarantees on the output set $S_\text{out}$ are 
    \begin{itemize}
        \item $\mu_{\text{up}}(S_\text{out}, T) = \limsup_{n \to \infty} \frac{\abs{S_\text{out} \cap \inbrace{t_1,\ldots,t_n}}}{n} \ge \rho_{\rm up}$,
        \item $\mu_{\text{low}}(S_\text{out}, T) = \liminf_{n \to \infty} \frac{\abs{S_\text{out} \cap \inbrace{t_1,\ldots,t_n}}}{n} \ge \rho_{\rm low}$,
    \end{itemize}
    where $\inbrace{t_1,\ldots,t_n}$ is the set of the first $n$ elements of $T$ (based on the canonical enumeration of the domain).
    It is not hard to see that since the symmetric difference of $K, T$ is finite, the exact same bounds hold when $T$ is replaced by $K.$
    \end{proof}

    \subsubsection{Impossibility of Element-Based Density Under Infinite Contamination} \label{sec:element-based-infinite-noise}
    Next, we move to the study of generation with density under infinite noise. First, by \cref{thm:set-based-lower-density-finite-noise}, we know that even under finite noise, set-based generation is not guaranteed to achieve any non-trivial density for arbitrary countable collection. Clearly, this already rules out the possibility of achieving non-trivial set-based density for arbitrary countable collection under infinite contamination. On the other hand, \cref{thm:element-based-density-finite-noise-omissions} shows that element-based generation always achieves a non-trivial upper (lower) element-based density under enumeration with finite contamination. It is natural to wonder if it is still possible to achieve non-trivial element-based density under infinite contamination, or even vanishing noise rate without omissions. The following theorem answers this question in the negative.

    \begin{theorem}\label{thm:density:upper-element:lower-bound}
        There exists a collection of two languages $L_1, L_2$ and an enumeration with $o(1)$-noise and without omissions from a target language $K$ in this collection, such that no element-based generator can generate from $K$ in the limit and achieve element-based upper density at least $\rho$ for any $\rho > 0$.
    \end{theorem}
    
    \begin{proof}
        Consider two languages $L_1, L_2$ with the property that $L_1 \subseteq L_2$ and $\mu_{\text{up}}(L_1, L_2) = 0$. Consider the following enumeration of the elements of $L_2$.
    
        Define a scheduling of time steps $T = \{1, 2, 4, 8, \dots\}$. We only need the fact that $T$ has upper density $0$ in $\mathbb{N}$. At time $n \in T$, the adversary outputs the next element in $L_2 \setminus L_1$ that hasn't been generated. At time $n \not\in T$, the adversary outputs the next element in $L_1$ that hasn't been generated.
    
        Note that the adversary will enumerate all the elements of $L_2$, and thus all the elements of $L_1$. Clearly, the adversary enumerates from $L_2$ with noise rate $0$. On the other hand, since the adversary enumerates elements not belonging to $L_1$ only at time $n \in T$, and $T$ has density $0$ in $\mathbb{N}$, the noise rate of the adversary's enumeration is $o(1)$ in $L_1$.
    
        Since the enumeration is a valid enumeration with $o(1)$-noise for both $L_1$ and $L_2$, and $L_1 \subseteq L_2$, any generator generating from $K$ in the limit has to generate from $L_1$ in the limit. Since $L_1$ has upper density $0$ in $L_2$, the generator achieves element-based upper density $0$ if the adversary sets $K = L_2$. 
    \end{proof}

\section{Generation with Density Beyond the Worst-Case}
\label{sec:beyond}

In this section we introduce a beyond-worst-case model
that restricts the order in which the adversary can present the elements of its chosen language $K$; importantly, this model does not restrict the choices of the adversary in selecting the target $K,$ it merely restricts its power to present its elements in arbitrary orders. 
Informally, the adversary cannot, infinitely often, enumerate elements of $K$ that appear arbitrarily later in the canonical enumeration than the elements it has enumerated so far. 
In particular, under the interpretation that the canonical enumeration places ``easier'' elements of languages before ``harder'' ones, we impose the restriction that the adversary cannot, infinitely often, enumerate very hard examples before easier ones. 
This is inspired by practical phenomena of LLM training; indeed, folklore results known as ``curriculum learning'' empirically show that it is crucial that LLMs are first trained on ``easier'' tasks before harder ones \citep{bengio2009curriculum,hacohen2019power}.

\begin{definition}[$M$-Bounded Displacement Enumeration]\label{def:bounded-displacement-enumeration}
    Let $x_{1:\infty}$ be any enumeration
    and $L$ be an arbitrary language.
    Let $\sinbrace{\ell_1, \ell_2, \dots} = L$ be the canonical enumeration of $L$
    and define 
    \[
        \sigma(n) = \sigma(x_n, L)\coloneqq 
        \begin{cases}
            j, &x_n=\ell_j\in L\,, \\
            0, &x_n\notin L\,.
        \end{cases}
    \]
    We say that $x_{1:\infty}$ is an \emph{$M$-bounded displacement enumeration with respect to $L$} if there is some $n^\star\in \N$
    such that for all $n\geq n^\star$,
    $\sigma(n)\leq Mn$.
\end{definition}
Remark that by definition,
$\ell_{\sigma(n)} = \ell_{\sigma(x_n, L)} = x_n$ for all $x_n\in L$.
Note also that we do not require $\sigma:\N\to \N\cup \sinbrace{0}$ to be injective or surjective since it is possible that $x_{1:\infty}\neq L$ as sets.

In order to develop some intuition for \Cref{def:bounded-displacement-enumeration},
we consider several examples below.
First,
we consider some simple noiseless examples.
\begin{example}
    Let $L=\N$ and consider the following enumerations.
    \begin{itemize}
        \item The canonical enumeration $x_n=n$ is $M$-bounded for $M=1$
        \item Let $L\coloneqq \N-M\cdot \N$ be the natural numbers that are not divisible by $M$.
        The enumeration where $x_n=M\cdot \ceil{n/2}$ when $n$ is odd 
        and $x_n$ being the $(\nfrac{n}2)$-th element of $L$ when $n$ is even
        is $(\nfrac{M}2)$-bounded.
        \item Let $L\coloneqq \sinbrace{n\in \N: \sqrt{n}\notin \N}$ be the non-squares.
        The enumeration consisting of $x_n=\ceil{n/2}^2$ when $n$ is odd
        and $x_n$ being the $(\nfrac{n}2)$-th element of $L$ when $n$ is even
        is \emph{not} $M$-bounded for any $M$.
    \end{itemize}
\end{example}
Next, we consider a slightly more involved example with constant noise rate.
As we will see in \Cref{sec:beyond:bounded-displacement:properties},
\Cref{ex:bounded-displacement-enumeration:constant-noise} demonstrates that generation under constant noise remains hard even with bounded displacement enumerations.
\begin{example}\label{ex:bounded-displacement-enumeration:constant-noise}
    Let $k\geq 2$ be an integer and define the language
    \[
        L_i
        \coloneqq [k]\cup \inparen{\N\setminus i\cdot \N}
        = \inbrace{1, \dots k}\cup \inbrace{n\in \N: \text{$n$ is not divisible by $i$}}\,.
    \]
    Consider the finite collection $\cL\coloneqq \sinbrace{L_i: i\in [k]}$
    and the canonical enumeration $x_n = n$ of the natural numbers.
    By construction,
    the noise rate \wrt{} each $L_i, i\in [k]$ is bounded above by $\nfrac1k$.
    Moreover,
    we have $\sigma(x_n, L_i)\leq \inparen{1+\nfrac1k}n$ for all $n\geq 1, i\in [k]$.
    Hence $x_{1:\infty}$ is a $M$-bounded displacement enumeration of every $L_i, i\in [k]$ with noise rate $\nfrac1k$.
\end{example}

\begin{remark}
    Intuitively,
    \Cref{def:bounded-displacement-enumeration} restricts the ``speed'' of the enumeration to a linear speed-up compared to the canonical enumeration.
    Slightly more precisely,
    the adversary cannot place elements from deep in the tail of the canonical enumeration at the front of the input enumeration.
    Moreover,
    the adversary cannot omit too many elements from the prefix of the canonical enumeration,
    as otherwise, it is forced to output elements from deeper in the tail of the canonical enumeration.

    It is illustrative to consider the lower bound construction for generation with density under $o(1)$-noise enumerations from \Cref{thm:density:upper-element:lower-bound} which consists of two languages $L'\sset L$.
    Let $L=\sinbrace{\ell_1, \ell_2, \dots}$ be the canonical enumeration of $L$.
    \Cref{thm:density:upper-element:lower-bound} essentially chooses 
    $L' = \sinbrace{\ell_{i_1}, \ell_{i_2}, \dots}$ where $i_n = \omega(n)$
    so that $L'$ is a sparse subsequence of $L$.
    Then, the adversary mainly enumerates from $L'$,
    occasionally outputting from $L$ so that the enumeration is a valid $o(1)$-noise enumeration of both $L', L$.
    Roughly speaking,
    under this enumeration,
    any algorithm cannot distinguish whether the target language is $L$ or $L'$ and is forced to generate from the sparse subset $L'$.
    However,
    this worst case enumeration cannot be bounded \wrt{} $L$ since $i_n=\omega(n)$.
    If we are promised that the enumeration is $M$-bounded \wrt{} the target language,
    the adversary cannot maliciously enumerate from such a sparse subset of the target language.
\end{remark}

\subsection{Properties of Bounded Displacement Adversaries}\label{sec:beyond:bounded-displacement:properties}
In order to develop a better understanding of \Cref{def:bounded-displacement-enumeration},
we study several elementary properties.

\paragraph{Closed under Subsets.}
An immediate but useful feature of bounded displacement enumerations is that they are closed under language subsets.
\begin{proposition}\label{prop:bounded-displacement-enumeration-subset}
    Suppose the enumeration $x_{1:\infty}$ is an $M$-bounded displacement enumeration with respect to some language $L$.
    Then for any $L'\sset L$,
    $x_{1:\infty}$ is an $M$-bounded displacement enumeration \wrt{} $L'$.
\end{proposition}

\begin{proof}
    Let $\sinbrace{\ell_1, \ell_2, \dots} = L$ be the canonical enumeration of $L$.
    The canonical enumeration of $L'$ can be obtained from that of $L$ by deleting some elements
    and decreasing the indices of remaining elements.
    In other words,
    there is some subsequence $i_1\leq i_2\leq \dots$ such that $\sinbrace{\ell_{i_1}, \ell_{i_2}, \dots} = L'$
    is the canonical enumeration of $L'$.
    For any $x=\ell_{i_j}\in L'$,
    its index $j=\sigma(x, L')$ in $L'$ is at most its index $i_j=\sigma(x, L)$ in $L$.
    Hence for all sufficiently large $n$,
    \[
        \sigma(x_n, L')\leq \sigma(x_n, L)\leq Mn\,,
    \]
    as desired.
\end{proof}

\paragraph{Change of Density.}
The next lemma is the key lemma for generation with lower density.
Intuitively,
it says that we can estimate the density of a language up to a multiplicative factor
by simply computing the empirical ``density'' \wrt{} the \emph{input enumeration}.
This is key since we do not know the canonical enumeration of the target language
(otherwise generation with density is trivial)
but we do have access to the input enumeration.
\begin{lemma}[Change of Density]\label{lem:change-of-density}
    Let $x_{1:\infty}$ be an $M$-bounded displacement enumeration of a language $K$
    with arbitrary noise rate (and possibly omissions).
    Then for any other language $L\sset K$,
    \begin{align*}
        \mu_{low}(L, K)
        &\geq \frac1M\cdot \liminf_n \frac{\card{L\cap x_{1:n}}}{n} \\
        \mu_{up}(L, K)
        &\geq \frac1M\cdot \limsup_n \frac{\card{L\cap x_{1:n}}}{n}\,.
    \end{align*}
\end{lemma}

\begin{proof}
    By \Cref{prop:bounded-displacement-enumeration-subset},
    $x_{1:\infty}$ is also an $M$-bounded displacement enumeration \wrt{} $L$.
    By \Cref{def:bounded-displacement-enumeration},
    for each $n\geq n^\star$,
    we have $\sigma(x_n, L)\leq Mn$.
    Define $B\coloneqq \max\sinbrace{\sigma(x_n, L): n<n^\star}$
    and consider the canonical enumeration $\sinbrace{\ell_1, \ell_2, \dots}$ of $K$.
    For $n\geq \max(n^\star, B)$,
    we have $L\cap \sinbrace{x_1, \dots, x_n}\sset L\cap \sinbrace{\ell_1, \dots, \ell_{Mn}}$,
    where we understand $\ell_0$ to be a special string not in $K$.
    This implies that
    \begin{align*}
        \frac{\card*{L\cap \inbrace{x_1, \dots, x_n}}}{n}
        &\leq \frac{\card*{L\cap \inbrace{\ell_1, \dots, \ell_{Mn}}}}{n} \\
        &= M\cdot \frac{\card*{L\cap \inbrace{\ell_1, \dots, \ell_{Mn}}}}{Mn}\,.
    \end{align*}
    Note that the limit inferior (superior) of the RHS is exactly the lower (upper) density of $L$ in $K$.
    Thus taking the limit inferior (superior) on both sides
    concludes the proof.
\end{proof}

\paragraph{Hardness of Identification.}
Since we have introduced a new model that is limiting the power of the adversary, it is worth understanding how restrictive 
this change is. Towards that end, we study how the landscape of identification changes under our definition. Our result below shows that, even for simple collections of languages, identification remains intractable for any $M > 1$.
We leave a full characterization of identification for $M$-bounded adversaries as an interesting open question.

\begin{theorem}
    \label{thm:beyond:identification-hard}
    There is a countable collection of languages $\cL$
    for which no algorithm can identify in the limit,
    even when restricted to $M$-bounded enumerations for any $M>1$.
\end{theorem}

\begin{proof}
    Let $\cL = \sinbrace{\N}\cup \sinbrace{\N\setminus\sinbrace{n}: n\in \N}$ be the collection consisting of the natural numbers
    and all its subsets obtained by deleting a single number.
    Let $\cA$ be any algorithm.

    The adversary begins by enumerating the natural numbers $1, 2, \dots$.
    If $\cA$ never outputs $\N$,
    we are done.
    Let $n_1$ be the first step when $\cA$ outputs $\N$.
    Then at step $n_1+1$,
    we output $n_1+2$ instead of $n_1+1$ and then continue enumerating the natural numbers in the canonical order.
    If $\cA$ never outputs $\N\setminus\sinbrace{n_1+1}$,
    we are done.
    Otherwise,
    let $n_2 > n_1$ be the first step when $\cA$ outputs $\N\setminus\sinbrace{n_1+1}$.
    At step $n_2+1$,
    the adversary ``fills in'' the previously skipped element by outputting $n_1+1$,
    instead of $n_2+1$, and then continues enumerating the natural numbers in the canonical order again.
    Repeat.

    All in all,
    the adversary can always ensure that $\cA$ does not correctly identify after any finite time.
    On the other hand,
    if we let $\sigma(n)$ denote the value of the element the adversary enumerates at step $n$,
    then $\sigma(n)\leq n+1$.
    Thus identification in the limit remains impossible in general
    even under $M$-bounded displacement enumerations,
    for any $M>1$.
\end{proof}

\paragraph{Hardness of Generation under Constant Noise.}
To further understand the landscape of generation under bounded displacement enumerations,
and as earlier promised,
we show that generation with finite noise remains hard.

\begin{theorem}
    \label{thm:beyond:generation-cnoise-hard}
    For any $k\geq 2$,
    there is a finite collection of $k$ languages $\cL$
    for which no algorithm can generate in the limit from $\cL$
    under constant $\nfrac1k$-rate noise,
    even when restricted to $\nfrac1k$-bounded enumerations.
\end{theorem}

\begin{proof}
    Consider the finite family $\cL$ from \Cref{ex:bounded-displacement-enumeration:constant-noise}.
    Let $k\geq 2$ be an integer and define the language
    \[
        L_i
        \coloneqq [k]\cup \inparen{\N\setminus i\cdot \N}
        = \inbrace{1, \dots k}\cup \inbrace{n\in \N: \text{$n$ is not divisible by $i$}}\,.
    \]
    Then $\cL = \sinbrace{L_i: i\in [k]}$.

    As we have seen in \Cref{ex:bounded-displacement-enumeration:constant-noise},
    the canonical enumeration of the natural numbers is $\nfrac1k$-bounded and has $\nfrac1k$-noise \wrt{} every $L_i, i\in [k]$.
    Thus any hypothetical algorithm that generates from $\cL$ under the assumed conditions must simultaneously generate from all $L_i$.
    In other words,
    it generates from $\cap_{i\in [k]} L_i = [k]$.
    This must be a contradiction as this is a finite set.
\end{proof}
\noindent All in all,
we have seen that bounded displacement enumerations
does not trivialize the identification problem, nor generation under constant rate noise.
\Cref{sec:beyond:set-lower} shows that we do gain an advantage for the regime of $o(1)$-noise
in terms of generation with density.

\subsection{Results for Set-Based Density}\label{sec:beyond:set-lower}
We now present our results for generation with set-based density guarantee under bounded displacement adversaries.
In \Cref{sec:beyond:set-lower:generation},
we first design an algorithm which roughly obtains $\nfrac1M$ set-based lower density under $M$-bounded enumerations.
\Cref{sec:beyond:set-lower:lower-bound} demonstrates that this guarantee is tight in the worst case sense:
In general, no algorithm can do better than $\nfrac1M$ in this setting. 

\subsubsection{Generation with Set-Based Lower Density}\label{sec:beyond:set-lower:generation}
    Having described the beyond-worst-case model we consider, we show how these types of enumerations help 
    achieve non-trivial lower density (in a set-based sense) under vanishing noise rate for \emph{all} countable collections.
    At each step,
    our algorithm outputs a finite set of languages indexed by $I(n)\sset \N$
    such that for sufficiently large $n$,
    their intersection is an infinite subset of the target, 
    \ie{}, $\cap_{i\in I}L_i\sset K$ (generation in the limit),
    and is dense in $K$.

    \begin{theorem}\label{thm:improper:generation:inf-inf-set-density}
        There is an algorithm $\generator$ that,
        for any collection $\cL$,
        target language $K\in \cL$,
        and error parameter $\eps\in (0, 1)$,
        given an $M$-bounded displacement enumeration $x_{1:\infty}$ with $o(1)$-noise and arbitrary omissions,
        $\generator$ outputs an intersection $\cap_{i\in I(n)} L_i$ of finitely many languages $\card{I(n)}<\infty$
        with infinite cardinality $\card{\cap_{i\in I(n)} L_i} = \infty$ at step $n$ such that
        \begin{enumerate}[(a)]
            \item $\cap_{i\in I(n)} L_i\sset K$ for all sufficiently large $n$ and
            \item $\mu_{low}\inparen{\cap_{i\in I(n)} L_i, K}\geq \frac{1-\eps}{M}$ for all sufficiently large $n$.
        \end{enumerate}
    \end{theorem}

    \paragraph{Pseudocode.} 
    The pseudocode for our algorithm is an instantiation of our generic meta-algorithm \Cref{alg:intersection-meta}
    and is presented in \Cref{alg:intersection-vanishing-noise-improper-inf-inf}.
    We highlight the changes in blue.
    We remark that the algorithm requires two non-standard parameters as input.
    $\eps\in (0, 1)$ is an error parameter which is required to compute priorities appropriately.
    The parameter $M\geq 1$ is the boundedness parameter (\Cref{def:bounded-displacement-enumeration}), which is required to correctly compute both the priorities and the stopping rule.
    \Cref{alg:intersection-vanishing-noise-improper-inf-inf} ensures $\frac{1-\eps}M$ lower density. 
    
    Recall the following notation for a string $x$ and language $L$ with canonical enumeration $L=\sinbrace{\ell_1, \ell_2, \dots}$
    \begin{align*}
        \sigma(x, L) &\coloneqq
        \begin{cases}
            j, &x=\ell_j\in L\,, \\
            0, &x\notin L\,.
        \end{cases}
    \end{align*}
    In other words,
    $\sigma(x, L)$ is the index of $x$ in the canonical enumeration of $L$ if $x\in L$
    and otherwise $\sigma(x; L)=0$ if $x\notin L$.
    As a shorthand for an enumeration $x_{1:\infty}$,
    we write
    \[
        \sigma(x_{1:n}, L)
        \coloneqq \max_{j\in [n]} \sigma(x_j, L)\,.
    \]
    We remark that by \Cref{def:bounded-displacement-enumeration},
    any $M$-bounded enumeration $x_{1:\infty}$ of $L$ satisfies $\sigma(x_{1:n}; L)\leq Mn$ for all sufficiently large $n$.
    
    \begin{algorithm}[tbh!]
        \caption{Algorithm for \Cref{thm:improper:generation:inf-inf-set-density}}
        \label{alg:intersection-vanishing-noise-improper-inf-inf}
        \begin{algorithmic}[1]
        \Require Countable {collection} $\cL=\{L_1,L_2,\dots\}$; 
        density error parameter $\eps\in (0, 1)$;
        enumeration $x_{1:\infty}$;
        bounded displacement parameter $M\geq 1$
        \State Let $S_n \gets \sinbrace{x_1, \dots, x_n}$ be the set of examples seen in the first $n$ steps
        \State Let $W_{n-1} \gets \sinbrace{w_1, \dots, w_{n-1}}$ be the set of strings output before the $n$-th step
        \vspace{2mm}
        \For{$i=1, 2, \dots, n$}
            \State {\color{blue} Compute the smallest $N_i^{(n)}$ such that 
            $L_i$ {is nearly consistent with $x_{1:m}$
            and $x_{1:m}$ is $M$-bounded \wrt{} $L_i$,
            for every $N_i^{(n)}\leq m\leq n$}
            \[
                N_i^{(n)} \gets
                \begin{cases}
                    \min\inbrace{N\geq 1: \forall m\in \sinbrace{N, \dots, n}, R_m(L_i; x_{1:m})\leq 2^{-i}\eps\land \sigma(x_{1:m}; L_i)\leq Mm}\,, &\text{exists}\,, \\
                    n+1\,, &\text{else}\,.
                \end{cases}
            \]}
            \State {\color{blue} Assign the language $L_i$ a priority of $P_i^{(n)}\gets i + N_i^{(n)}$}
        \EndFor
        \vspace{2mm}
        \State Re-order $\inbrace{L_1, \dots, L_n}$ in increasing priority, tie-breaking by index,
        as $\sinbrace{L_{i_n(1)}, \dots, L_{i_n(n)}}$,
        \ie{}, for each $j\in [n-1]$,
        ensure either $P_{i_n(j)}^{(n)} < P_{i_n(j+1)}^{(n)}$
        or $\sinparen{P_{i_n(j)}^{(n)} = P_{i_n(j+1)}^{(n)}}\land \sinparen{i_n(j) < i_n(j+1)}$.
        \State {\color{blue} Compute the largest index $J(n)$ such that the intersection in the re-ordering up to $L_{i_n(J_n})$ is dense in $L_{i_n(j)}$
        for all $j\leq J_n$
        \[
            J_n \gets \max\inbrace{\bar j\in [n]: \forall j\leq \bar j, \mu_{low}\inparen{\cap_{j\leq J_n} L_{i_j(n)}, L_{i_j(n)}}\geq \frac{1-\eps}M}\,.
        \]}
        \State {\color{blue} Output $\bigcap_{j\leq J_n} L_{i_j(n)}$}.
        \end{algorithmic}
    \end{algorithm}

    \paragraph{Analysis.}
    Recall the following notation for a given $p\geq 1$
    \[
        P_i^\infty = \lim_{n\to \infty} P_i^{(n)}\,,
        \qquad \cL(p) = \inbrace{L_i: P_i^\infty\leq p}\,,
        \qquad \Cl(\cL(p)) = \cap_{L\in \cL(p)} L\,.
    \]
    \begin{proof}[Proof of \Cref{thm:improper:generation:inf-inf-set-density}]
        Let $\cH_n \coloneqq \sinbrace{L_{i_n(j)}: j\leq J_n}$ denote the prefix class.
        Write $K=L_{i^\star}$ for the target language
        and define $p\coloneqq P_{i^\star}^\infty < \infty$,
        which is guaranteed to exist by the assumption that $x_{1:\infty}$ is a valid $M$-bounded $o(1)$-noisy enumeration of $L_{i^\star}$.
        We have $L_{i^\star}\in \cL(p)$ by construction.
        We will argue that $\cL(p)\sset \cH_n$ for all sufficiently large $n$
        so that the output $\cL(\cH_n)$ satisfies $\mu_{low}\inparen{\Cl(\cH_n), K}\geq \frac{1-\eps}{M}$ by the definition of \Cref{alg:intersection-vanishing-noise-improper-inf-inf} since $K\in \cL(p)\sset \cH_n$.

        We first note that the priority of each language is increasing in $n$ by definition.
        Hence by \Cref{lem:meta-algorithm:prefix-stabilizes},
        there is some $n^\star\in \N$ such that for $n\geq n^\star$,
        the priority $P_i^{(n)} = P_i^\infty$ of $L_i\in \cL(p)$ no longer changes
        and moreover,
        $P_i^{(n)}$ takes on strictly smaller value than any language not in $\cL(p)$.
        In other words,
        $\cL(p)$ is ordered before all other languages.
        Thus, in order to argue that the stopping rule $J_n$ stops after considering all members of $\cL(p)$,
        it suffices to show that for any subset $\cL'\sset \cL(p)$,
        $\mu_{low}\inparen{\Cl(\cL'), L}\geq \frac{1-\eps}{M}$ for every $L\in \cL'(p)$.

        Now,
        by the definition of the priorities in \Cref{alg:intersection-vanishing-noise-improper-inf-inf}
        for $n\geq n^\star$
        every $L\in \cL'\sset \cL(p)$ satisfies
        \[
            R(L; x_{1:n})\coloneqq \frac{\card{x_{1:n}\setminus L}}{n}
            \leq \frac{\eps}{2^i}\,.
        \]
        Thus for $n\geq n^\star$,
        and any non-empty $\cL'\sset \cL(p)$,
        we have
        \begin{align*}
            \frac{\card{\Cl(\cL')\cap x_{1:n}}}n
            &\geq 1-\sum_{i: L_i\in \cL'} \frac{\card{x_{1:n}\setminus L_i}}{n}
            \geq 1- \sum_{i\geq 1} \frac{\eps}{2^i}\geq 1-\eps \,.
        \end{align*}
        Also by the definition of the priorities in \Cref{alg:intersection-vanishing-noise-improper-inf-inf},
        for $n\geq n^\star$
        every $L\in \cL'\sset \cL(p)$ satisfies
        \[
            \sigma(x_{1:n}, L)\leq Mn\,.
        \]
        In other words,
        the input enumeration $x_{1:\infty}$ is also an $M$-bounded enumeration \wrt{} $L$.
        Thus,
        the change in density formula (\Cref{lem:change-of-density}) applies
        to the languages $\Cl(\cL')\sset L$
        and we conclude that for all $L\in \cL'$,
        \begin{align*}
            \mu_{low}\inparen{\Cl(\cL'), L}
            &\geq \frac1M \liminf_n \frac{\card{\Cl(\cL')\cap x_{1:n}}}{n} \\
            &\geq \frac{1-\eps}{M}\,.
        \end{align*}
        This concludes the proof by our initial remark.
    \end{proof}

\paragraph{Implications to Generation without Noise.} 
Interestingly, \cref{thm:improper:generation:inf-inf-set-density}
has direct implications for generation without noise, that could be of interest beyond the scope of our work. In particular, since this result shows that (even with noise) if we restrict the adversary to enumerations with $M$-bounded displacements, we can achieve $\nfrac{1}{M}$ lower density in a set-based manner for \emph{all} countable collections. Recall that \cite{kleinberg2025density} showed that one can achieve $c$ set-based lower density under worst-case enumerations if and only if $\cL$ does not contain an infinite perfect tower, parametrized by $c.$ Our result shows that this lower bound can be circumvented by restricting the adversary to $\nfrac{1}{c}$ bounded enumerations. We leave a complete characterization of the landscape of this beyond-worst-case generation without noise as an interesting open direction.

\subsubsection{Lower Bound for Set-Based Density}\label{sec:beyond:set-lower:lower-bound}
In the previous sections,
we demonstrated that given an $M$-bounded enumeration,
it is always possible to perform set-based generation with set-based lower density $\frac{1-\eps}M$,
where $\eps\in (0, 1)$ is arbitrary.
A natural question is whether this is the best possible.
In this section,
we show that in the worst case,
no generator can obtain density greater than $\frac1M$,
even \emph{upper density}. 

\begin{theorem}\label{thm:beyond:set-lower:lower-bound}
    There exists a collection of two languages $L_1, L_2$ such that 
    if the adversary provides an $M$-bounded displacement enumeration with noise rate $o(1)$ from this collection, 
    then any generator that generates from the target $K$ in the limit achieves set-based upper density at most $\nfrac1M$.
\end{theorem}
The proof is via a similar construction as \Cref{thm:density:upper-element:lower-bound},
except we must construct the enumeration and collection slightly more carefully to satisfy $M$-boundedness.

\begin{proof}
    Consider two languages $L_1=M\cdot \N \coloneqq \sinbrace{M\cdot n: n\in \N}$ and $L_2=\N$ with the property that $L_1 \subseteq L_2$ and $\mu_{\text{up}}(L_1, L_2) = \frac1M$. Consider the following enumeration of the elements of $L_2$.

    Define a scheduling of time steps $T = \{1, 2, 4, 8, \dots\}$. We only need the fact that $T$ has density $0$ in $\mathbb{N}$.
    Let $x_{1:\infty}$ be the following enumeration:
    At time $n \in T$, the adversary outputs the next element in $L_2 \setminus L_1$ that it has not output. 
    At time $n \not\in T$, the adversary outputs the next element in $L_1$ that it has not output.
    Now, $\sigma(x_n, L_1)\leq n$ since we can interpret $x_{1:\infty}$ as the canonical enumeration of $L_1 = M\cdot \N$ with some elements of $L_2\setminus L_1$ interleaved.
    Note that the latter can only decrease $\sigma(x_n, L_1)$.
    On the other hand
    \[
        \sigma(x_n, L_2) \leq
        \begin{cases}
            n, &n\in T\,, \\
            M\cdot n, &n\notin T\,.
        \end{cases}
    \]
    Hence $x_{1:\infty}$ is an $M$-bounded enumeration \wrt{} both $L_1, L_2$.

    Note that the adversary will enumerate all the elements of $L_2$, and thus all the elements of $L_1$. Clearly, the adversary enumerates from $L_2$ with noise rate $0$. On the other hand, since the adversary enumerates elements not belonging to $L_1$ only at time $n \in T$, and $T$ has density $0$ in $\mathbb{N}$, the noise rate of the adversary's enumeration is also $o(1)$ in $L_1$.

    Since the enumeration is a valid enumeration with $o(1)$-noise for both $L_1$ and $L_2$, and $L_1 \subsetneq L_2$, any set-based generator generating from $K$ in the limit has to output $S_n\sset L_1$ for all sufficiently large $n$. 
    Since $L_1$ has upper density $\nfrac1M$ in $L_2$, the generator achieves set-based upper density at most $\nfrac1M$ if the adversary sets $K = L_2$. 
\end{proof}

\subsection{Results for Element-Based Density}\label{sec:beyond:element-lower}
In this section,
we extend our results for set-based density to element-based density.
\Cref{sec:beyond:element-lower:generation} applies our black-box reduction from set-based generation with density
to roughly attain element-based density $\frac1{2M}$ with noise under $M$-bounded displacement enumerations.
Similarly, we show that no algorithm can do better than $\frac1{2M}$ in general within \Cref{sec:beyond:element-lower:lower-bound}.
Note that, unlike the set-based density guarantee, our upper and lower bounds have a gap of $\frac1{2M}$.

\subsubsection{Generation with Element-Based Lower Density}\label{sec:beyond:element-lower:generation}
    In this section,
    we design an algorithm for element-based generation with lower density
    given an $M$-bounded enumeration.
    \begin{theorem}\label{thm:improper:generation:inf-element-density}
        There is an generator $\generator$ that,
        for any collection $\cL$,
        target language $K\in \cL$,
        and error parameter $\eps\in (0, 1)$,
        given an $M$-bounded displacement enumeration $x_{1:\infty}$ with $o(1)$-noise rate and arbitrary omissions,
        the output $w_{1:\infty}$ of $\generator$ has lower density
        \[
            \mu_{low}(w_{1:\infty}, K)\geq \frac{1-\eps}{2M}\,.
        \]
    \end{theorem}
    Having completed the heavy lifting of set-based generation with density in \Cref{sec:beyond:set-lower:generation},
    we can now leverage our black-box reduction from set-based generation with density to element-based generation with density from \Cref{sec:low-den-set-to-low-den-element} to design a simple algorithm.
    \begin{proof}
        Our \Cref{alg:intersection-vanishing-noise-improper-inf-inf} achieving lower set-based density (\Cref{thm:improper:generation:inf-inf-set-density}) falls under the set-based generation algorithm in the premise of \Cref{thm:set-density-imply-element-density}.
        Hence, we can apply the black-box reduction from \Cref{thm:set-density-imply-element-density} to construct the desired generator.
    \end{proof}

\subsubsection{Lower Bound for Element-Based Upper Density}\label{sec:beyond:element-lower:lower-bound}
We now demonstrate that the same counterexample for set-based density also extends to the element-based density under bounded displacement enumerations.
\begin{theorem}\label{thm:beyond:element-lower:lower-bound}
    There exists a collection of two languages $L_1, L_2$ such that 
    if the adversary provides an $M$-bounded displacement enumeration with $o(1)$-noise from this collection, 
    then any generator that generates from the target $K$ in the limit achieves element-based upper density at most $\nfrac1M$.
\end{theorem}
The proof is identical to that of \Cref{thm:beyond:set-lower:lower-bound}.
In the setting where arbitrary enumeration is allowed,
the adversary can choose to skip enumerating elements that the generator outputs
and force the generator to obtain density at most $\nfrac12$.
We can picture this as skipping every other element of some given enumeration.
However, if the original enumeration is $M$-bounded,
skipping every other element leads to a $2M$-bounded enumeration.
Thus, the adversary's strategy for arbitrary enumerations does not hold for bounded enumerations.
This is the main reason we have a gap of $\frac1{2M}$ between the density guarantee of \Cref{thm:improper:generation:inf-element-density} and the lower bound in \Cref{thm:beyond:element-lower:lower-bound}.

\begin{proof}
    Consider two languages $L_1=M\cdot \N \coloneqq \sinbrace{M\cdot n: n\in \N}$ and $L_2=\N$ with the property that $L_1 \subseteq L_2$ and $\mu_{\text{up}}(L_1, L_2) = \frac1M$. Consider the following enumeration of the elements of $L_2$.

    Define a scheduling of time steps $T = \{1, 2, 4, 8, \dots\}$. We only need the fact that $T$ has density $0$ in $\mathbb{N}$.
    Let $x_{1:\infty}$ be the following enumeration:
    At time $n \in T$, the adversary outputs the next element in $L_2 \setminus L_1$ that it has not output. 
    At time $n \not\in T$, the adversary outputs the next element in $L_1$ that it has not output.
    Now, $\sigma(x_n, L_1)\leq n$ since we can interpret $x_{1:\infty}$ as the canonical enumeration of $L_1 = M\cdot \N$ with some elements of $L_2\setminus L_1$ interleaved.
    Note that the latter can only decrease $\sigma(x_n, L_1)$.
    On the other hand
    \[
        \sigma(x_n, L_2) \leq
        \begin{cases}
            n, &n\in T\,, \\
            M\cdot n, &n\notin T\,.
        \end{cases}
    \]
    Hence $x_{1:\infty}$ is an $M$-bounded enumeration \wrt{} both $L_1, L_2$.

    Note that the adversary will enumerate all the elements of $L_2$, and thus all the elements of $L_1$. Clearly, the adversary enumerates from $L_2$ with noise rate $0$. On the other hand, since the adversary enumerates elements not belonging to $L_1$ only at time $n \in T$, and $T$ has density $0$ in $\mathbb{N}$, the noise rate of the adversary's enumeration is also $o(1)$ in $L_1$.

    Since the enumeration is a valid $o(1)$-noisy enumeration for both $L_1$ and $L_2$, and $L_1 \subsetneq L_2$, any set-based generator generating from $K$ in the limit has to output $x_n\in L_1$ for all sufficiently large $n$. 
    Since $L_1$ has upper density $\nfrac1M$ in $L_2$, the generator achieves set-based upper density at most $\nfrac1M$ if the adversary sets $K = L_2$. 
\end{proof}
We leave as an open question the tight density guarantee given an $M$-bounded enumeration.

\section*{Acknowledgements}
Felix Zhou acknowledges the support of the Natural Sciences and Engineering Research Council of Canada (NSERC).

\clearpage
\printbibliography

\newpage
\appendix

\section{Generation with Repetitions in the Input Enumeration}\label{apx:repetitions}
We first remark that our hardness results hold when allowing for repetitions as we have proven them for the special case with 0 repetition.

We briefly sketch how to handle repetitions in the input enumeration within our algorithms.
For an enumeration $x_1, x_2, \dots$, 
possibly with repetitions,
we define the sub-enumeration $x_{i_1}, x_{i_2}, \dots$ consisting of the first occurence of the unique elements in $x_{1:\infty}$.
That is, $x_{i_{1:\infty}}$ is obtained from $x_{1:\infty}$ by deleting elements $x_n$ that already occurred at some step $m<n$.
Modify all our assumptions so that they hold with respect to the unique elements observed so far.
For example, redefine the empirical noise rate so that the denominator is the number of unique elements seen so far
\[
    R(x_{1:n}, L_)\coloneqq \frac{\abs{x_{1:n}\setminus L}}{\card{x_{1:n}}}\,.
\]
Hypothetically,
if we could execute our algorithms on $x_{i_{1:\infty}}$,
we would be done.
This is essentially the case for set-based generation,
as running our algorithm on $x_{1:\infty}$ is equivalent to running our algorithm on $x_{i_{1:\infty}}$,
except we can simply output the same set on time steps when the adversary enumerates a repeated element.
Note that this suffices for both generation and generation with density.

The case of element-based generation is slightly more involved.
Fortunately, 
our element-based algorithms (with the exception of our reduction to \cite{kleinberg2025density}, which already handles repetitions) first compute an infinite set as a function of only the adversary's enumeration,
and then output an unseen element from this set.
Thus we can also perform the same reduction as set-based generation,
with the exception that on time steps when the adversary enumerates a repeated element,
we output the first unseen element from the set chosen in the last step when the adversary did not enumerate a repeat.
It is easy to see that this algorithm generates in the limit as long as the original algorithm generates in the limit with no repetitions.
To see that density guarantees still hold,
we remark that the set of elements generated on input $x_{1:n}$ is a superset of the elements generated on $x_{i_{1:n}}$.
Indeed, in step $n$ when the adversary enumerates a unique element,
the algorithm will have identified the same set as the hypothetical algorithm executed on the unique sub-enumeration.
It will also output the same first unseen element,
unless it has already been output due to a previous step caused by a repetition.
Note here we crucially used the fact that the infinite set we are picking from is only a function of the (unique) input enumeration.

\section{Additional Results}

    \subsection{Generation with Vanishing and Unknown Noise Rate: Sorting by Index Only}

    Fix a countable collection of languages $\cL$. Set $c_i \coloneqq \frac{1}{2^{i+1}}$. At time $n$, let $\cL^{(n)} \subseteq \cL$ be the subcollection of languages $L_i$ with empirical noise rate $R(L_i; x_{1:n}) \le c_i$.

    \begin{proposition} \label{prop:bad-subset}
        Fix a subset of languages $\cL_S \coloneqq \{L_i: i \in S\}$. If $|\Cl(\cL_S)| < \infty$, then there exists $T$ such that for $n \ge T$, $\cL_S \not\subseteq \cL^{(n)}$, \ie{}, there exists at least one language $L_i \in \cL_S \setminus \cL^{(n)}$.
    \end{proposition}

    \begin{proof}
        Assume that $\cL_S \subseteq \cL^{(n)}$. Then, we have
        \begin{align*}
            \left|\Cl(\cL_S) \cap \{x_1, \dots, x_n\}\right|
            &= |\{x_1, \dots, x_n\}| - \left| \bigcup_{i \in S} \left(\{x_1, \dots, x_n\} \setminus L_i\right)\right|\\
            &\ge n - \sum_{i \in S} \left|\{x_1, \dots, x_n\} \setminus L_i\right|\\
            &\ge n - \sum_{i \in S} n \cdot c_i
            \intertext{since $L_i \in \cL_S \subseteq \cL^{(n)}$,}
            &\ge \frac{n}{2}\,,
        \end{align*}
        where the last inequality follows from $\sum_{i=1}^\infty c_i \le \frac{1}{2}$.

        Since $|\Cl(\cL_S)| < \infty$, we know that the above inequality stops holding if $n > 2|\Cl(\cL_S)|$. Thus, there is a finite time $T$ so that for $n \ge T$, $\cL_S \not\subseteq \cL^{(n)}$.
    \end{proof}

    Now we prove the following algorithm also generates in the limit under vanishing noise rate, where we replaced the priority in \cref{alg:intersection} with simply sorting by indices of the languages with low empirical noise rate. Note that in particular, this algorithm does not enjoy the stable prefix property guaranteed by \cref{lem:meta-algorithm:prefix-stabilizes}.

    \begin{algorithm}[tbh!]
        \caption{Algorithm for \Cref{thm:vanishing-noise-generation-sorting}}
        \label{alg:intersection-sorting}
        \begin{algorithmic}[1]
        \Require Countable  collection  $\cL=\{L_1,L_2,\dots\}$;  thresholds $c_1, c_2,\ldots\in (0, 1)$; enumeration $x_{1:\infty}$
        \State Let $S_n \gets \sinbrace{x_1, \dots, x_n}$ be the set of examples seen in the first $n$ steps
        \State Let $W_{n-1} \gets \sinbrace{w_1, \dots, w_{n-1}}$ be the set of strings output before the $n$-th step
        \State Let $\cA_n \gets \inbrace{L_i: R(L_i; x_{1:n}) \le c_i, 1\le i \le n } $ be the active set of languages with empirical noise rate bounded by $c_i$
        \State {\color{blue} Compute the longest prefix $\cL_n$ of $\cA_n$ with infinite intersection, \ie{}, 
        \[
            \cL_n \gets \sup\inbrace{L_i \in \cA_n: \abs{\Cl(\inbrace{L_1, \dots, L_i} \cap \cA_n)} = \infty}\,.
        \]}
        \State Output any $w_n\in \Cl(\cL_n) \setminus (S_n\cup W_{n-1})$.
        \end{algorithmic}
    \end{algorithm}
    
    \begin{theorem}\label{thm:vanishing-noise-generation-sorting}
        Let $\cL$ be a countable collection of languages. Then, \cref{alg:intersection-sorting} generates in the limit from $\cL$ under $o(1)$-noise and infinite omissions.
    \end{theorem}

    \begin{proof}
        Recall that $\cL^{(n)} \coloneqq \inbrace{L_i: R(L_i; x_{1:n}) \le c_i}$. Since the adversary enumerates from $K$ with vanishing noise rate, we know there exists a time $T'$ such that for all $n \ge T'$, $K \in \cL^{(n)}$ and $K \in \{L_1, \dots, L_n\}$, and thus $K \in \cA_n$.

        Suppose $K = L_{i^\star}$. Now consider the set 
        \begin{align*}
            B = \left\{S \subseteq [i^\star]:  \left|\bigcap_{i \in S} L_i\right| < \infty\right\}.
        \end{align*}\
        For any $S \in B$, by \cref{prop:bad-subset}, there exists $T_S$ such that for $n \ge T_S$, $\{L_i: i \in S\} \not\subseteq \cL^{(n)}$. Without loss of generality, we may assume $T_S \ge T'$ for any $S \in B$.

        Let $T = \max_{S \in B} T_S$. For $n \ge T$, we must have $\{i: L_i \in \cL^{(n)}\} \cap [i^\star] \not\in B$, since for $n \ge T$, $\{L_i: i \in S\} \not\subseteq \cL^{(n)}$ for any $S \in B$. Therefore, for $n \ge \max\inbrace{T, i^\star}$, $K = L_{i^\star} \in \cL^{(n)}$, and the intersection of the prefix up to $K$, $\cL^{(n)}\cap \{L_1, \dots, L_{i^\star}\}$, has infinite number of elements. This means that the generator will generate $w_n  \in \Cl(\cL_n)$ from the intersection of a prefix $\cL_n$ of $\cA_n \supseteq \cL^{(n)}\cap \{L_1, \dots, L_{i^\star}\}$ containing $K = L_{i^\star}$, and thus generating in the limit.
    \end{proof}
    
\subsection{Generating from Finite Collections with Noise}
    In \cref{thm:constant-noise-characterization}, we characterize the collections that can be generated at noise rate $c$. 
    In this section, we ask the dual question.
    We consider a class of collections, namely all finite collections of size at most $k$, and ask what   is the highest noise rate $c>0$ under which all these collections can be generated in the limit.
    We prove the following result.
    \begin{theorem}
    \label{thm:finite-collections}
        For each $k\ge 2$,
            the following holds:
            \begin{enumerate}
                \item If $c\geq \nfrac{1}{k}$, then there is a finite collection $\cL$ of size $k$ that \emph{not} generable in the limit under $c$-noise.
                \item If $c < \nfrac{1}{k}$, then all finite collections $\cL$ of size $k$ can be generated in the limit under $c$-noise.
            \end{enumerate}
            
    \end{theorem}

    \subsubsection*{Necessary Condition for Generation from Finite Collections}
    First, we prove the necessary condition. 
    
    \begin{proof}[Proof of \cref{thm:finite-collections} (Necessity)]
    Let $\cX=\N$ and, for $i\in\{0,1,\ldots,k-1\}$, define
    \[
    L_{i+1} = \N\setminus\inbrace{n\in\N:  n\equiv i\pmod k}\,.
    \]
    In words, $L_{i+1}$ omits exactly the residue class $i\bmod k$.
    
    \paragraph{Adversary.} Consider the adversary’s enumeration $x_t=t$ (the canonical ordering). 
    Note that this is a valid ordering for all languages in $\cL$ with the fraction of noisy examples being $1/k+o(1)$; therefore the noise density is $c=1/k$ in the limit.
    
    \paragraph{Impossiblity of Generation.}
    Fix any generator producing outputs $w_1,w_2,\dots$ with $w_t\notin\{x_1,\dots,x_t\}$ for all $t\geq t^\star$ for some finite $t^\star$.
    Among the $k$ residue classes modulo $k$, one class $r\bmod k$ appears infinitely often in the output sequence $\{w_t\}_{t\ge 1}$ by the pigeonhole principle.
    Take the target language to be $K=L_{r+1}$ (which is possible due to the earlier observation that the adversaries stream is valid for languages in $\cL$).  
    Then, whenever $w_t\equiv r\pmod k$, we have $w_t\notin K$ by construction; hence the generator makes infinitely many mistakes \mbox{and, thus, fails to generate in the limit from $K$.}
    \end{proof}
     
    \subsubsection*{Sufficient Condition for Generation from Finite Collections}
        Next, we prove the sufficiency part of \cref{thm:finite-collections}.
    
        \paragraph{Overview of the algorithm and proof.}
        At a high level, our algorithm is as follows: at each time $t$, we compute a robust version space allowing up to $n_t$ mistakes among the first $t$ distinct examples, with $n_t$ carefully chosen so that it grows linearly in $t$ and slightly faster than the realized number of noisy examples.
        The latter condition ensures that the robust version space eventually contains the true language $K$ and a more careful argument shows that all languages in the version space share infinitely many elements.
        
        Before formalizing our result we begin with some preliminaries.
    
        \paragraph{Preliminaries: Bounds on the Noisy Closure Dimension.}
        Let $\cL=\{L_1,\dots,L_k\}$ be a finite class (where each $L_i$ infinite) and let $K\in\cL$ be the target language .
        For a finite sequence $x_{1:d}$ and a noise-budget $n\ge0$, define
        \[
            V(x_{1:d};n) = 
                \inbrace{\,L\in\cL:\ \abs{\{x_{1:d}\}\cap L}\,\ge\, d-n\,}\,,
            \quadand
            \inangle{x_{1:d}}_{\cL,n} 
                = 
                \!\!\bigcap_{L\in V(x_{1:d};n)}\!\!L
        \]
        (with $\langle x_{1:d}\rangle_{\cL,n}=\bot$ if $H(x_{1:d};n)=\emptyset$). The $n$-\emph{Noisy Closure dimension} $\mathrm{NC}_n(\cL)$ is the largest $d$ for which there exist distinct $x_{1:d}$ with $\langle x_{1:d}\rangle_{\cL,n}\ne\bot$ but finite. 
        For a finite size-$k$ collection $\cL$,\cite{raman2025noisy} show that there is a class-dependent constant $d^\star<\infty$ such that\footnote{The constant $d^\star$ depends only on $\cL$ (finite) and can be precomputed from it as in the proof of Corollary~3.4 of \cite{raman2025noisy}.}
        \begin{equation}\label{eq:NC-upper}
            \mathrm{NC}_n(\cL) \le  n\,k + d^\star + 1\qquad\text{for all }n\in\N\,.
        \end{equation}
        (See their Corollary~3.4 in \cite{raman2025noisy} and its proof for the construction of $d^\star$.)  
    
    \paragraph{Our Result.}
        Next, we present the main result of this section
    
    \begin{theorem}[Finite classes tolerate density $\nu<1/k$]
    \label{thm:density-finite}
    Fix any size-$k$ collection $\cL=\{L_1,\dots,L_k\}$ and $\nu<1/k$.
    There exists a generator $\generator$ (not knowing $\nu$) that generates from $K$ in the limit against every adversary whose noise rate is at most $\nu$ in the limit.
    \end{theorem}
    
    \begin{proof}
        Recall that we consider a density-bounded adversary that presents a \emph{density-bounded noisy enumeration} of $K$, \ie{}, a stream of distinct examples $(x_t)_{t\ge1}$ in which every $x\in K$ appears at some finite time and
        \[
        \limsup_{d\to\infty}\frac{1}{d}\,\abs{\{t\le d:\ x_t\notin K\}}\ \le\ \nu\,.
        \]
        Let $S_t=\{x_1,\dots,x_t\}$ be the stream of examples presented in the first $t$ steps and $d_t=\abs{S_t}$ be the size of $S_t.$
        Define the dynamic noise budget
        \begin{equation}\label{eq:nt-def}
        n_t = \max\!\inbrace{
            \, 0,\ 
            \floor{
                \frac{d_t - (d^\star+2)}{k}
            }
        }\,.
        \end{equation}
        \paragraph{Generation $\generator$.}
        We claim that it is sufficient for $\generator$ proceed as follows: At time $t$, 
        \begin{enumerate}
            \item Compute $n_t$ by \eqref{eq:nt-def} and the robust noisy-closure $C_t\coloneqq \langle x_{1:t}\rangle_{\cL,n_t}$.
            \item If $C_t\ne\bot$, output any $w_t\in C_t\setminus S_t$. Otherwise, output any arbitrary placeholder string.
        \end{enumerate}
    
        \paragraph{Correctness of $\generator$.}
        In the remainder of the proof, we prove $\generator$'s correctness in several steps.
        
        \begin{claim}[Sufficiency of Dynamic budget]
            \label{claim:finite:dynamicBudget}
            There is a finite $T<\infty$ such that for all $t\geq T$, $N_t\le n_t$.
        \end{claim}
        \begin{proof}[Proof of \cref{claim:finite:dynamicBudget}]
            Let $N_t \coloneqq \abs{\{i\le t:\ x_i\notin K\}}$ be the number of corruptions by time $t$.
            Due to the limsup condition, there exists an $\epsilon\in\inparen{0,(\nfrac{1}{k})-c}$ and $t_0$ large enough that for all $t\ge t_0$,
            \begin{equation}\label{eq:noise-upper}
            N_t\ \le\ (c+\epsilon)\,d_t\,.
            \end{equation}
            On the other hand, by \eqref{eq:nt-def}, for $t$ with $d_t$ sufficiently large we have
            \[
            n_t\ \ge\ \frac{d_t-(d^\star+2)}{k}-1 \ =\ \frac{d_t}{k}-\frac{d^\star+2+k}{k}.
            \]
            Combining with \eqref{eq:noise-upper} and using $c+\epsilon<\frac1k$, we obtain $N_t\le n_t$ for all $t\ge T$ for some finite $T$.
            In particular, it is sufficient to select $T=t_0 + \frac{1}{\eps k}(d^\star + 2 + k)$.
        \end{proof}
        \begin{claim}[Closure is Eventually Infinite]
            \label{claim:finite:infiniteClosure}
            For any $t\geq T$ (for $T$ from \cref{claim:finite:dynamicBudget}), $\abs{C_t}=\infty$ and $C_t\subseteq K$.
        \end{claim}
        \begin{proof}
            Fix any $t\ge T$. Since $N_t\le n_t$, the true language $K$ belongs to $V(x_{1:t};n_t)$ and, hence, $C_t=\langle x_{1:t}\rangle_{\cL,n_t}\ne\bot$.
            Moreover, because $d_t>\mathrm{NC}_{n_t}(\cL)$ by \eqref{eq:NC-upper} and the definition of $n_t$, the noisy-closure $C_t$ is either empty or infinite (by the definition of the noisy closure dimension); moreover, since the previous sentence rules out empty, $C_t$ must be infinite. 
            Finally, $C_t\subseteq K$ because $K\in V(x_{1:t};n_t)$ and $C_t$ is the intersection of supports over $V(x_{1:t};n_t)$.
        \end{proof}
        Now we are ready to conclude: By \cref{claim:finite:infiniteClosure}, for all $t\ge T$ the set $C_t\setminus S_t$ is nonempty and contained in $K$. Thus, $G$ outputs $w_t\in K\setminus S_t$ for all $t\ge T$.
    \end{proof}
    \begin{remark}[On the criticality of $c=\nfrac{1}{k}$]
        The sufficiency proof reveals why the constant $1/k$ appears. 
        The generator succeeds eventually when its noise budget $n_t$ satisfies both
        \[
            c < \frac{n_t}{d_t} 
            \quad\text{and}\quad
            \mathrm{NC}_{n_t} < d_t
        \]
        for all $t \geq T$ for some finite $T$. The first condition ensures $K$ enters the version space, while the second (with $K$ in the version space) ensures the intersection of all version space languages is infinite.
        Since the noisy closure dimension $\mathrm{NC}_n(\cdot)$ scales as $\Theta(nk)$ (see \cref{eq:NC-upper}), the $1/k$ bound emerges from the above constraints:
        \begin{itemize}
            \item The first condition requires $n_t = \Omega(c d_t)$.
            \item The second condition, using $\mathrm{NC}_{n_t} = \Theta(n_tk)$, requires $\Omega(c d_t k) \leq \mathrm{NC}_{n_t} < d_t$.
        \end{itemize}
        Thus $c d_t k < d_t$, implying $c < 1/k$.
    \end{remark}

\section{Additional Preliminaries}\label{app:preliminaries}
    In this section, we present the definitions of $\limsup$ and $\liminf$ that are used throughout the paper.
    \begin{definition}
        The limit inferior and limit superior of a sequence of reals $a=(a_n)_{n\in\N}$ are defined as:
        \[
            \liminf_{n\to\infty} a_n ~\coloneqq~ \sup_{N\in\N}\;\inf_{n\ge N} a_n
            \qquadand
            \limsup_{n\to\infty} a_n ~\coloneqq~ \inf_{N\in\N}\;\sup_{n\ge N} a_n\,.
        \]
    \end{definition}
    Note that, unlike $\lim_{n\to \infty} a_n$, limit inferior and limit superior exist for all sequences of reals $a$.
    We note that these limits are allowed to take values in the extended reals $\R\cup\{\pm\infty\}$. 
    
    Some simple facts are as follows, we refer the reader to \cite{Rudin1976PMA} for more extensive discussion.
    \begin{fact}
        Given a sequence of real numbers $a=(a_n)_{n\in\N}$, the following always hold:
        \begin{enumerate}
            \item $\liminf a_n \le \limsup a_n$
            \item If $\liminf a_n = \limsup a_n$, the $\lim_{n\to\infty}a_n$ exists and $\liminf a_n = \lim_{n\to\infty}a_n = \limsup a_n.$
        \end{enumerate}
    \end{fact}
    Next, we discuss some examples to build intuition.
\begin{example}
    The sequence $a_n = (-1)^n$ has $\liminf a_n = -1$ and $\limsup a_n = 1$
\end{example}
\vspace{-5mm}
\begin{example}
    Consider the sequence $a=(a_n)_{n\in \N}$ where $a_n=1$ if $n=2^m$ for some $m\in \N$ and, otherwise, $a_n=0$.
    It has $\liminf a_n = 0$ and $\limsup a_n = 1$.
\end{example}
\vspace{-5mm}
\begin{example}
    The sequence $a_n = (-1)^n + \nfrac{1}{n}$ has $\liminf a_n = -1$ and $\limsup a_n = 1$.
    In contrast, the related sequence $b_n = 1 + \nfrac{(-1)^n}{n}$ satisfies $\liminf b_n = \limsup b_n = 1$.
\end{example}

\end{document}